%% file: main_arxiv.tex
\documentclass{article} 
\usepackage[utf8]{inputenc}
\usepackage[margin=1in]{geometry}
\usepackage{cancel}
\usepackage{caption}
\usepackage{subcaption}
\usepackage{amsthm}
\usepackage{amsmath}
\usepackage{textcomp}
\usepackage{hyperref}
\hypersetup{colorlinks,allcolors=blue}
\usepackage{amsfonts}
\usepackage{mathtools}
\usepackage{amssymb}
\usepackage{fancyhdr}
\usepackage{enumerate}
\usepackage{bm}
\usepackage{thmtools,thm-restate}
\usepackage[pdftex]{graphicx}
\usepackage{comment}
\usepackage{xspace}
\usepackage{xcolor}
\usepackage{algorithm}
\usepackage{commands}
\usepackage[noend]{algpseudocode}

\usepackage[round]{natbib}

\title{Expressivity of Neural Networks via Chaotic Itineraries\\beyond Sharkovsky's Theorem}
\author{ Clayton Sanford\thanks{Columbia University, clayton@cs.columbia.edu}\ \ \ \  Vaggos Chatziafratis\thanks{Northwestern University, vaggos@\{northwestern, cs.stanford, ucsc\}.edu. Part of this work was done while the author was supported by Northwestern University.}}

\usepackage{graphicx}

\begin{document}

\maketitle

\input{abstract_arxiv}

\input{intro_arxiv}

\input{lb}
\input{phase}
\input{supp_compare}
\input{discussion}
\section*{Acknowledgements}
C.S. is supported by the National Science Foundation Graduate Research Fellowship Program (NSF GRFP); grants NSF CCF-1563155 and NSF CCF 1814873; a grant from the Simons Collaboration on Algorithms and Geometry; and a Google Faculty Research Award to Daniel Hsu. V.C. was supported by Northwestern University. The authors are grateful to Daniel Hsu for helpful comments and feedback on an early draft of this work.

\bibliography{bib}

\newpage

\appendix

\input{supp_toy}
\input{supp_itineraries}
\input{supp_lb_arxiv}
\input{supp_doubling}
\end{document}

%% file: abstract_arxiv.tex
Given a target function $f$, how large must a neural network be in order to approximate $f$?
Recent works examine this basic question on neural network \textit{expressivity} from the lens of dynamical systems and provide novel ``depth-vs-width'' tradeoffs for a large family of functions $f$. 
They suggest that such tradeoffs are governed by the existence of  \textit{periodic} points or \emph{cycles} in $f$. 
Our work, by further deploying dynamical systems concepts, illuminates a more subtle connection between periodicity and expressivity: we prove that periodic points alone lead to suboptimal depth-width tradeoffs and we improve upon them by demonstrating that certain ``chaotic itineraries'' give stronger exponential tradeoffs, even in regimes where previous analyses only imply polynomial gaps. 
Contrary to prior works, our bounds are nearly-optimal, tighten as the period increases, and handle strong notions of inapproximability (e.g., constant $L_1$ error). 
More broadly, we identify a phase transition to the \textit{chaotic regime} that exactly coincides with an abrupt shift in other notions of function complexity, including VC-dimension and topological entropy.

%% file: intro_arxiv.tex
\section{Introduction}



Whether a neural network (NN) succeeds or fails at a given task crucially depends on whether or not its architecture (depth, width, types of activation units etc.) is suitable for the task at hand. 
For example, a ``size-inflation'' phenomenon has occurred in recent years, in which NNs tend to be deeper and/or larger. 
Recall that in 2012, AlexNet had 8 layers.
In 2015, ResNet won the ImageNet competition with 152 layers~\citep{krizhevsky2012imagenet,he2016deep},
This trend still continues to date, with modern models using billions of parameters~\citep{brown2020language}. 
The empirical success of deep neural networks motivates researchers to ask: What are the theoretical benefits of depth, and what are the depth-vs-width tradeoffs?

This question gives rise to the study of neural network \textit{expressivity}, which characterizes the class of functions that are representable (or approximately representable) by a NN of certain depth, width, and activation. 
For instance, \cite{eldan2016COLT} propose a family of ``radial'' functions in $\mathbb{R}^d$ that are easily expressible with 3-layered feedforward neural nets of small width, but require any approximating 2-layer network to have exponentially (in $d$) many neurons. In other words, they formally show that depth---even if increased by 1---can be exponentially more valuable than width. 

Not surprisingly, understanding the expressivity of NNs was an early question asked in 1969, when Minsky and Papert showed that the Perceptron can only learn linearly separable data and fails on simple XOR functions~\citep{minsky2017perceptrons}. 
The natural question of which functions can multiple such Perceptrons (i.e., multilayer feedforward NN) express was addressed later by \cite{cybenko1989approximation,hornik1989multilayer} proving the so-called \textit{universal approximation} theorem. 
This states, roughly, that just one hidden layer of standard activation units (e.g., sigmoids, ReLUs etc.) suffices to approximate any continuous function arbitrarily well. 
Taken at face value, any continuous function is a 2-layer (i.e., 1-hidden-layer) network in disguise, and hence, there is no reason to consider deeper networks. 
However, the width required can grow arbitrarily, and many works in the following decades quantify those depth-vs-width tradeoffs.


Towards this direction, one typically identifies a function together with a ``measure of complexity'' to demonstrate benefits of depth. 
For example, the seminal work by \cite{telgarsky15,telgarsky16} relies on the number of oscillations of a simple triangular wave function. 
Other relevant notions of complexity to the expressivity of NNs include the VC dimension~\citep{warren1968lower,anthony1999neural,schmitt2000lower}, the number of linear regions~\citep{montufar2014number,arora2016understanding} or activation patterns~\citep{hanin2019deep}, the dimension of algebraic varieties~\citep{kileel2019expressive}, the Fourier spectrum~\citep{barron1993universal,eldan2016COLT,daniely2017depth,lee2017ability, bresler2020sharp}, fractals~\citep{malach2019deeper}, topological entropy~\citep{bzl20}, Lipschitzness~\citep{safran2019depth,hsu2021approximation}, global curvature and trajectory length~\citep{poole2016NIPS,raghu2017ICML} just to name a few.

This work builds upon recent papers~\citep{cnpw19,cnp20}, which study expressivity from the lens of discrete-time dynamical systems and extend Telgarsky's results beyond triangle (tent) maps. 
At a high-level, their idea is the following: if the initial layers of a NN output a real-valued 
function $f$, then concatenating the \textit{same} layers $k$ times one after the other outputs $ f^k  \vcentcolon = f\circ f\circ \ldots\circ f$, i.e., the composition of $f$ with itself $k$ times. 
By associating each discrete timestep $k$ to the output of the corresponding layer in the network, one can study expressivity via the underlying properties of $f$'s trajectories. 
Indeed, if $f$ contains higher-order fixed points, called \textit{periodic} points, then deeper NNs can efficiently approximate $f^k$, but shallower nets would require exponential width, governed by $f$'s periodicity.

Inspired by these novel connections to discrete dynamical systems, we pose the following natural question:
\begin{center}
\textit{Apart from periodicity, are there other properties of $f$'s trajectories governing the expressivity tradeoffs?}
\end{center}

We indeed prove that $f$'s periodicity alone is not the end of the story, and we improve on the known depth-width tradeoffs from several perspectives. 
We exhibit functions of the same period with very different behaviors (see Sec.~\ref{ssec:warmup}) that can be distinguished by the concept of ``chaotic itineraries.'' 
We analyze these here in order to achieve nearly-optimal tradeoffs for NNs. 
Our work highlights why previous works that examine periodicity alone only obtain loose bounds. 
More specifically:

\begin{itemize} 
\item We accurately quantify the oscillatory behavior of a large family of functions $f$. This leads to sharper and nearly-optimal lower bounds for the width of NNs that approximate $f^k$.
\item Our lower bounds cover a stronger notion of approximation error, i.e., \textit{constant} separations between NNs, instead of bounds that become small depending heavily on $f$ and its periodicity.
\item At a conceptual level, we introduce and study certain chaotic itineraries, which supersede Sharkovsky's theorem~(see Sec.~\ref{subsec:dds}). 
\item We elucidate connections between periodicity and other function complexity measures like the VC-dimension and the topological entropy~\citep{alseda2000}. 
We show that all of these measures undergo a phase transition that exactly coincides with the emergence of the chaotic regime based on periods.
\end{itemize}

To the best of our knowledge, we are the first to incorporate the notion of chaotic itineraries from discrete dynamical systems into the study of NN expressivity. Other related works that have previously used ideas from chaotic dynamical systems (e.g., Li-Yorke chaos) to show the convergence (or not) of standard optimization methods, e.g., gradient descent, multiplicative weights update algorithm, include~\cite{lee2019first,palaiopanos2017multiplicative}. Before stating and interpreting our results, we provide some basic definitions.

\subsection{Function Approximation and NNs}
This paper employs three notions of approximation to compare functions $f, g: [0,1] \to [0,1]$.
\begin{itemize}
    \item $L_1(f,g)= \norm[1]{f - g} = \int_{0}^1 \abs{f(x) - g(x)} dx.$
    \item $L_\infty(f,g)=\norm[\infty]{f - g} = \sup_{x \in [0, 1]} \abs{f(x) - g(x)}.$
    \item Classification error $\mathcal{R}_{S,t}$: For $t \in [0,1]$, let $\thres{t}{x} = \indicator{x \geq t}$. Let $S = \{x_1, \dots, x_n\} \subseteq [0,1]$. Then,
    $\errcls{S,t}{f,g}= \frac{1}{n} \sum_{i=1}^n \indicator{\thres{t}{f(x_i)} \neq \thres{t}{g(x_i)}}.$ 
\end{itemize}
For what follows, let $\mathcal{N}(u,\ell)$ be the family of feedforward NNs of depth $\ell$ and width at most $u$ per layer with ReLU activation functions.\footnote{Recall ReLU$(x)=\max(x,0)$.} All our results also hold for the more general family of semialgebraic activations~\citep{telgarsky16}.








\subsection{Discrete Dynamical Systems}
\label{subsec:dds}
To construct families of functions that yield depth-separation results, we rely on a standard notion of \textit{unimodal} functions from dynamical systems~\citep{mss73}.

\begin{definition}
     Let $f: [0,1]\to[0,1]$ be a continuous and piece-wise differentiable function. We say $f$ is a \emph{unimodal mapping} if:
    \begin{enumerate}
        \item $f(0) = f(1) = 0$, and $f(x) > 0$ for all $x \in (0,1)$.
        \item There exists a unique maximizer $x'\in (0,1)$ of $f$, i.e., $f$ is strictly increasing on the interval $[0, x')$ and strictly decreasing on $(x',1]$.
    \end{enumerate}
\end{definition}

Our constructions rely on unimodal functions that are concave and also symmetric (i.e., $f(x) = f(1-x)$ for all $x \in [0, 1]$). 
We note that the resulting function family is fairly general, already capturing the triangle waves of \cite{telgarsky15,telgarsky16} and the logistic map used in previous depth-separation results~\citep{schmitt2000lower}.
Moreover, the study of one-dimensional discrete dynamical systems by applied mathematicians explicitly identifies unimodal mappings as important objects of study~\citep{mss73, alseda2000}.

Recall that a fixed point $x^*$ of $f$ is a point where $f(x^*)=x^*$. A more general notion of higher-order fixed points is that of \textit{periodicity}.

 
\begin{definition}
For some $p \in \N$, we say that $x_1,\dots, x_p\in [0,1]$ is a \emph{$p$-cycle} if $f(x_j) = x_{j+1}$ for all $j \in [p-1]$ and $f(x_p) = x_1$.
We say that $f$ has \emph{periodicity} $p$ if such a cycle exists and that $x_1$ is a \emph{point of period $p$} if $x_1$ belongs to a $p$-cycle. 
Equivalently, $x_1$ is a point of period $p$ if $f^p(x_1) = x_1$ and $f^{k}(x_1) \neq x_1$\footnote{Throughout the paper, $f^{k}$ means composition of $f$ with itself $k$ times.} for all $k \in [p-1]$.\footnote{As is common, $[m] = \{1, 2, \dots, m\}$.}
\end{definition}

Does the existence of some $p$-cycle in $f$ have any implications about the existence of other cycles?
These relations between the periods of $f$ are of fundamental importance to the study of dynamical systems. 
In particular, \cite{li1975period} proved in 1975 that ``period 3 implies chaos'' in their celebrated work, which also introduced the term ``chaos'' to mathematics and later spurred the development of chaos theory. 
Interestingly, an even more general result was already obtained a decade earlier in Eastern Europe, by~\cite{sharkovsky1964coexistence,sharkovsky1965cycles}:

\begin{theorem}[Sharkovsky's Theorem]
\label{thm:sharkovsky}
Let $f: [0,1]\to[0,1]$ be continuous. If $f$ contains period $p$ and $p \triangleright p'$, then $f$ also contains period $p'$, where the symbol `` $\triangleright$'' is defined based on the following (decreasing) ordering: 
\[
3 \triangleright 5 \triangleright 7\triangleright \ldots \triangleright2\cdot3 \triangleright 2\cdot5 \triangleright 2\cdot7\triangleright\ldots
\]
\[
\ldots \triangleright 2^2\cdot3 \triangleright 2^2\cdot5 \triangleright 2^2\cdot7\triangleright\ldots\triangleright2^3\triangleright2^2\triangleright2\triangleright1.
\]
\end{theorem}

This ordering, called \textit{Sharkovsky's ordering}, is a total ordering on the natural numbers, where $l\triangleright r$ whenever $l$ is to the left of $r$. The maximum number in this ordering is 3; if $f$ contains period 3, then it also has all other periods, which is also known as \emph{Li-Yorke chaos}. 
\cite{cnpw19,cnp20} apply this theorem to obtain depth-width tradeoffs based on periods and obtain their most powerful results when $p = 3$. 
We go beyond Sharkovsky's theorem and prove that tradeoffs are determined by the ``itineraries'' of periods.\footnote{These are called ``patterns'' in~\cite{alseda2000}.}


\begin{definition}[Itineraries]
    For a $p$-cycle $x_1, \dots, x_p$, suppose that $x_{a_1} < \dots < x_{a_p}$ for $a_j \in [p]$.
    The \emph{itinerary} of the cycle is the cyclic permutation of $x_{a_1}, \dots, x_{a_p}$ induced by $f$, which we represent by the string $\ba= a_1\dots a_p$.
    Because cyclic permutations are invariant to rotation, we assume (wlog) that $a_1 = 1$.

    
\end{definition}  

\begin{definition} [Chaotic Itineraries]
    A $p$-cycle is a \emph{chaotic itinerary} or an \emph{increasing cycle} if its itinerary is $12\dots p$. That is, $x_1 < \dots < x_p$. 
\end{definition}
Examining chaotic itineraries circumvents the limitations of prior works based on periods and yields sharper exponential depth-width tradeoffs. 
Unlike other function complexity properties, the existence of a chaotic itinerary is easily verifiable (see App.~\ref{assec:identifying}).
To grasp clean examples of itineraries, see App.~\ref{assec:it-exs}.




\subsection{Our Main Contributions}

Our principal goal is to use knowledge about $f$'s itineraries to more accurately quantify the ``oscillations'' of $f^k$ as a measure of complexity and draw connections to other complexity measures. 
Section~\ref{sec:lb} produces sharper and more robust NN approximability tradeoffs than prior works by leveraging chaotic itineraries and unimodality.
Section~\ref{sec:phase} shows how a phase transition in VC-dimension and topological entropy of $f$ occurs exactly when the growth rate of oscillations shifts from polynomial to exponential.

While previous works count oscillations too, they either construct too narrow a range of functions\footnote{e.g.,~\cite{telgarsky15,telgarsky16} only analyzes triangles.}, obtain loose depth-width tradeoffs\footnote{e.g.,~\cite{cnpw19,cnp20} have a suboptimal dependence on $p$ under stringent Lipschitz assumptions.}, or have unsatisfactory approximation error. \footnote{e.g.,~\cite{cnpw19,cnp20,bzl20} do not obtain constant error rates.} 
In Section~\ref{sec:lb}, we improve along these three directions by taking advantage of the unimodality and itineraries of $f$. 
The \textit{unimodality} of $f$ allows us to quantify both the number of piecewise monotone pieces of $f^k$ (i.e., oscillations) and the corresponding height between the highest and lowest values of $f^k$'s oscillations. 
This improvement on the height enables stronger notions of function approximation (e.g., constant error rates with no dependence on $f$ or its period $p$). 
\textit{Chaotic itineraries} allow an improved analysis of the number of oscillations in $f^k$ and grant sharper exponential lower bounds on the width of any shallow net $g$ approximating $f^k$. 

We say that our results are \textit{nearly-optimal} because we exhibit a broad family of functions $f$ that are inapproximable by shallow networks of width $O(\rho^k)$ for $\rho$ arbitrarily close to 2.
Because no unimodal function $f$ can induce more than $2^k$ oscillations in $f^k$, we cannot aspire to tighter exponent bases in this setting.\footnote{Our results also transfer to non-unimodal functions via the observation that for bimodal $g$, there is some unimodal $f$ such that the number of oscillations of $g$ is at most twice those of $f$.}
On the other hand, none of the bounds from previous works (except the narrow bounds of Telgarsky) produce width bounds of more than $\Omega(\phi^k)$, where $\phi \approx 1.618$ is the Golden Ratio.
To demonstrate our sharper tradeoffs, we state a special case of our results for the $L_\infty$ error.

\begin{theorem}\label{thm:motivation-lb-inc} 
For $p \geq 3$ and $k \in \N$, consider any symmetric, concave unimodal mapping $f$ with an increasing $p$-cycle and any $g \in \mathcal{N}(u, \ell)$ with width \[u \leq \frac18 \left(\max\paren{2-\frac4{2^p}, \phi}\right)^{k/\ell}\]
Then, $L_\infty(f^k, g) = \Omega(1)$, independent of $f,p,k$.
\end{theorem}  
    
\begin{remark} When $g$ is shallow with depth $\ell=O(k^{1-\eps})$ (e.g., $\ell=k^{0.99}$), then its width must be exponentially large in order to well-approximate $f^k$.
This exponential separation is sharper than prior works~\citep{cnpw19,cnp20}, and quickly becomes even sharper (tending to 2) with larger values of $p$. 
This is counterintuitive as Sharkovsky's ordering implies that period 3 is the most chaotic and prior works recover a suboptimal rate of at most $\phi\approx 1.618$ (see Table~\ref{table:rho_inc}). 
\end{remark}
\begin{remark}
Our approximation error is constant independent of all other parameters $f, k, p$. 
Previous results~\citep{cnpw19,cnp20,bzl20} obtain a gap that depends on $f,p$ and may be arbitrarily small.
Moreover, we have required nothing of the Lipschitz constant of $f$, unlike the strict assumptions on the Lipschitz constant $L$ of $f$ by~\cite{cnp20} (e.g., they require $L=\phi$ for period $p=3$).
Indeed, Propositions~\ref{prop:need-symmetry} and \ref{prop:need-concavity} in the Appendix~\ref{assec:sym-conc} (see also Fig.~\ref{fig:asymmetric}, Fig.~\ref{fig:non-concave}), and Figure~\ref{fig:non-concave-intro} illustrate how their lower bounds break down for large $L$ and how their $L_\infty$ bounds can shrink, becoming arbitrarily weak for certain 3-periodic $f$.
\end{remark}

    \begin{figure}
        \centering
        \includegraphics[width=0.45\textwidth]{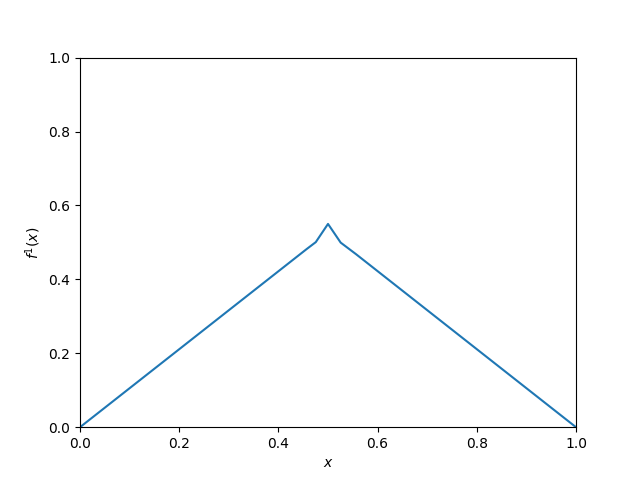}
         \includegraphics[width=0.45\textwidth]{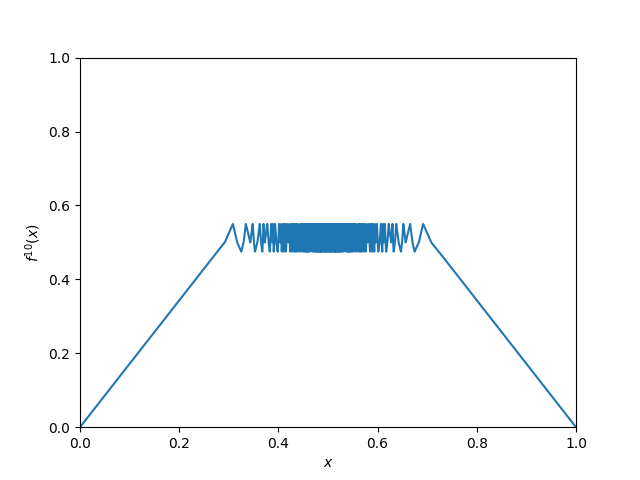}
        \caption{To illustrate the limitations of prior work \citep{cnpw19, cnp20, bzl20}, we visualize a unimodal mapping $f$ on the left and $f^{10}$ on the right where $f$ has a 3-cycle without $f^{10}$ having large oscillations.
        As guaranteed by the previous results, the total number of oscillations of $f^{k}$ is $\Omega(\phi^k)$; however, each is of magnitude at most $\eps$ for this construction and $f^k$ can be $\eps$-approximated by a two-layer constant-width neural network, by Proposition~\ref{prop:need-concavity}. Here, $\eps = 0.1$, but it may be arbitrarily small. Hence, to obtain non-trivial constant-accuracy lower bounds on the size of neural network needed to approximate $f^k$, it is insufficient to require only that $f$ has a 3-cycle.}
        \label{fig:non-concave-intro}
    \end{figure}

We also present analogous results for the classification error and the $L_1$ errors. Please see the full statements in Theorems~\ref{thm:main-linf} and \ref{thm:main-l1}. Furthermore, Theorems~\ref{thm:odd-linf} and \ref{thm:odd-l1} offer an improvement on the results of \cite{cnp20} by giving constant-accuracy $L_\infty$ lower bounds without needing a chaotic itinerary.

In addition, Section~\ref{sec:phase} relates our chaotic itineraries to standard notions of function complexity like the VC dimension and the topological entropy (for precise definitions, see Sec.~\ref{sec:phase}). The types of periodic itineraries of $f$ give rise to two regimes: the \textit{doubling} regime and the \textit{chaotic} regime. In the former, we have a polynomial number of oscillations, while the latter is characterized by an exponential number of oscillations. Here we show the following correspondence:

\begin{theorem}[Informal]\label{thm:informal-phase}
The transition between these two regimes exactly coincides with a sharp transition in the VC-dimension of the iterated mappings $f^k$ for fixed $f$ (from bounded to infinite) and in the topological entropy (from zero to positive).
\end{theorem}

\paragraph{Our Techniques} To quantify the oscillations of $f^k$, we use its chaotic itineraries to decompose the $[0,1]$ interval into several subintervals $\{I_j\}_{j=1}^{j=p-1}$. We count the number of times $f^k$ ``visits'' each $I_j$, by identifying a suitable matrix $A$ whose spectral radius is a lower bound on the growth rate of oscillations. The associated characteristic polynomial of $A$ is $\lambda^{p} - 2\lambda^{p-1} + 1$ and has larger spectral radius that that of prior works for \textit{all} periods. 
Moreover, the corresponding oscillations of at least one of the subintervals $I_j$ do not shrink in size, giving a bound on the total number of oscillations of a \textit{sufficient} size. 
This provides a lower bound on the height between the peak and the bottom of these oscillations that later provides \textit{constant} approximation errors for small shallow NNs.

More broadly, our work builds on the efforts to characterize large families of functions that give depth separations and addresses questions raised by~\cite{eldan2016COLT,telgarsky16,poole2016NIPS, malach2019deeper} about the properties of hard-to-represent functions. Similar to periods, the concept of chaotic itineraries can serve as a certificate of complexity, which is also easy to verify for unimodal $f$ (see Proposition~\ref{prop:half-inc-cycle} in Appendix).

\section{Warm-up Examples}\label{ssec:warmup}

This section presents illustrative examples and instantiates our results for some simple cases. 
These highlight the limitations of exclusively considering periodicity of cycles alone---and not itineraries---when developing accurate oscillation/crossing bounds~(see also Def.~\ref{def:osc},~\ref{def:cross}) and sharp expressivity tradeoffs.


Consider the three unimodal mappings in Figure~\ref{fig:toy-ex-compare-map}, $f_{\ba}$ with itineraries $\ba \in \{1324, 1234, 123\}$. 
Observe that $f_{1234}$ has the cycle $(\frac{1}{5}, \frac{2}{5}, \frac{3}{5}, \frac{4}{5})$, $f_{1324}$ has $(\frac{1}{5}, \frac{3}{5}, \frac{2}{5}, \frac{4}{5})$, and $f_{123}$ has $(\frac{1}{4}, \frac{1}{2}, \frac{3}{4})$. Despite their similarities, they give rise to significantly different behaviours in $f_{\ba}^k$.

What do prior works based on NN approximation with respect to periods and Sharkovsky's theorem alone tell us?  
\cite{cnpw19,cnp20} show that the 3-cycle of $f_{123}$ ensures that $f^k$ has $\Omega(\phi^k)$ oscillations, where $\phi \approx 1.618$ is the golden ratio. However, their theorems do not imply anything for $f_{1324}$ and $f_{1234}$, since 4 is a power of 2, and they require odd periods.  

\begin{figure}
    \centering
    \includegraphics[width=0.6\textwidth]{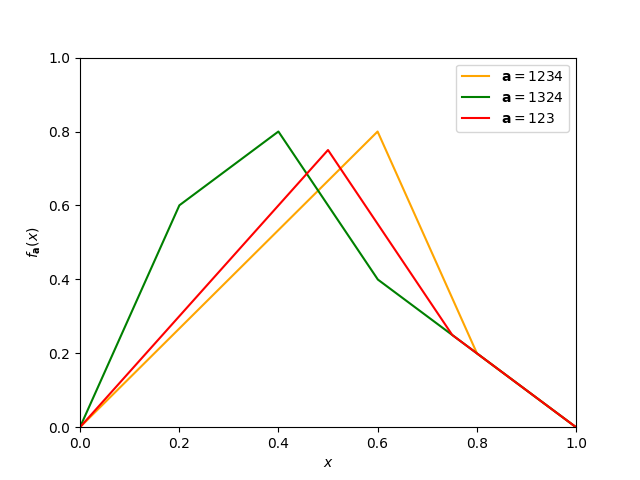}
    \caption{Plots of unimodal mappings with different itineraries $f_{1234}$, $f_{1324}$, and $f_{123}$. Despite their similarities, $f_{1234}$ leads to the most oscillations and sharpest depth-width tradeoffs (see Fig.~\ref{fig:toy-intro}).}
    \label{fig:toy-ex-compare-map}
\end{figure}
As it turns out, $f_{1234}$ leads to exponential oscillations and $f_{1324}$ leads only to polynomial oscillations:
\begin{itemize}
    \item A mapping with a 1324-itinerary is guaranteed no other cycles except the 2-cycle and a fixed point~\citep{mss73}. 
    Sharkovsky's theorem and \cite{cnpw19} predict this outcome, since 4 is the third-right-most element of the Sharkovsky ordering, and its existence alone promises nothing more.
    The ordering of itineraries introduced by~\citep{mss73} (see Table~\ref{table:itineraries} in Appendix) indicates that the particular 1324-itinerary only implies the periods 2 and 1, and confirms this intuition.
    We classify this itinerary as part of the \emph{doubling regime} and prove in Theorem~\ref{thm:properties-doubling} that any $f^k$ with a \textit{maximal} 1324-itinerary (that is, there is no 8-cycle) cannot exhibit sharp depth-width tradeoffs: for any $\eps>0$, there exists a 2-layer ReLU neural network $g$ of width $O(\frac{k^3}{\eps})$ such that 
    $L_\infty(f_{1324}^k, g) \leq \eps.$

    \item Going beyond Sharkovsky's theorem, a mapping with a 1234-itinerary---even though it is of period 4---it is guaranteed to contain a 3-cycle as well (see Table~\ref{table:itineraries} in Appendix). Hence, ``itinerary-1234 implies period-3, implies chaos,'' and $f_{1234}^k$ has at least $\Omega(\phi^k)$ oscillations and is hard to approximate by small shallow NNs. Moreover, Theorem~\ref{thm:main-linf} and Table~\ref{table:rho_inc} show that $f_{1234}^k$ actually has $\Omega(\rho^k)$ oscillations for $\rho \approx 1.839 > \phi$. A corollary is that any NN $g$ of depth $\sqrt{k}$ and width $O(1.839^{\sqrt{k}})$ has $L_{\infty}(f_{1234}^k, g) = \Omega(1)$, which is a stronger separation (\textit{constant} error) than the ones given by~\cite{cnpw19, cnp20}.
\end{itemize}
\begin{figure}
    \centering
    \includegraphics[width=0.6\textwidth]{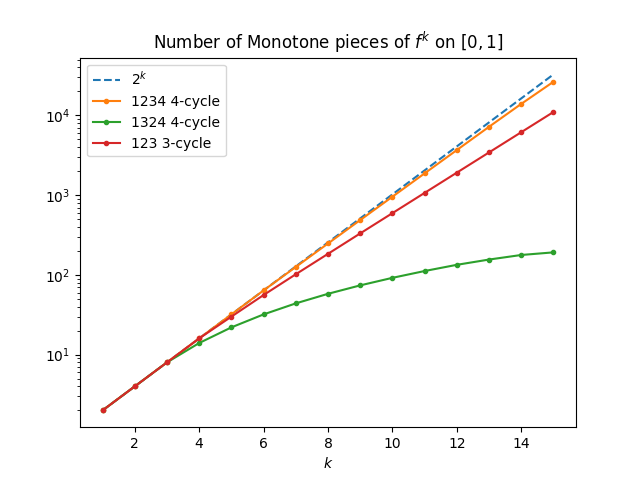}
    \caption{The chaotic itinerary $f_{1234}$ has more oscillations than $f_{123}$ even though $3 \triangleright 4$ by Sharkovsky's Theorem. Itineraries $f_{1234}$ and $f_{1324}$ (both of period 4) differ dramatically in oscillation count, showing why periodicity alone fails to capture the optimal tradeoffs.}
    \label{fig:toy-intro}
\end{figure}
The reverse is not true: Sharkovsky's Theorem guarantees that period-3 implies period-4, but the only 4-cycle guaranteed by the theorem is actually the non-chaotic 1324-itinerary, already shown to lead to minimal function complexity.

Furthermore, as $p$ increases, the existence of a chaotic itinerary $12\dots p$ on $f$ ensures that $f^k$ has $\Omega(\rho^k)$ oscillations for $\rho \to 2$.\footnote{Similarly to~\cite{telgarsky16}, the optimal achievable rate is $\rho\le 2$ if we start with a unimodal $f$ (e.g., tent map). If one used multimodal functions as a building block (e.g., starting with $f'=f^2$ or $f'=f^3$), we could achieve larger rates (e.g., 4 or 8 respectively).} Figure~\ref{fig:toy-intro} demonstrates these differences in oscillations (by counting the number of monotone pieces of functions $f^k_{\mathbf{a}}$ with a maximal itinerary-$\mathbf{a}$). As indicated theoretically, the number of oscillations of $f_{1324}$ is polynomially-bounded, while the others grow exponentially fast, with $f_{1234}$ being closer to $2^k$. Please see Appendix~\ref{asec:warmup} for more such examples. 

Generally, prior constructions where the oscillation count of $f^k$ increase at a rate faster than $\phi^k$ were too narrow (including only the triangle map). 
Because $f_{1234}$ breaks the barrier, we abstract away the details and point to chaotic itineraries as the main source of complexity, leading to sharper depth-width tradeoffs. 

While periodicity tells a compelling story about why $f_{123}^k$ is difficult to approximate, it fails to explain why $f_{1234}^k$ is even more complex.
The exponential-vs-polynomial gap in the function complexity of $f_{1234}$ and $f_{1324}$ depends solely on the order of the elements of the cycle and distinguishes functions that NNs can easily approximate from those they cannot.

The remainder of the paper addresses the question introduced here---when does the itinerary tell us much more than the length of the period---in a general context that explores a ``hierarchy'' of such chaotic itineraries, strengthens a host of NN inapproximability bounds (Sec.~\ref{sec:lb}), and reveals tight connections with other complexity notions, like the VC-dimension and topological entropy~(Sec.~\ref{sec:phase}).

%% file: lb.tex
\section{Depth-Width Tradeoffs via Chaotic Itineraries}\label{sec:lb}


We give our main hardness results on the inapproximability of functions generated by repeated compositions of $f$ to itself when $f$ has certain cyclic behavior.
Section~\ref{ssec:lb-increasing} applies insights about chaotic itineraries to prove constant $L_\infty$ and $L_1$ lower bounds on the accuracy of approximating $f^k$ when $f$ has an increasing cycle.
Section~\ref{ssec:lb-odd} strengthens previous bounds on the number of oscillations when $f$ has an odd cycle, which is not necessarily increasing.
Appendix~\ref{asec:compare} presents Table~\ref{table:compare} that illustrates the key differences between results.

\subsection{Notation}
To measure the function complexity of $f^k$, we count the number of times $f^k$ oscillates. 
We employ two notions of oscillation counts.
The first is relatively weak and counts every interval on which $f$ is either increasing or decreasing, regardless of its size.

\begin{definition}\label{def:osc}
    Let $f: [0, 1] \to [0, 1]$. $\oscm(f)$ represents the \emph{number of monotone pieces} of $f$. That is, it is the minimum $m$ such that there exists $x_0 = 0 < x_1 < \dots < x_{m-1} < x_m = 1$ where $f$ is monotone on $[x_{j-1}, x_{j}]$ for all $j\in [m]$.
\end{definition}

The second instead counts the number of times a fixed interval of size $b-a$ is crossed:

\begin{definition}\label{def:cross}
    Let $f: [0, 1] \to [0, 1]$ and $[a, b] \in [0,1]$. 
    $\osci{a,b}(f)$ represents the \emph{number of crossings} of $f$ on the interval $[a,b]$. 
    That is, it is the maximum $c$ such that there exist \[0 \leq x_1 < x'_1 \leq x_2 < x'_2 \leq \dots \leq x_c < x_c' \leq 1\] where for all $j \in [c]$, $f([x_j, x'_j]) \subset [a,b]$ and either $f(x_j) = a$ and $f(x'_j) = b$ or vice versa.
\end{definition}

\newcommand{\rhoinc}[1]{\rho_{\mathrm{inc}, #1}}
\newcommand{\rhoodd}[1]{\rho_{\mathrm{odd}, #1}}

\paragraph{Characteristic Polynomials} The base of the exponent of our width bounds is shown to equal the largest root of one of two polynomials:
\begin{align*}
    P_{\mathrm{inc}, p}(\lambda) &= \lambda^p - 2\lambda^{p-1} + 1, \\
    P_{\mathrm{odd}, p}(\lambda) &= \lambda^p - 2\lambda^{p-2} - 1.
\end{align*}
Let $\rhoinc{p}$ and $\rhoodd{p}$ be the largest roots of $P_{\mathrm{inc}, p}$ and $P_{\mathrm{odd}, p}$ respectively.
Table~\ref{table:rho_inc} illustrates that as $p$ grows, $\rhoinc{p}$ increases to 2, while $\rhoodd{p}$ drops to $\sqrt{2}$.
Note that $\rhoodd{p} \in (\sqrt{2}, \sqrt{2 + 2/2^{p/2}})$ \citep{alseda2000}.
We bound the growth rate of $\rhoinc{p}$ with the following: 

\begin{restatable}{fact}{factrootlb}\label{fact:root-lb}
    $\rhoinc{p} \in [\max(2 - \frac{4}{2^p}, \phi), 2)$, where $\phi = \frac{1+ \sqrt{5}}{2}$ is the Golden Ratio.
\end{restatable}

We prove Fact~\ref{fact:root-lb} in Appendix~\ref{assec:fact}.

\begin{table}[]
    \centering
     \caption{
     Approximate values of $\rhoinc{p}$, the lower bound on $\rhoinc{p}$ in Fact~\ref{fact:root-lb}, and $\rhoodd{p}$ (for odd $p$).\\}
    \begin{tabular}{|l|l|l|l|}
    \hline
    $p$  & $\rhoinc{p}$                              & Fact~\ref{fact:root-lb}   & $\rhoodd{p}$     \\ \hline\hline
    $3$  & $1.618$ & $1.618$ &  $1.618$  \\ \hline
    $4$  & $1.839$                        & $1.75$  & n/a \\ \hline
    $5$  & $1.928$                        & $1.875$ & $1.513$  \\ \hline
    $6$  & $1.966$                        & $1.938$ & n/a\\ \hline
    $7$  & $1.984$                        & $1.969$ & $1.466$\\ \hline
    $8$  & $1.992$                        & $1.984$ & n/a\\ \hline
    $9$  & $1.996$                        & $1.992$ & $1.441$\\ \hline
    $10$ & $1.999$                        & $1.996$ & n/a                 \\ \hline
    \end{tabular}
    \label{table:rho_inc}
\end{table}

\subsection{Inapproximability of Iterated Functions with Increasing Cycles}\label{ssec:lb-increasing}

Our inapproximability results that govern the size of neural network $g$ necessary to adequately approximate $f^k$ when $f$ has an increasing cycle (like Theorem~\ref{thm:motivation-lb-inc}) rely on a key lemma that bounds the number of constant-size oscillations of $f^k$.

\begin{restatable}[Oscillation Bound for Increasing Cycles]{lemma}{lemmaincreasingconstantosc}\label{lemma:increasing-constant-osc}
    Suppose $f$ is a symmetric, concave unimodal mapping with an increasing $p$-cycle for some $p \geq 3$.
    Then, there exists $[a,b] \subset [0, 1]$ with $b - a \geq \frac{1}{18}$ such that $\osci{a,b}(f^k) \geq \tfrac12\rhoinc{p}^k$ for all $k \in \N$.
\end{restatable}

We prove Lemma~\ref{lemma:increasing-constant-osc} in Appendix~\ref{assec:inc-constant-osc-proof}.
For an increasing $p$-cycle $x_1, \dots, x_p$, we lower-bound $\oscm(f^k)$ (the total number of monotone pieces, regardless of size) by relating the number of times $f^k$ crosses each interval $[x_j, x_{j+1}]$ to the number of crossings of $f^{k-1}$.
Doing so entails analyzing the largest eigenvalues of a transition matrix, which gives rise to the polynomial $P_{\mathrm{inc}, p}$.
We prove that the intervals crossed must be sufficiently large due to the symmetry, concavity, and unimodality of $f$.

\begin{remark}\label{remark:inc-close}
    If one does not wish to assume that $f$ is unimodal, symmetric, or concave, then the proof can be modified to show that $\osci{a,b}(f^k) = \Omega(\rho^k)$ for the same $\rho$, but for $a$ and $b$ dependent on $f$.
    These results are similar in flavor to those of \cite{cnpw19, cnp20, bzl20}, and they suffer from the same drawback: potentially vacuous approximation bounds when $a$ and $b$ are close.
    Appendix~\ref{assec:sym-conc} shows natural functions that are either not symmetric or not concave, whose oscillations shrink in size arbitrarily.
\end{remark}



\subsubsection{$L_\infty$ Approximation and Classification}
Our first result is a restatement of Theorem~\ref{thm:motivation-lb-inc} that quantifies inapproximability in terms of both $L_\infty$ and classification error, which are comparable to the respective results of \cite{bzl20} and \cite{cnpw19}.

\begin{theorem}\label{thm:main-linf}
    Suppose $f$ is a symmetric concave unimodal mapping with an increasing $p$-cycle for some $p \geq 3$.
    Then, any $k \in \N$ and $g\in \mathcal{N}(u, \ell)$ with $u \leq \frac{1}{8} \rhoinc{p}^{k/\ell}$ have
    $\norml[\infty]{f^k - g} = \Omega(1)$.
    
    Moreover, there exists $S$ with $\abs{S} = \frac{1}{2} \floorl{\rhoinc{p}^{k/\ell}}$ and $t \in (0,1)$ such that $\errcls{S,t}{f^k, g} \geq \frac{1}{4}.$
\end{theorem}
    The proof follows from our main Lemma~\ref{lemma:increasing-constant-osc} above and Theorem~\ref{thm:general-lb-cls}/Corollary~\ref{cor:general-lb-linf} in the Appendix (two previous inapproximability bounds based on oscillations).

Despite relying on unimodality assumptions and the existence of increasing cycles, Theorem~\ref{thm:main-linf} obtains much stronger bounds than its previous counterparts:
\begin{itemize}
    \item The assumption that $f$ has an increasing cycle causes a much larger exponent base for the width bound. 
    \cite{cnpw19, cnp20} only prove that the existence of 3-cycle mandates a width of $\Omega(\phi^{k/\ell})$.
    We exactly match that bound for $p=3$, and improve upon it when $p > 3$.
    As illustrated by Table~\ref{table:rho_inc}, increasing $p$ pushes the base $\rhoinc{p}$ rapidly to 2, which is the maximum exponent base for the increase of oscillations of any unimodal map. (And the maximal topological entropy of a unimodal map.)
    This also approximately matches the bases from \cite{bzl20}, which scale with the topological entropy of $f$.
    \item As illustrated in Appendix~\ref{assec:sym-conc}, the inaccuracy of neural networks with respect to the $L_\infty$ approximation in~\cite{cnpw19, cnp20,bzl20} may be arbitrarily small for certain choices of $f$. 
    Our unimodality assumptions ensure that the oscillations of $f^k$ are large and hence, that the inaccuracy of $g$ is constant.
\end{itemize}

\subsubsection{$L_1$ Approximation}
We also strengthen the bound on $L_1$-inapproximability given by \cite{cnp20} by again introducing a stronger exponent and applying unimodality to yield a constant-accuracy bound.

\begin{theorem}\label{thm:main-l1}
    Consider any $L$-Lipschitz $f: [0,1] \to [0,1]$ with an increasing $p$-cycle for some $p \geq 3$.
    If $L = \rhoinc{p}$, then for any $k \in \N$, any $g \in \mathcal{N}(u,\ell)$ with $u \leq \frac{1}{16} \rhoinc{p}^{k/\ell}$ has $\norml[1]{f^k - g} = \Omega(1).$
\end{theorem}

    The proof follows again using our main Lemma~\ref{lemma:increasing-constant-osc} and using Theorem~\ref{thm:general-lb-l1} in the Appendix.

We make Theorem~\ref{thm:main-l1} more explicit by showing that many tent maps meet the Lipschitzness condition.
Let $\ftent{r} = 2r\min(x, 1-x)$ be the tent map, parameterized by $r \in (0,1)$
Our result improves upon~\cite{cnp20}, by obtaining constant approximation error and using the larger $\rhoinc{p}$ rather than $\rhoodd{p}$.

\begin{restatable}{corollary}{cormainl}\label{cor:main-l1}
    For any $p \geq 3$ and $k \in \N$, any $g \in \mathcal{N}(u,\ell)$ with $u \leq \frac{1}{16} \rhoinc{p}^{k/\ell}$ has $\norml[1]{\ftent{\rhoinc{p}}^k - g} = \Omega(1).$
\end{restatable}

We prove Corollary~\ref{cor:main-l1} in Appendix~\ref{assec:cor-main-proof}. The only non-trivial part of the proof involves proving the existence of an increasing $p$-cycle that causes $f^k$ to have $\Omega(\rhoinc{p}^k)$ oscillations.


\subsection{Improved Bounds for Odd Periods}\label{ssec:lb-odd}

While Theorems~\ref{thm:main-linf} and \ref{thm:main-l1} give stricter bounds on the width of neural networks needed to approximate iterated functions $f^k$ than \cite{cnpw19, cnp20}, they also require extra assumptions about the cycles---namely, that the cycles are increasing.
However, more powerful inapproximability results with constant error are still possible even without additional assumptions.
Specifically, we leverage unimodality to improve the desired inaccuracy to a constant without compromising width.

As before, the results hinge on a key technical lemma that bounds the number of interval crossings.

\begin{restatable}{lemma}{lemmaoddconstantosc}\label{lemma:odd-constant-osc}
    For some odd $p \geq 3$, suppose $f$  is a symmetric concave unimodal mapping with an odd $p$-cycle.
    Then, there exists $[a,b] \subset [0, 1]$ with $b - a \geq 0.07$ such that $\osci{a,b}(f^k) = \rhoodd{p}^{k-p}$ for any $k \in \N$.
\end{restatable}

We prove Lemma~\ref{lemma:odd-constant-osc} in Appendix~\ref{assec:lemma-odd-proof}. The challenging part is to find a lower bound on the length of the intervals crossed.

Like before, we provide lower-bounds on approximation up to a constant degree.

\begin{theorem}\label{thm:odd-linf}
    For some odd $p \geq 3$, suppose $f$ is a symmetric, concave unimodal mapping with any $p$-cycle.
    Then, any $k \in \N$ and any $g \in \mathcal{N}(u,\ell)$ with $u \leq \frac{1}{8} \rhoodd{p}^{(k-p)/\ell}$ have
    $\norml[\infty]{f^k - g} = \Omega(1).$
    
    Moreover, there exists $S$ with $\abs{S} = \frac{1}{2} \floorl{\rhoodd{p}^k}$ and $t \in (0,1)$ such that 
    $\errcls{S,t}{f^k, g} \geq \frac{1}{4}.$
\end{theorem}
    The proof is immediate from Lemma~\ref{lemma:odd-constant-osc}, Theorem~\ref{thm:general-lb-cls}, and Corollary~\ref{cor:general-lb-linf} in the Appendix.

We also get the analogous result but for the $L_1$ error:

\begin{theorem}\label{thm:odd-l1}
    Consider any $L$-Lipschitz $f: [0,1] \to [0,1]$ with a $p$-cycle for some odd $p \geq 3$.
    If $L = \rhoodd{p}$, then, any $k \in \N$ and $g \in \mathcal{N}(u,\ell)$ with $u \leq \frac{1}{16} \rhoodd{p}^{(k-p)/\ell}$ have $\norml[1]{f^k - g} = \Omega(1).$
\end{theorem}

    The proof is immediate from Lemma~\ref{lemma:increasing-constant-osc} and Theorem~\ref{thm:general-lb-l1} in the Appendix.



%% file: phase.tex
\section{Periods, Phase Transitions and Function Complexity}\label{sec:phase}

We formalize the correspondence between different notions of function complexity in dynamical systems and learning theory: neural network approximation, oscillation count, cycle itinerary, topological entropy, and VC-dimension.
We make Theorem~\ref{thm:informal-phase} rigorous by presenting two regimes into which unimodal mappings can be classified---the \emph{doubling regime} and the \emph{chaotic regime}---and show that all of these measurements of complexity hinge on which regime a function belongs to.\footnote{These two regimes correspond to different settings of the parameters $r$ in the bifurcation diagram of Figure~\ref{fig:bifurcation} in the Appendix. The doubling regime is the left-hand-side, where the stable periods routinely split in two before the first chaos is encountered. The chaotic regime is to the right-hand-side, which is characterized by chaos punctuated by intermittent stability.}

Some components of the claims regarding the topological entropy are the immediate consequences of other results; however, we include them to give a complete picture of the gap between the two regimes.
To the best of our knowledge, we believe the bound on monotone pieces of $f^k$ in the doubling regime and both VC-dimension bounds below to be novel.

We define VC-dimension and introduce topological entropy in Appendix~\ref{asec:phase}, along with the proofs of both theorems.
For VC-dimension, we consider the hypothesis class $\hyp{f,t} := \{\thres{t}{f^k}: k \in \mathbb{N}\}$, which corresponds to the class of iterated fixed maps.

\begin{restatable}{theorem}{thmpropertiesdoubling}[Doubling Regime]\label{thm:properties-doubling}
    Suppose $f$ is a symmetric unimodal mapping whose maximal cycle is a primary cycle of length $p = 2^q$. That is, there exists a $p$-cycle but no $2p$-cycles (and thus, no cycles with lengths non-powers-of-two). Then, the following are true:
    \begin{enumerate}
        \item For any $k \in \N$, $\oscm(f^k) = O((4k)^{q+1})$.
        \item For any $k \in \N$, there exists $g \in \mathcal{N}(u,2)$ with $u= O((4k)^{q+1}/\eps)$ such that $\norm[\infty]{g - f^k} \leq \epsilon$. 
        Moreover, if $f = \ftent{r}$, then there exists $g \in \mathcal{N}(u,2)$ with $u = O((4k)^{q+1})$ and $g = f^k$.
        \item $\htop(f) = 0$.
        \item For any $t \in (0, 1)$, $\vc(\hyp{f, t}) \leq 18p^2$.
    \end{enumerate}
\end{restatable}

\begin{restatable}{theorem}{thmpropertieschaotic}[Chaotic Regime]\label{thm:properties-chaotic}
    Suppose $f$ is a unimodal mapping that has a $p$-cycle where $p$ is not a power-of-two. Then, the following are true:
    \begin{enumerate}
        \item There exists some $\rho \in (1, 2]$ such that for any $k \in \N$, $\oscm(f^k) = \Omega(\rho^k)$.
        \item For any $k \in \N$ and any $g \in \mathcal{N}(u,\ell)$ with $\ell \leq k$ and $u \leq \frac{1}{8} \rho^{k/\ell}$, there exist samples $S$ with $\abs{S}= \frac12 \floor{\rho^k}$ such that $\errcls{S, 1/2}{f^k, g} \geq \frac14$.
        \item $\htop(f) \geq \rho > 0$.
        \item There exists a $t \in (0,1)$ such that $\vc(\hyp{f, t}) = \infty$.
    \end{enumerate}
\end{restatable}

\begin{remark}\label{remark:non-primary}
    As discussed in Appendix~\ref{asec:itineraries}, any non-primary cycle implies the existence of a cycle whose length is not a power of two. Thus, these results also apply if there exists any non-primary power-of-two cycle, such as the 1234-itinerary 4-cycle.
\end{remark}



%% file: supp_compare.tex
\section{Comparison with Prior Works}\label{asec:compare}

Given the large number of results presented in this paper and the many axes of comparison one can draw between these results and their predecessors in \cite{telgarsky16, cnpw19, cnp20}, we provide Table~\ref{table:compare} to illuminate these comparisons.
It reinforces our key contributions, namely that (1) the presence of increasing cycles makes a function more difficult to approximate than a 3-cycle alone; (2) requiring that $f$ satisfy unimodality constraints gives lower-bounds to constant accuracy that cannot be made vacuous by adversarial choices of $f$; and (3) the key distinction between ``hard'' and ``easy'' functions is the existence of non-primary power-of-two cycles.

\begin{table}[]
{\renewcommand{\arraystretch}{1.25}
\scalebox{0.72}{
\begin{tabular}{l|l|l|l|l|l|l|l|l|l|p{3.8cm}|}
\cline{2-11}
                                  & \textbf{Condition}                        & \textbf{Approx.}    & \textbf{Unimodal?} & \textbf{Concave?} & \textbf{Symmetric?} & \textbf{$L\leq\rho$?} & \textbf{Acc.}          & \textbf{Exp.}               & \textbf{Hard?} & \textbf{Source}                                                                      \\ \hline
\multicolumn{1}{|l|}{\textbf{1}}  & \textbf{Maximal PO2}                      & \textbf{$L_\infty$} & \textbf{Yes}       & \textbf{No}       & \textbf{Yes}        & \textbf{No}        & \textbf{$\Omega(1)$}   & \textbf{Any}                & \textbf{No}    & \textbf{Thm~\ref{thm:properties-doubling}}                                           \\ \hline
\multicolumn{1}{|l|}{\textbf{2}}  & $\htop(f) \geq \rho$                      & $L_\infty$          & No                 & No                & No                  & No                 & $\epsilon(f)$          & $\rho$                      & Yes            &BZL~Thm~16                                                                  \\ \hline
\multicolumn{1}{|l|}{\textbf{3}}  & \textbf{Non-primary}                      & \textbf{Cls.}       & \textbf{No}        & \textbf{No}       & \textbf{No}         & \textbf{No}        & \textbf{$\frac{1}{4}$} & \textbf{$(1, \phi]$}        & \textbf{Yes}   & \textbf{CNPW~Thm~1.6, Remark~\ref{remark:non-primary}}                      \\ \hline
\multicolumn{1}{|l|}{\textbf{4}}  & \textbf{Non-primary}                      & \textbf{$L_\infty$} & \textbf{No}        & \textbf{No}       & \textbf{No}         & \textbf{No}        & \textbf{$\epsilon(f)$} & \textbf{$(1, \phi]$}        & \textbf{Yes}   & \textbf{CNPW~Thm~1.6, Remark~\ref{remark:non-primary}, BZL~Thm~16} \\ \hline
\multicolumn{1}{|l|}{\textbf{5}}  & Non-PO2                                   & Cls.                & No                 & No                & No                  & No                 & $\frac{1}{4}$          & $(1, \phi]$                 & Yes            & CNPW~Thm~1.6                                                                \\ \hline
\multicolumn{1}{|l|}{\textbf{6}}  & Non-PO2                                   & $L_\infty$          & No                 & No                & No                  & No                 & $\epsilon(f)$          & $(1, \phi]$                 & Yes            & CNPW~Thm~1.6, BZL~Thm~16                                           \\ \hline
\multicolumn{1}{|l|}{\textbf{7}}  & Odd cycle                                 & Cls.                & No                 & No                & No                  & No                 & $\frac{1}{4}$          & $(\sqrt{2}, \phi]$          & Yes            & CNP~Thm~1.1                                                                 \\ \hline
\multicolumn{1}{|l|}{\textbf{8}}  & Odd cycle                                 & $L_\infty$          & No                 & No                & No                  & No                 & $\epsilon(f)$          & $(\sqrt{2}, \phi]$          & Yes            & CNP~Thm~1.1, BZL~Thm~16                                            \\ \hline
\multicolumn{1}{|l|}{\textbf{9}}  & \textbf{Odd cycle}                        & \textbf{$L_\infty$} & \textbf{Yes}       & \textbf{Yes}      & \textbf{Yes}        & \textbf{No}        & \textbf{$\Omega(1)$}   & \textbf{$(\sqrt{1}, \phi]$} & \textbf{Yes}   & \textbf{Thm~\ref{thm:odd-linf}}                                                      \\ \hline
\multicolumn{1}{|l|}{\textbf{10}} & Odd cycle                                 & $L_1$               & No                 & No                & No                  & Yes                & $\epsilon(f)$          & $(\sqrt{2}, \phi]$          & Yes            & CNP~Thm~1.2                                                                 \\ \hline
\multicolumn{1}{|l|}{\textbf{11}} & {$\ftent{\rho_p/2}$} & {$L_1$}      & {Implied}   & {Implied}  & {Implied}    & {Implied}   & {$\Omega(1)$}   & {$(\sqrt{2}, \phi]$}        & {Yes}   & {CNP~Lemma~3.6}                                                       \\ \hline

\multicolumn{1}{|l|}{\textbf{12}} & \textbf{Odd cycle}                        & \textbf{$L_1$}      & \textbf{Yes}       & \textbf{Yes}      & \textbf{Yes}        & \textbf{Yes}       & \textbf{$\Omega(1)$}   & \textbf{$(\sqrt{2}, \phi]$} & \textbf{Yes}   & \textbf{Thm~\ref{thm:odd-l1}}                                                        \\ \hline
\multicolumn{1}{|l|}{\textbf{13}} & \textbf{Inc. Cycle}                       & \textbf{Cls.}       & \textbf{No}        & \textbf{No}       & \textbf{No}         & \textbf{No}        & \textbf{$\frac{1}{4}$} & \textbf{$[\phi,2)$}         & \textbf{Yes}   & \textbf{Thm~\ref{thm:main-linf}, Remark~\ref{remark:inc-close}}                      \\ \hline
\multicolumn{1}{|l|}{\textbf{14}} & \textbf{Inc. Cycle}                       & \textbf{$L_\infty$} & \textbf{No}        & \textbf{No}       & \textbf{No}         & \textbf{No}        & \textbf{$\epsilon(f)$} & \textbf{$[\phi,2)$}         & \textbf{Yes}   & \textbf{Thm~\ref{thm:main-linf}, Remark~\ref{remark:inc-close}}                      \\ \hline
\multicolumn{1}{|l|}{\textbf{15}} & \textbf{Inc. Cycle}                       & \textbf{$L_\infty$} & \textbf{Yes}       & \textbf{Yes}      & \textbf{No}         & \textbf{No}        & \textbf{$\Omega(1)$}   & \textbf{$[\phi,2)$}         & \textbf{No}    & \textbf{Prop~\ref{prop:need-symmetry}}                                               \\ \hline
\multicolumn{1}{|l|}{\textbf{16}} & \textbf{Inc. Cycle}                       & \textbf{$L_\infty$} & \textbf{Yes}       & \textbf{No}       & \textbf{Yes}        & \textbf{No}        & \textbf{$\Omega(1)$}   & \textbf{$[\phi,2)$}         & \textbf{No}    & \textbf{Prop~\ref{prop:need-concavity}}                                              \\ \hline
\multicolumn{1}{|l|}{\textbf{17}} & \textbf{Inc. Cycle}                       & \textbf{$L_\infty$} & \textbf{Yes}       & \textbf{Yes}      & \textbf{Yes}        & \textbf{No}        & \textbf{$\Omega(1)$}   & \textbf{$[\phi,2)$}         & \textbf{Yes}   & \textbf{Thm~\ref{thm:main-linf}}                                                     \\ \hline
\multicolumn{1}{|l|}{\textbf{18}} & \textbf{Inc. Cycle}                       & \textbf{$L_1$}      & \textbf{No}        & \textbf{No}       & \textbf{No}         & \textbf{Yes}       & \textbf{$\epsilon(f)$} & \textbf{$[\phi,2)$}         & \textbf{Yes}   & \textbf{Thm~\ref{thm:main-l1}, CNP~Thm~1.2}                                 \\ \hline
\multicolumn{1}{|l|}{\textbf{19}} & \textbf{Inc. Cycle}                       & \textbf{$L_1$}      & \textbf{Yes}       & \textbf{Yes}      & \textbf{Yes}        & \textbf{Yes}       & \textbf{$\Omega(1)$}   & \textbf{$[\phi,2)$}         & \textbf{Yes}   & \textbf{Thm~\ref{thm:main-l1}}                                                       \\ \hline
\multicolumn{1}{|l|}{\textbf{20}} & \textbf{$\ftent{\rho_p/2}$} & \textbf{$L_1$}      & \textbf{Implied}   & \textbf{Implied}  & \textbf{Implied}    & \textbf{Implied}   & \textbf{$\Omega(1)$}   & \textbf{$[\phi, 2)$}        & \textbf{Yes}   & \textbf{Cor~\ref{cor:main-l1}}                                                       \\ \hline
\multicolumn{1}{|l|}{\textbf{21}} & $\ftent{1}$                               & $L_1$               & Implied            & Implied           & Implied             & Implied            & $\Omega(1)$            & 2                           & Yes            & Telgarsky                                                                  \\ \hline
\end{tabular}}}
\caption{Compares the conditions and limitations of the theoretical results presented in this paper and its predecessors. New results are bolded.}
\label{table:compare}
\end{table}

We provide context for each column to clarify what its cells mean and how to compare their values.

\begin{itemize}
    \item ``\textbf{Condition}'' specifies what must be true of the complexity of $f$ in order for the relevant bounds to occur. All but the latter two conditions describe a very broad array of functions, while the last two focus only on a restricted subset of tent mappings.
    \begin{itemize}
        \item ``Maximal PO2'' means that the maximal cycle of $f$ is a primary\footnote{See Appendix~\ref{assec:order}.} $p$-cycle where $p$ is a power of two. This means that $f$ lies in the doubling regime described in Theorem~\ref{thm:properties-doubling}.
        \item ``$\htop(f) \geq \rho$'' considers any $f$ with a lower-bound on its topological entropy for some $\rho > 1$. Notably, all conditions other than ``Maximal PO2'' satisfy this for some $\rho$.
        \item ``Non-primary'' means that any non-primary cycle exists in $f$. That is, if $f$ is known to have a non-primary power-of-two cycle, then the results apply. 
        \item ``Non-PO2'' refers to any $f$ that has a $p$-cycle where $p$ is not a power of two.
        \item ``Odd cycle'' includes any $f$ that has a $p$-cycle where $p$ is odd.
        \item ``Inc. cycle'' means that $f$ has an increasing $p$-cycle for some $p$, i.e. a cycle with itinerary $12\dots p$.
        \item $\ftent{\rho_p/2}$ refers to families of tent maps scaled by $\rho_p$ solving the polynomials from \cite{cnp20} Lemma~3.6 (for odd periods) and Corollary~\ref{cor:main-l1} (for increasing cycles).
        \item The last row refers exclusively to the tent map of height 1 and slope 2.
    \end{itemize}

    \item ``\textbf{Approx.}'' refers to how difference between neural network $g$ and iterated map $f^k$ is measured. 
    The options are $L_1$, $L_\infty$, and classification error.
    It's easier to show that $g$ can $L_1$-approximate $f^k$ than it is to show that $g$ can $L_\infty$-approximate $f$; conversely, it's most impressive to show lower bound results with respect to the $L_1$ error than it is for the $L_\infty$ error.
    
    \cite{cnpw19, cnp20} consider classification error, \cite{bzl20} focus on $L_\infty$ approximation, and \cite{cnp20} also consider $L_1$ approximation. We routinely translate classification errors to $L_\infty$ errors using Corollary~\ref{cor:general-lb-linf}, which draws on Theorem~16 of \cite{bzl20}.
    
    \item ``\textbf{Unimodal?},'' ``\textbf{Concave?},'' and ``\textbf{Symmetric?},'' have ``Yes'' if and only if $f$ must meet the respective property for the proof to hold. 
    They have ``Implied'' if the value of ``Condition'' already ensures that the property is satisfied and the requirement need not be enforced.
    
    \item ``$L \leq \rho$?'' is ``Yes'' if the results only hold if $f$ is chosen with a Lipschitz constant less than the rate of growth of its oscillations. This is a very restrictive condition met by very few functions (including no logistic maps with cycles).
    
    \item ``\textbf{Acc.}'' specifies the desired accuracy of the hardness result. ``$\Omega(1)$'' means that there exists some constant $\epsilon$ such that for any choice of $f$ in the category, any neural network $g$ will be unable to approximate $f$ up to accuracy $\epsilon$.
    ``$\eps(f)$'' means that the degree of approximation may depend on the chosen function $f$ (and the period $p$) that belongs to the category; these bounds may be vacuous by an adversarial choice of $f$.
    As a result, hardness results with ``$\Omega(1)$'' are more impressive.
    
    \item ``\textbf{Exp.}'' refers to the base of the exponent of the lower-bound on the width necessary to approximate $f^k$ using a shallow network $g$.
    Larger values indicate stronger bounds.
    
    \item ``\textbf{Hard?}'' is ``Yes'' if for every $f$ satisfying the conditions to the left, $f$ cannot be approximated up to the specified accuracy by any neural network $g$.
    It is ``No'' if there exists some $f$ satisfying the conditions that can be approximated to a stronger degree of accuracy.
    
    \item ``\textbf{Source}'' denotes where to find the result. Some of the less interesting results are not given their own theorems and rather are immediate implications of several theorems across this body of literature. For the sake of space, we use ``CNPW'' to refer to \citep{cnpw19}; ``CNP'' for \citep{cnpw19}; ``BZL'' for \citep{bzl20}; and ``Telgarsky'' for \citep{telgarsky16}.
\end{itemize}

%% file: discussion.tex
\section{Conclusion}

In this work, we build new connections between deep learning theory and dynamical systems by applying results from discrete-time dynamical systems 
to obtain novel depth-width tradeoffs for the expressivity of neural networks. While prior works relied on Sharkovsky's theorem and periodicity to provide families of functions that are hard-to-approximate with shallow neural networks, we go beyond periodicity. Studying the chaotic itineraries of unimodal mappings, we reveal subtle connections between expressivity and different types of periods, and we use them to shed new light on the benefits of depth in the form of enhanced width lower bounds and stronger approximation errors. More broadly, we believe that it is an exciting direction for future research to exploit similar tools and concepts from the literature of dynamical systems in order to improve our understanding of neural networks, e.g., their dynamics, optimization and robustness properties.

%% file: supp_toy.tex
\section{Supplement for Section~\ref{ssec:warmup}}\label{asec:warmup}

Figures~\ref{fig:toy-ex} and \ref{fig:toy-ex-logistic} demonstrate two emblematic cases where the differences in function complexity of $f_{123}$, $f_{1234}$, and $f_{1324}$ are most evident.
Both figures provide a function for each $f_{\mathbf{a}}$ that has a maximal itinerary of $\mathbf{a}$. (That is, there is no ``higher-ranked'' itinerary from Table~\ref{table:itineraries} present in $f_{\mathbf{a}}$; all other cycles are induced by the existence of a cycle with itinerary $\mathbf{a}$.)

Figures~\ref{fig:toy-ex} and \ref{fig:toy-ex-monotone} provide a simple case where the elements of the cycles are evenly spaced ($\frac14,\frac12, \frac34$ for $f_{123}$; $\frac15, \frac25, \frac35, \frac45$ for $f_{1234}, f_{1324}$).
Despite the fact that $f_{1234}$ and $f_{1324}$ have the same maximum value, they exhibit substantially different fractal-like patterns, which produce exponentially more oscillations for $f_{1234}$.

Figure~\ref{fig:toy-ex-logistic} and \ref{fig:toy-ex-logistic-monotone} instead considers logistic maps of the form $\flog{r}(x) = 4rx(1-x)$ for the values of $r$ where itinerary $\mathbf{a}$ is \emph{super-stable}, or when nearby iterates converge to the cycle exponentially fast.
These functions are concave, symmetric, and unimodal.
Here, complexity strictly increases with the maximum value of $\flog{r}$. Indeed, $f_{1234}, f_{123}$ and $f_{1324}$ ordered by height is the order by which they exhibit most to least chaotic behavior.

\begin{figure}[h]
    \centering
    \includegraphics[width=0.32\textwidth]{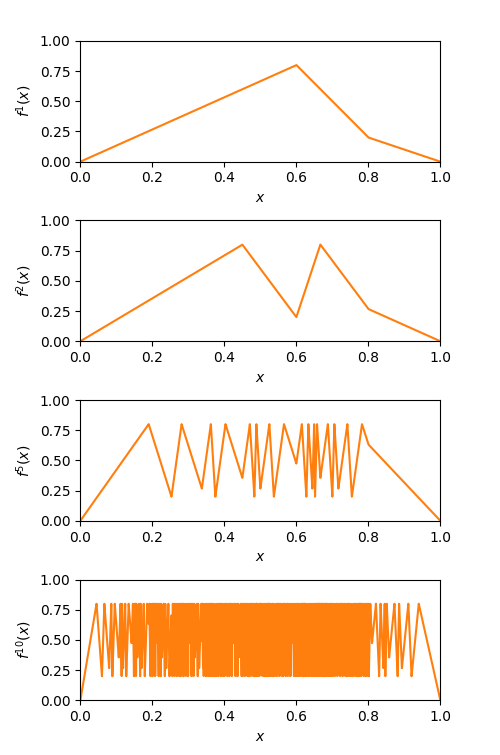}
    \includegraphics[width=0.32\textwidth]{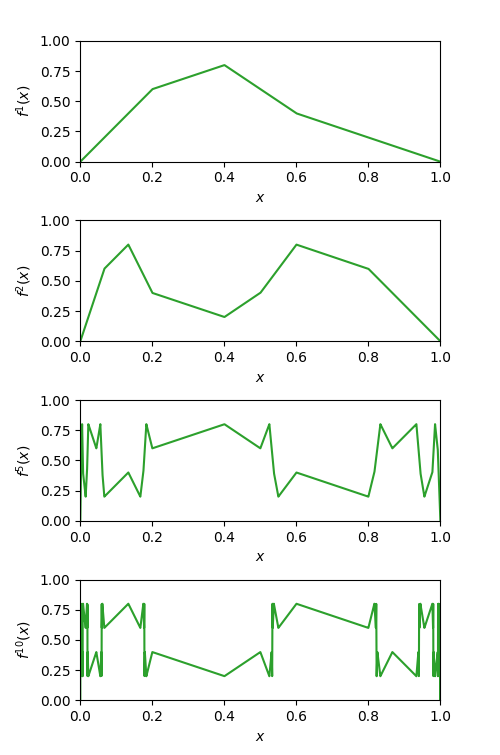}
    \includegraphics[width=0.32\textwidth]{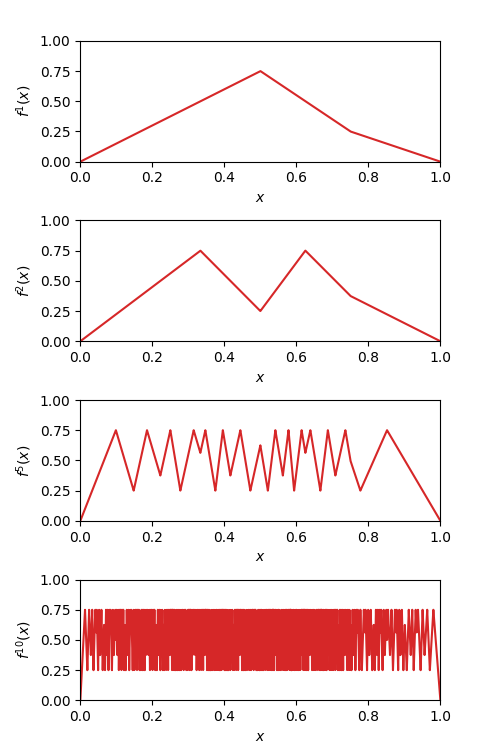}
    \caption{A comparison of the function complexity (as measured by the number of monotone pieces) of $f^k$ for unimodal mappings $f$ having cycles with different itineraries. The left shows $f$, $f^2$, $f^5$, and $f^{10}$ for a function with a 1234 4-cycle. The center has a 1324 4-cycle. The bottom has a 123 3-cycle. 
    Figure~\ref{fig:toy-ex-monotone} shows how the number of monotone pieces of $f^k$ increases with $k$ for each mapping.}
    \label{fig:toy-ex}
\end{figure}

\begin{figure}
    \centering
    \includegraphics[width=0.45\textwidth]{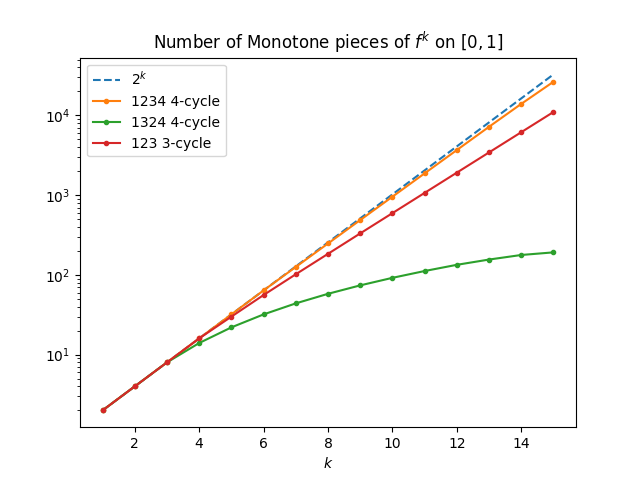}
    \caption{Visualizes the number of monotone pieces of $f^k$ which increases with $k$ for each mapping along with $2^k$ (the maximum number of monotone pieces of any unimodal $f$).
    Note that the 1234 itinerary produces a more ``complex'' function with more monotone pieces than 123, despite the Sharkovsky analysis from \cite{cnpw19} arguing that 3-cycles are the most powerful when determining iteration counts.
    Moreover, the number of monotone pieces of the  1234 and 123 itineraries increases exponentially, while that of the 1324 itineraries does not. (Identical to Figure~\ref{fig:toy-intro}.)}
    \label{fig:toy-ex-monotone}
\end{figure}

\begin{figure}
    \centering
    \includegraphics[width=0.32\textwidth]{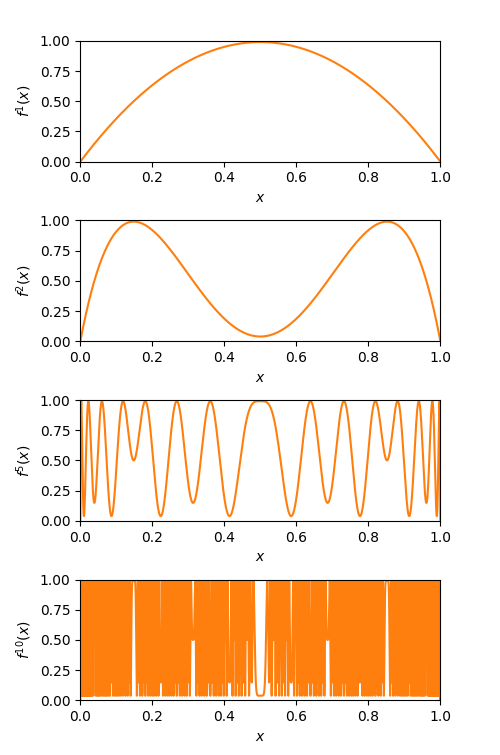}
    \includegraphics[width=0.32\textwidth]{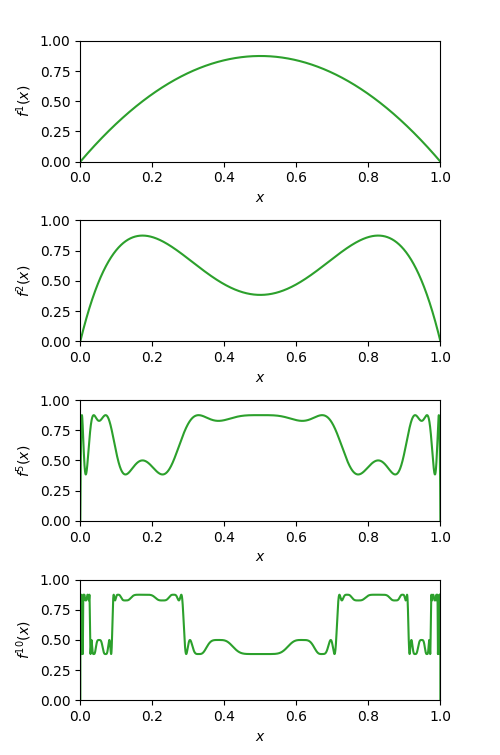}
    \includegraphics[width=0.32\textwidth]{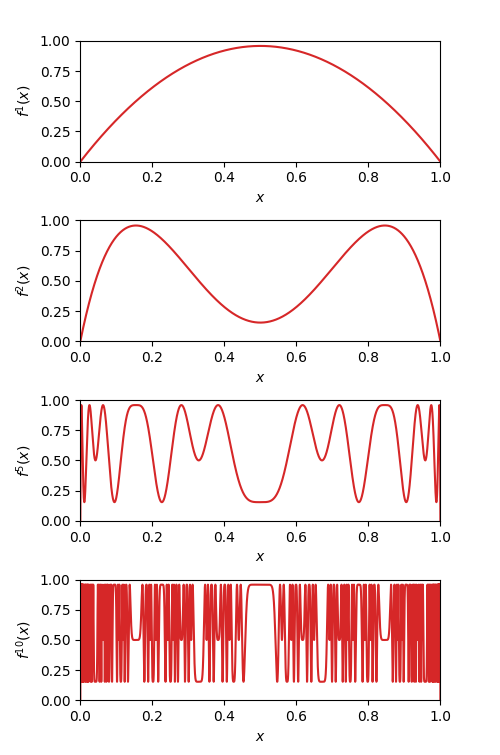}
    \caption{Demonstrates the same ideas as Figure~\ref{fig:toy-ex}, except instead of using asymmetric and non-concave piecewise functions, we use the scaled logistic map, $\flog{r}$.
    Using Table~1 of \cite{mss73}, we set the parameter $r$ to $3.96$, $3.50$, and $3.83$ respectively to ensure that a super-stable 1234, 1324, and 123 cycle exists.}
    \label{fig:toy-ex-logistic}
\end{figure}

\begin{figure}
    \centering
        \includegraphics[width=0.45\textwidth]{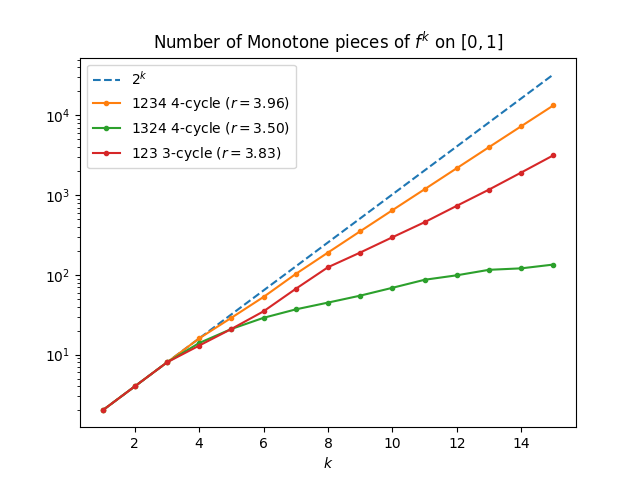}

    \caption{Like Figure~\ref{fig:toy-ex-monotone}, vizualizes the differences in number of monotone pieces for the logistic mappings described in Figure~\ref{fig:toy-ex-logistic}.}
    \label{fig:toy-ex-logistic-monotone}
\end{figure}

%% file: supp_itineraries.tex
\section{More Examples for Itineraries}\label{asec:itineraries}

\subsection{Examples of Itineraries}\label{assec:it-exs}

Let the tent map and logistic map be defined by $\ftent{r}(x) = 2r\max(x, 1-x)$ and $\flog{r}(x) = 4rx(1-x)$ respectively, for parameter $r \in (0,1)$.

\begin{example}\label{ex:two-cycle}
    For all $r \in (\frac{1}{2}, 1]$, there is a two-cycle $C$ of itinerary 12 (which is the only itinerary for a 2-cycle) in $\ftent{r}$ with
    \[C = \paren{\frac{2r}{1 + 4r^2}, \frac{4r^2}{1 + 4r^2}}.\]
\end{example}

\begin{example}
    When $r = \frac{1+ \sqrt{5}}{4}$, there is a two-cycle $C$ of $\flog{r}$ with 
    \[C = \paren{ \frac{1}{2}, \frac{1+ \sqrt{5}}{4}}.\]
\end{example}

\begin{example}
    When $r \in [\frac{1+ \sqrt{5}}{4}, 1]$, $\ftent{r}$ has a three-cycle $C$ of itinerary 123 with 
    \[C= \paren{\frac{2r}{1 + 8r^3}, \frac{4r^2}{1 + 8r^3}, \frac{8r^3}{1 + 8r^3}}.\] 
    Note that this and Example~\ref{ex:two-cycle} are consistent with Sharkovsky's Theorem; whenever there exists a three-cycle, there also exists a two-cycle.
\end{example}

\begin{example}
    When $r \in [\frac{1}{2}, 1]$, there also exists a four-cycle $C$ of itinerary $1324$ for $\ftent{r}$ with 
    \[C = \paren{ \frac{8r^3 - 4r^2 + 2r}{16r^2 + 1},\frac{16r^4 - 8r^3 + 4r^2}{16r^2 + 1},\frac{16r^4 - 8r^3 + 2r}{16r^2 + 1}, \frac{16r^4 - 4r^2 + 2r}{16r^2 + 1}}.\]
    Again, this reaffirms Sharkovsky's Theorem, since this cycle always exists when the above three-cycle exists.
\end{example}

\begin{example}
    However, when $r \in (0.9196\dots, 1]$, there also exists a four-cycle $C$ of itinerary $1234$ for $\ftent{r}$ with 
    \[C= \paren{\frac{2r}{16r^2 + 1}, \frac{4r^2}{16r^2 + 1}, \frac{8r^3 }{16r^2 + 1}, \frac{16r^4}{16r^2 + 1}}.\] 
    This demonstrates a relationship \textit{beyond} Sharkovsky's theorem: whenever a $1234$ four-cycle exists, a $123$ three-cycle also exists. This will be integral to the bounds we show.
\end{example}

\begin{example}
    The triangle map from~\cite{telgarsky16}, $\ftent{1}$ has an increasing $p$-cycle $C_p$ for every $p \in \N$ with
    \[C_p = \paren{\frac{2}{1 + 2^p}, \frac{2^2}{1 + 2^p}, \dots, \frac{2^p}{1 + 2^p}}.\]
    Thus Theorem~\ref{thm:general-lb-cls} and Fact~\ref{fact:root-lb} retrieve the fact used by Telgarsky that $M(\ftent{1}) = \Omega(2^k)$.
\end{example}

\subsection{Orderings of Itineraries}\label{assec:order}
As has been mentioned before, the existence of some cycles can be shown to imply the existence of other cycles. 
Sharkovsky's Theorem famously does this by showing that if $p \triangleright p'$, then the existence of a $p$-cycle implies the existence of a $p'$-cycle.
Proposition~\ref{prop:half-inc-cycle} can be used to imply that the existence of a chaotic $p$-cycle implies the existence of a chaotic $(p-1)$-cycle.
These pose a broader question: Is there a complete ordering on all cycle itineraries that can appear in unimodal mappings? 
And does this ordering coincide with the amount of ``chaos'' induced by a cycle?

Researchers of discrete dynamical systems have thoroughly investigated these questions; we refer interested readers to~\cite{mss73, alseda2000} for a more comprehensive survey.
We introduce the basics of this theory as it relates to our results.

\cite{mss73} present a partial ordering over cyclic itineraries present in unimodal mappings, which serves as a measurement of the complexity of the function.
That is, two itineraries $\mathbf{a}$ and $\mathbf{a}'$ may be related analogously to Sharkovsky's Theorem with $\mathbf{a} \triangleright \mathbf{a}'$, if $f$ having itinerary $\mathbf{a}$ implies that $f$ has itinerary $\mathbf{a}'$.
This ordering for all cycles of length at most 6 is illustrated in Table~\ref{table:itineraries}.
For instance, if a unimodal map has a cycle with itinerary 12435, then it also has a cycle with itinerary 135246.

\begin{table}[]
    \centering
        \caption{For any unimodal function $f$, let $f_r(x) := rf(x)$ for $r > 0$. As $r$ increases, any such family obtains new cycles in the same order, and those cycles are super-stable in the same order. This translates Table~1 of \cite{mss73} to our notation and shows at what values of $r$, $\flog{r}$ has various super-stable cycles of length at most 6.\\}
    \begin{tabular}{|l|l|l|l|l|} \hline
        Cycle length $p$ & Itinerary  & Regime & $r$ s.t. super-stable for $\flog{r}$ & Cycle Type \\\hline\hline
        2 & 12 & Doubling & $0.8090$ & Primary \\ \hline
        4 & 1324 & Doubling & $0.8671$ & Primary \\ \hline
        6 & 143526 & Chaotic & $0.9069$ &  Primary \\ \hline
        5 & 13425 & Chaotic & $0.9347$ & Stefan, Primary \\ \hline
        3 & 123 & Chaotic & $0.9580$ & Stefan, Increasing, Primary \\ \hline
        6 & 135246 & Chaotic & $0.9611$ & \\ \hline
        5 & 12435 & Chaotic & $0.9764$ & \\ \hline
        6 & 124536 & Chaotic & $0.9844$ & \\ \hline
        4 & 1234 & Chaotic & $0.9901$ & Increasing \\ \hline
        6 & 123546 & Chaotic & $0.9944$ & \\ \hline
        5 & 12345 & Chaotic & $0.9976$ & Increasing \\ \hline
        6 & 123456 & Chaotic & $0.9994$ & Increasing \\ \hline
    \end{tabular}
    \label{table:itineraries}
\end{table}

We make several observations about the table and make connections to the itineraries discussed elsewhere in the paper.
\begin{itemize}
    \item The table does not contradict Sharkovsky's Theorem. 
    Note that $3 \triangleright 5 \triangleright 6 \triangleright 4 \triangleright 2$, and order in which the first itinerary occurs of a  period is the same as the Sharkovsky ordering:
    \[12 \triangleleft 1324 \triangleleft 143526 \triangleleft 13425 \triangleleft 123 .\] 
    \item The last cycle to occur for a given period is its increasing cycle and it occurs as $p$ increases (not with the Sharkovsky ordering of $p$):
    \[12 \triangleleft 123 \triangleleft 1234 \triangleleft 12345 \triangleleft 123456.\]
    \item The first cycle to appear for every odd period is its \textit{Stefan cycle} (123, 13425). 
    This is proved by~\cite{alseda2000} and justifies why Theorem~\ref{thm:odd-linf} relies on the existence of a Stefan cycle whenever there is an odd period.
    \item There exist cycles of power-of-two length (e.g. 1234) that induce non-power-of-two cycles (e.g. 123).
\end{itemize}

Following the last bullet point, we distinguish between the $2^q$-cycles that only induce cycles of length $2^{i}$ for $i < q$ and those that induce non-power-of-two cycles.
To do so, we say that the itinerary of a $p$-cycle is \textit{primary} if it induces no other $p$-cycle with a different itinerary.

We say that an itinerary $\ba'=a_1'\dots a_{2p}'$ of a $2p$-cycle is a \textit{2-extension} of itinerary $\ba = a_1\dots a_p$ of a $p$-cycle if \[a_i = \ceil{\frac{a_i'}{2}} = \ceil{\frac{a_{i+p}'}{2}}\] for all $i$.
For instance, $12$ is a 2-extension of $1$, $1324$ is of $12$, $15472638$ is of $1324$, and $135246$ is of $123$. 

Theorem 2.11.1 of \cite{alseda2000} characterizes which itineraries are primary.
It critically shows that a power-of-two cycle is primary if and only if it is composed of iterated 2-extensions of the trivial fixed-point itinerary 1. 
As a result, 1324 is a primary itinerary and 1234 is not.
This sheds further light on the warmup example given in Section~\ref{ssec:warmup} and expanded upon in Appendix~\ref{asec:warmup}, where $f_{1324}^k$ has a polynomial number of oscillations, while $f_{1234}^k$ has an exponential number.

According to Theorem~2.12.4 of \cite{alseda2000}, the existence a non-primary itinerary of any period implies the existence of some cycle with period not a power of two. 
Hence, $f$ can \textit{only} be in the doubling regime (where all periods are powers of two) if all of those power-of-two periods are primary.
The existence of any non-primary power-of-two period (such as 1234 or 13726548) implies that the $f$ is in the chaotic regime.

This ordering can also be visualized using the bifurcation diagrams in Figure~\ref{fig:bifurcation}.
The diagram plots the convergent behavior of $f^k_r(x)$ for large $k$, where $r$ is some parameter and reflects the complexity of the unimodal function $f_r$. 
(When $r = 0$, $f_r = 0$; when $r = 1$, $\xmax = 1$, and $\osci{0,1}(f^k) = 2^k$.)
As $r$ increases, the number of oscillations of $f_r^k$ increases and with it, new cycles are introduced.
Each new cycle has a \textit{stable} region over parameters $r$ where $f_r^k(x)$ converges to the cycle, and the bifurcation diagram visualizes when each of these stable regions occurs.
While the three functions families $f_r$ have different underlying unimodal functions, they produce qualitatively identical bifurcation diagrams that feature the same ordering of itineraries.

\begin{figure}
    \centering
    \includegraphics[width=0.32\textwidth]{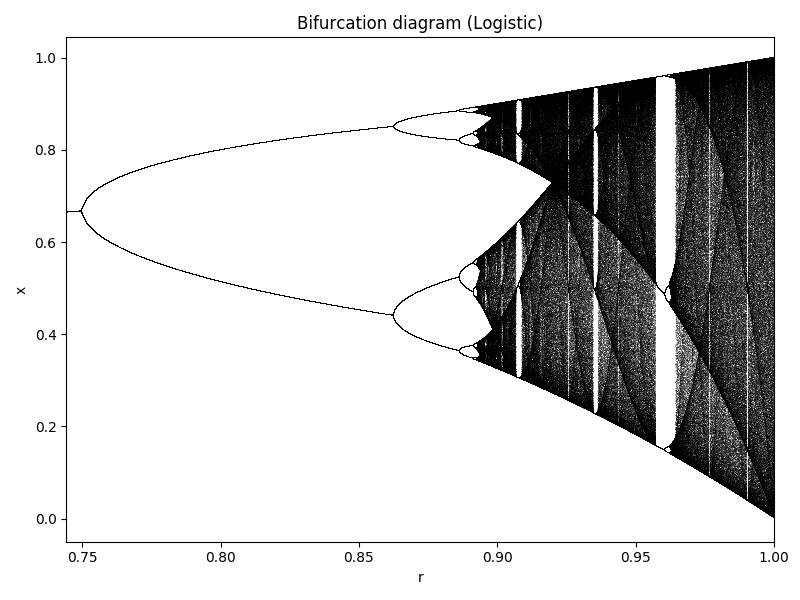}
    \includegraphics[width=0.32\textwidth]{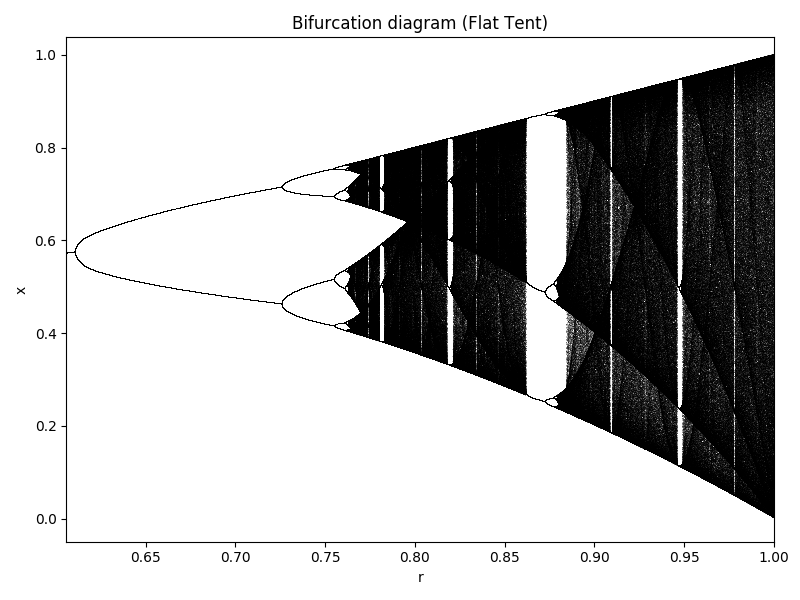}
    \includegraphics[width=0.32\textwidth]{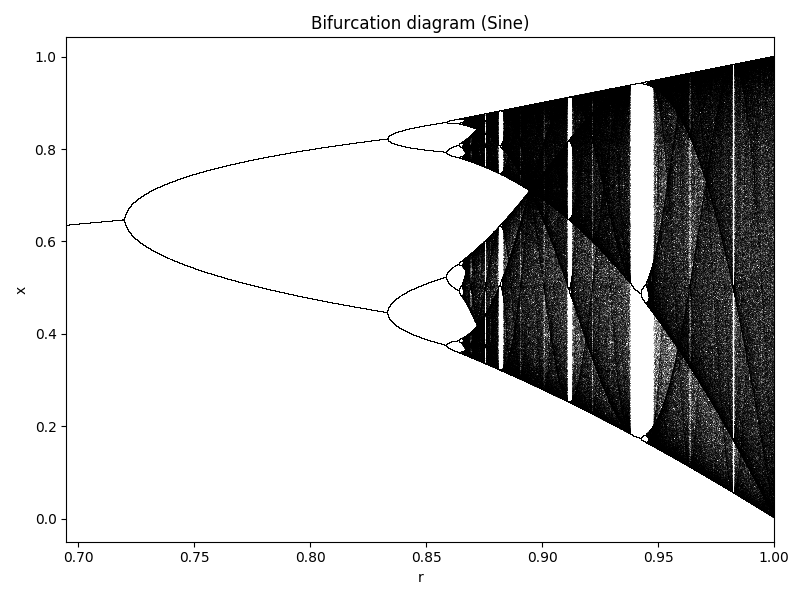}
    \caption{Bifurcation diagrams---which display the qualitative behavior of a family of functions $f_r$ as the parameter $r\in[0,1]$ changes---showing the convergence behavior for iterates $f_r^k(x)$ for large $k$. For fixed $r$ on the horizontal axis, the points plotted correspond to $f^k(x_0)$ for very large $k$. Regions of $r$ where a vertical slice contains $p$ discrete points indicates the existence of a \textit{stable} $p$-cycle, since $f^k(x_0)$ converges exclusively to those points. 
    Regions where the slice has a dispersed mass of points exhibit chaos. As $r$ increases, cycles of different itineraries appear and experience stability in the same order indicated by Table~\ref{table:itineraries}. In the first plot, $f_r$ is the logistic map $f_r(x) = \flog{r}(x) = 4rx(1-x)$. The second $f_r$ is the``flat tent map,'' $f_r(x) = \min\{\frac{5rx}{2}, r, \frac{5rx}{2}(1-x)\}$, and the third is the sine map, $f_r(x) = r \sin(\pi x)$.
    The three are qualitatively identical and exhibit self-similarity.}
    \label{fig:bifurcation}
\end{figure}

Our discussions of the \textit{doubling} and \textit{chaotic} regimes in Section~\ref{sec:phase} are inspired by these bifurcation diagrams.
Parameter values $r$ are naturally partitioned into two categories: those on the left side of the diagram where the plot is characterized by a branching of cycles (the doubling regime) and those on the right side where there are extended regions of chaos, interrupted by small stable regions (the chaotic regime).

\subsection{Identifying Increasing Cycles in Unimodal Maps}\label{assec:identifying}
It is straightforward to determine whether a symmetric and unimodal $f$ has an increasing $p$-cycle.
Algorithmically, one can do so by verifying that $f(\frac{1}{2}) > \frac{1}{2}$ and counting how many consecutive values of $k \geq 2$ satisfy $f^k(x_0) < \frac{1}{2}$.
\begin{proposition}\label{prop:half-inc-cycle}
    Consider some $p \geq 2$ and a symmetric unimodal mapping $f$.
    $f$ has an increasing $p$-cycle if \[f^2\paren{\frac{1}{2}} < \dots < f^p\paren{\frac{1}{2}} \leq \frac{1}{2} < f\paren{\frac12},\] then $f$ has an increasing $p$-cycle.
\end{proposition}

\begin{proof}
    
    Refer to Figure~\ref{fig:inc-cycle} for a visualization of the variables and inequalities defined.

    Let $x' = f(\frac{1}{2})$.
    By the unimodality of $f$ and the fact that $x' > \frac{1}{2}$, there exists some $x'' > \frac{1}{2}$ such that \[f(x'') < f^2(x'') < \dots < f^{p-1}(x'') = \frac{1}{2}.\]
    
    Because $f$ is monotonically increasing on $[0, \frac{1}{2}]$, the following string of inequalities hold.
    \begin{equation}\label{eq:interleave}
        f(x') \leq f(x'') < f^2(x') \leq f^2(x'') <
        \dots < f^{p-1}(x') \leq f^{p-1}(x'') = \frac{1}{2}
    \end{equation}
    It then must hold that $x' \geq x''$.
    
    Let $g(x) = f^p(x) - x$ and note that $g$ is continuous. 
    Because $\frac{1}{2}$ maximizes $f$, it must be the case that $f^p(x') \leq x'$ and $g(x') \leq 0$.
    Because $f^{p}(x'') = x'$ and $x'' \leq x'$, $g(x'') \geq 0$.
    Hence, there exists $x^* \in [x'', x']$ such that $g(x^*) = 0$ and $f^p(x^*) = x^*$.
    
    Since $x^* \in [x'', x']$, it must also be the case that $f^j(x^*) \in [f^j(x'), f^j(x'')]$ for $j \in [p-1]$.
    By Equation~\eqref{eq:interleave}, it follows that
    \[f(x^*) < f^2(x^*) < \dots < f^{p-1}(x^*) < f^p(x^*) = x^*.\]
    Hence, there exists an increasing $p$-cycle.
\end{proof}

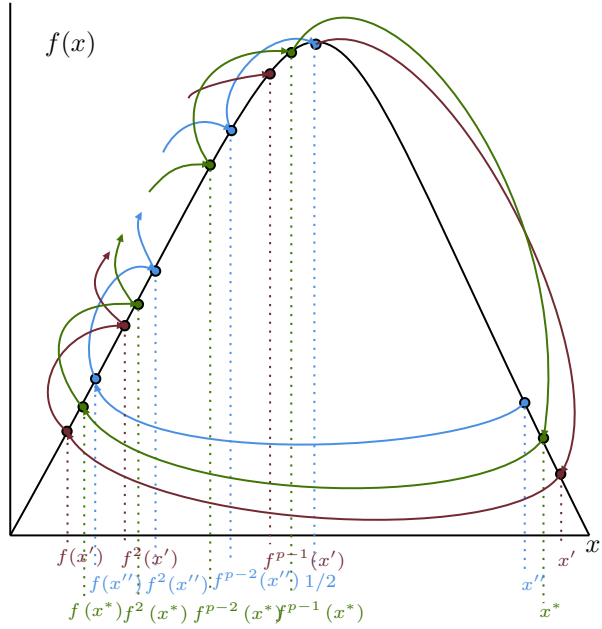
\begin{figure}
    \centering
    \input{fig/fig-inc-cycle}
    \caption{Visualizes the proof of Proposition~\ref{prop:half-inc-cycle}.}
    \label{fig:inc-cycle}
\end{figure}

%% file: fig/fig-inc-cycle.tex
\tikzset{every picture/.style={line width=0.75pt}} 

\begin{tikzpicture}[x=0.75pt,y=0.75pt,yscale=-0.45,xscale=0.45]

\draw    (5.5,6.15) -- (5.5,603.61) ;
\draw    (654.5,603.61) -- (5.5,603.61) ;
\draw    (5.5,603.61) .. controls (422.51,-142.86) and (288.85,-125.93) .. (654.5,603.61) ;
\draw  [fill={rgb, 255:red, 74; green, 144; blue, 226 }  ,fill opacity=1 ] (342.84,52.66) .. controls (342.84,49.76) and (345.19,47.41) .. (348.09,47.41) .. controls (350.99,47.41) and (353.35,49.76) .. (353.35,52.66) .. controls (353.35,55.56) and (350.99,57.91) .. (348.09,57.91) .. controls (345.19,57.91) and (342.84,55.56) .. (342.84,52.66) -- cycle ;
\draw [color={rgb, 255:red, 74; green, 144; blue, 226 }  ,draw opacity=1 ] [dash pattern={on 0.84pt off 2.51pt}]  (346.34,55.58) -- (346.34,641.67) ;
\draw  [fill={rgb, 255:red, 74; green, 144; blue, 226 }  ,fill opacity=1 ] (576.88,454.78) .. controls (576.88,451.88) and (579.23,449.53) .. (582.13,449.53) .. controls (585.03,449.53) and (587.38,451.88) .. (587.38,454.78) .. controls (587.38,457.68) and (585.03,460.03) .. (582.13,460.03) .. controls (579.23,460.03) and (576.88,457.68) .. (576.88,454.78) -- cycle ;
\draw [color={rgb, 255:red, 74; green, 144; blue, 226 }  ,draw opacity=1 ] [dash pattern={on 0.84pt off 2.51pt}]  (582.13,460.03) -- (582.13,644.33) ;
\draw [color={rgb, 255:red, 74; green, 144; blue, 226 }  ,draw opacity=1 ]   (253.54,149.54) .. controls (227.96,103.55) and (298.93,26.42) .. (345.96,51.44) ;
\draw [shift={(348.09,52.66)}, rotate = 211.32999999999998] [fill={rgb, 255:red, 74; green, 144; blue, 226 }  ,fill opacity=1 ][line width=0.08]  [draw opacity=0] (8.93,-4.29) -- (0,0) -- (8.93,4.29) -- cycle    ;
\draw  [fill={rgb, 255:red, 74; green, 144; blue, 226 }  ,fill opacity=1 ] (95.96,427.93) .. controls (95.96,425.03) and (98.31,422.68) .. (101.22,422.68) .. controls (104.12,422.68) and (106.47,425.03) .. (106.47,427.93) .. controls (106.47,430.84) and (104.12,433.19) .. (101.22,433.19) .. controls (98.31,433.19) and (95.96,430.84) .. (95.96,427.93) -- cycle ;
\draw  [fill={rgb, 255:red, 74; green, 144; blue, 226 }  ,fill opacity=1 ] (163.08,307.12) .. controls (163.08,304.22) and (165.43,301.87) .. (168.33,301.87) .. controls (171.23,301.87) and (173.59,304.22) .. (173.59,307.12) .. controls (173.59,310.02) and (171.23,312.37) .. (168.33,312.37) .. controls (165.43,312.37) and (163.08,310.02) .. (163.08,307.12) -- cycle ;
\draw  [fill={rgb, 255:red, 74; green, 144; blue, 226 }  ,fill opacity=1 ] (248.29,149.54) .. controls (248.29,146.64) and (250.64,144.29) .. (253.54,144.29) .. controls (256.45,144.29) and (258.8,146.64) .. (258.8,149.54) .. controls (258.8,152.44) and (256.45,154.79) .. (253.54,154.79) .. controls (250.64,154.79) and (248.29,152.44) .. (248.29,149.54) -- cycle ;
\draw [color={rgb, 255:red, 74; green, 144; blue, 226 }  ,draw opacity=1 ]   (176.8,174.05) .. controls (180.82,164.28) and (205.8,122.23) .. (251.44,148.3) ;
\draw [shift={(253.54,149.54)}, rotate = 211.32999999999998] [fill={rgb, 255:red, 74; green, 144; blue, 226 }  ,fill opacity=1 ][line width=0.08]  [draw opacity=0] (8.93,-4.29) -- (0,0) -- (8.93,4.29) -- cycle    ;
\draw [color={rgb, 255:red, 74; green, 144; blue, 226 }  ,draw opacity=1 ]   (168.33,307.12) .. controls (144.57,264.4) and (146.95,252.46) .. (151.55,243.14) ;
\draw [shift={(152.87,240.59)}, rotate = 477.9] [fill={rgb, 255:red, 74; green, 144; blue, 226 }  ,fill opacity=1 ][line width=0.08]  [draw opacity=0] (8.93,-4.29) -- (0,0) -- (8.93,4.29) -- cycle    ;
\draw [color={rgb, 255:red, 74; green, 144; blue, 226 }  ,draw opacity=1 ]   (101.22,427.93) .. controls (75.63,381.94) and (119.98,281.59) .. (166.22,305.92) ;
\draw [shift={(168.33,307.12)}, rotate = 211.32999999999998] [fill={rgb, 255:red, 74; green, 144; blue, 226 }  ,fill opacity=1 ][line width=0.08]  [draw opacity=0] (8.93,-4.29) -- (0,0) -- (8.93,4.29) -- cycle    ;
\draw [color={rgb, 255:red, 74; green, 144; blue, 226 }  ,draw opacity=1 ]   (582.13,454.78) .. controls (531.32,509.95) and (132.37,531.6) .. (101.65,434.66) ;
\draw [shift={(101.22,433.19)}, rotate = 434.62] [fill={rgb, 255:red, 74; green, 144; blue, 226 }  ,fill opacity=1 ][line width=0.08]  [draw opacity=0] (8.93,-4.29) -- (0,0) -- (8.93,4.29) -- cycle    ;
\draw [color={rgb, 255:red, 74; green, 144; blue, 226 }  ,draw opacity=1 ] [dash pattern={on 0.84pt off 2.51pt}]  (100.71,425.93) -- (100.71,641.49) ;
\draw [color={rgb, 255:red, 74; green, 144; blue, 226 }  ,draw opacity=1 ] [dash pattern={on 0.84pt off 2.51pt}]  (168.33,307.12) -- (168.33,644.33) ;
\draw [color={rgb, 255:red, 74; green, 144; blue, 226 }  ,draw opacity=1 ] [dash pattern={on 0.84pt off 2.51pt}]  (252.38,149.54) -- (252.38,635) ;
\draw  [fill={rgb, 255:red, 114; green, 41; blue, 50 }  ,fill opacity=1 ] (617.15,534.74) .. controls (617.15,531.84) and (619.5,529.49) .. (622.4,529.49) .. controls (625.3,529.49) and (627.65,531.84) .. (627.65,534.74) .. controls (627.65,537.64) and (625.3,539.99) .. (622.4,539.99) .. controls (619.5,539.99) and (617.15,537.64) .. (617.15,534.74) -- cycle ;
\draw [color={rgb, 255:red, 114; green, 41; blue, 50 }  ,draw opacity=1 ][fill={rgb, 255:red, 114; green, 41; blue, 50 }  ,fill opacity=1 ] [dash pattern={on 0.84pt off 2.51pt}]  (622.98,537.66) -- (622.98,615.86) ;
\draw [color={rgb, 255:red, 114; green, 41; blue, 50 }  ,draw opacity=1 ]   (348.09,52.66) .. controls (372.31,41.7) and (399.49,46.98) .. (427.24,63.64) .. controls (549.65,137.14) and (683.2,432.24) .. (623.32,533.23) ;
\draw [shift={(622.4,534.74)}, rotate = 301.98] [fill={rgb, 255:red, 114; green, 41; blue, 50 }  ,fill opacity=1 ][line width=0.08]  [draw opacity=0] (8.93,-4.29) -- (0,0) -- (8.93,4.29) -- cycle    ;
\draw  [fill={rgb, 255:red, 114; green, 41; blue, 50 }  ,fill opacity=1 ] (63.86,486.88) .. controls (63.86,483.98) and (66.22,481.63) .. (69.12,481.63) .. controls (72.02,481.63) and (74.37,483.98) .. (74.37,486.88) .. controls (74.37,489.78) and (72.02,492.13) .. (69.12,492.13) .. controls (66.22,492.13) and (63.86,489.78) .. (63.86,486.88) -- cycle ;
\draw [color={rgb, 255:red, 114; green, 41; blue, 50 }  ,draw opacity=1 ][fill={rgb, 255:red, 114; green, 41; blue, 50 }  ,fill opacity=1 ] [dash pattern={on 0.84pt off 2.51pt}]  (69.7,489.8) -- (69.7,615.86) ;
\draw  [fill={rgb, 255:red, 114; green, 41; blue, 50 }  ,fill opacity=1 ] (128.65,368.4) .. controls (128.65,365.5) and (131,363.15) .. (133.9,363.15) .. controls (136.8,363.15) and (139.15,365.5) .. (139.15,368.4) .. controls (139.15,371.3) and (136.8,373.66) .. (133.9,373.66) .. controls (131,373.66) and (128.65,371.3) .. (128.65,368.4) -- cycle ;
\draw [color={rgb, 255:red, 114; green, 41; blue, 50 }  ,draw opacity=1 ][fill={rgb, 255:red, 114; green, 41; blue, 50 }  ,fill opacity=1 ] [dash pattern={on 0.84pt off 2.51pt}]  (133.9,368.4) -- (133.9,612.36) ;
\draw  [fill={rgb, 255:red, 114; green, 41; blue, 50 }  ,fill opacity=1 ] (291.48,85.93) .. controls (291.48,83.02) and (293.83,80.67) .. (296.73,80.67) .. controls (299.63,80.67) and (301.99,83.02) .. (301.99,85.93) .. controls (301.99,88.83) and (299.63,91.18) .. (296.73,91.18) .. controls (293.83,91.18) and (291.48,88.83) .. (291.48,85.93) -- cycle ;
\draw [color={rgb, 255:red, 114; green, 41; blue, 50 }  ,draw opacity=1 ][fill={rgb, 255:red, 114; green, 41; blue, 50 }  ,fill opacity=1 ] [dash pattern={on 0.84pt off 2.51pt}]  (296.73,85.93) -- (296.73,613.53) ;
\draw [color={rgb, 255:red, 114; green, 41; blue, 50 }  ,draw opacity=1 ]   (69.12,486.88) .. controls (19.71,446.44) and (61.55,366.52) .. (131.76,368.32) ;
\draw [shift={(133.9,368.4)}, rotate = 182.8] [fill={rgb, 255:red, 114; green, 41; blue, 50 }  ,fill opacity=1 ][line width=0.08]  [draw opacity=0] (8.93,-4.29) -- (0,0) -- (8.93,4.29) -- cycle    ;
\draw [color={rgb, 255:red, 114; green, 41; blue, 50 }  ,draw opacity=1 ]   (131.56,366.07) .. controls (83.41,326.64) and (101.83,307.33) .. (115.77,287.14) ;
\draw [shift={(117.27,284.94)}, rotate = 483.69] [fill={rgb, 255:red, 114; green, 41; blue, 50 }  ,fill opacity=1 ][line width=0.08]  [draw opacity=0] (8.93,-4.29) -- (0,0) -- (8.93,4.29) -- cycle    ;
\draw [color={rgb, 255:red, 114; green, 41; blue, 50 }  ,draw opacity=1 ]   (205.39,113.36) .. controls (206.54,103.63) and (253.08,90.55) .. (294.22,86.18) ;
\draw [shift={(296.73,85.93)}, rotate = 534.4100000000001] [fill={rgb, 255:red, 114; green, 41; blue, 50 }  ,fill opacity=1 ][line width=0.08]  [draw opacity=0] (8.93,-4.29) -- (0,0) -- (8.93,4.29) -- cycle    ;
\draw  [fill={rgb, 255:red, 65; green, 117; blue, 5 }  ,fill opacity=1 ] (597.89,494.47) .. controls (597.89,491.57) and (600.24,489.22) .. (603.14,489.22) .. controls (606.04,489.22) and (608.39,491.57) .. (608.39,494.47) .. controls (608.39,497.37) and (606.04,499.72) .. (603.14,499.72) .. controls (600.24,499.72) and (597.89,497.37) .. (597.89,494.47) -- cycle ;
\draw [color={rgb, 255:red, 65; green, 117; blue, 5 }  ,draw opacity=1 ][fill={rgb, 255:red, 114; green, 41; blue, 50 }  ,fill opacity=1 ] [dash pattern={on 0.84pt off 2.51pt}]  (603.14,499.72) -- (603.14,679.67) ;
\draw  [fill={rgb, 255:red, 65; green, 117; blue, 5 }  ,fill opacity=1 ] (81.96,459.45) .. controls (81.96,456.55) and (84.31,454.2) .. (87.21,454.2) .. controls (90.11,454.2) and (92.46,456.55) .. (92.46,459.45) .. controls (92.46,462.35) and (90.11,464.7) .. (87.21,464.7) .. controls (84.31,464.7) and (81.96,462.35) .. (81.96,459.45) -- cycle ;
\draw [color={rgb, 255:red, 65; green, 117; blue, 5 }  ,draw opacity=1 ][fill={rgb, 255:red, 114; green, 41; blue, 50 }  ,fill opacity=1 ] [dash pattern={on 0.84pt off 2.51pt}]  (88.44,460.7) -- (88.44,673.39) ;
\draw [color={rgb, 255:red, 114; green, 41; blue, 50 }  ,draw opacity=1 ]   (622.4,534.74) .. controls (564.04,621.85) and (124.61,592.81) .. (69.91,488.46) ;
\draw [shift={(69.12,486.88)}, rotate = 424.2] [fill={rgb, 255:red, 114; green, 41; blue, 50 }  ,fill opacity=1 ][line width=0.08]  [draw opacity=0] (8.93,-4.29) -- (0,0) -- (8.93,4.29) -- cycle    ;
\draw [color={rgb, 255:red, 65; green, 117; blue, 5 }  ,draw opacity=1 ]   (603.14,494.47) .. controls (544.78,581.58) and (142.33,565.25) .. (88,461.03) ;
\draw [shift={(87.21,459.45)}, rotate = 424.2] [fill={rgb, 255:red, 65; green, 117; blue, 5 }  ,fill opacity=1 ][line width=0.08]  [draw opacity=0] (8.93,-4.29) -- (0,0) -- (8.93,4.29) -- cycle    ;
\draw  [fill={rgb, 255:red, 65; green, 117; blue, 5 }  ,fill opacity=1 ] (143.82,344.47) .. controls (143.82,341.57) and (146.17,339.22) .. (149.07,339.22) .. controls (151.97,339.22) and (154.33,341.57) .. (154.33,344.47) .. controls (154.33,347.38) and (151.97,349.73) .. (149.07,349.73) .. controls (146.17,349.73) and (143.82,347.38) .. (143.82,344.47) -- cycle ;
\draw [color={rgb, 255:red, 65; green, 117; blue, 5 }  ,draw opacity=1 ][fill={rgb, 255:red, 114; green, 41; blue, 50 }  ,fill opacity=1 ] [dash pattern={on 0.84pt off 2.51pt}]  (149.07,349.73) -- (149.07,673) ;
\draw  [fill={rgb, 255:red, 65; green, 117; blue, 5 }  ,fill opacity=1 ] (224.36,188.17) .. controls (224.36,185.27) and (226.71,182.92) .. (229.62,182.92) .. controls (232.52,182.92) and (234.87,185.27) .. (234.87,188.17) .. controls (234.87,191.07) and (232.52,193.42) .. (229.62,193.42) .. controls (226.71,193.42) and (224.36,191.07) .. (224.36,188.17) -- cycle ;
\draw [color={rgb, 255:red, 65; green, 117; blue, 5 }  ,draw opacity=1 ][fill={rgb, 255:red, 114; green, 41; blue, 50 }  ,fill opacity=1 ] [dash pattern={on 0.84pt off 2.51pt}]  (229.62,193.42) -- (229.62,668.33) ;
\draw  [fill={rgb, 255:red, 65; green, 117; blue, 5 }  ,fill opacity=1 ] (315.41,62.11) .. controls (315.41,59.21) and (317.76,56.85) .. (320.66,56.85) .. controls (323.56,56.85) and (325.91,59.21) .. (325.91,62.11) .. controls (325.91,65.01) and (323.56,67.36) .. (320.66,67.36) .. controls (317.76,67.36) and (315.41,65.01) .. (315.41,62.11) -- cycle ;
\draw [color={rgb, 255:red, 65; green, 117; blue, 5 }  ,draw opacity=1 ][fill={rgb, 255:red, 114; green, 41; blue, 50 }  ,fill opacity=1 ] [dash pattern={on 0.84pt off 2.51pt}]  (320.46,67.36) -- (320.46,670.33) ;
\draw [color={rgb, 255:red, 65; green, 117; blue, 5 }  ,draw opacity=1 ]   (84.58,457.7) .. controls (42.77,402.53) and (47.76,336.96) .. (147.56,344.36) ;
\draw [shift={(149.07,344.47)}, rotate = 184.61] [fill={rgb, 255:red, 65; green, 117; blue, 5 }  ,fill opacity=1 ][line width=0.08]  [draw opacity=0] (8.93,-4.29) -- (0,0) -- (8.93,4.29) -- cycle    ;
\draw [color={rgb, 255:red, 65; green, 117; blue, 5 }  ,draw opacity=1 ]   (143.82,344.47) .. controls (117.93,308.63) and (119.5,299.41) .. (130.42,267.59) ;
\draw [shift={(131.27,265.1)}, rotate = 469.03] [fill={rgb, 255:red, 65; green, 117; blue, 5 }  ,fill opacity=1 ][line width=0.08]  [draw opacity=0] (8.93,-4.29) -- (0,0) -- (8.93,4.29) -- cycle    ;
\draw [color={rgb, 255:red, 65; green, 117; blue, 5 }  ,draw opacity=1 ]   (161.04,218.41) .. controls (178.11,199.06) and (190.74,184.71) .. (226.8,187.89) ;
\draw [shift={(229.62,188.17)}, rotate = 186.26] [fill={rgb, 255:red, 65; green, 117; blue, 5 }  ,fill opacity=1 ][line width=0.08]  [draw opacity=0] (8.93,-4.29) -- (0,0) -- (8.93,4.29) -- cycle    ;
\draw [color={rgb, 255:red, 65; green, 117; blue, 5 }  ,draw opacity=1 ]   (229.62,188.17) .. controls (187.8,133) and (216.49,53.46) .. (316.52,60.71) ;
\draw [shift={(318.04,60.83)}, rotate = 184.61] [fill={rgb, 255:red, 65; green, 117; blue, 5 }  ,fill opacity=1 ][line width=0.08]  [draw opacity=0] (8.93,-4.29) -- (0,0) -- (8.93,4.29) -- cycle    ;
\draw [color={rgb, 255:red, 65; green, 117; blue, 5 }  ,draw opacity=1 ]   (320.66,62.11) .. controls (403.25,-86.25) and (648.37,213.74) .. (603.14,494.47) ;
\draw [shift={(603.14,494.47)}, rotate = 279.15] [fill={rgb, 255:red, 65; green, 117; blue, 5 }  ,fill opacity=1 ][line width=0.08]  [draw opacity=0] (8.93,-4.29) -- (0,0) -- (8.93,4.29) -- cycle    ;

\draw (332.01,642.19) node [anchor=north west][inner sep=0.75pt]  [font=\scriptsize,color={rgb, 255:red, 74; green, 144; blue, 226 }  ,opacity=1 ]  {$1/2\ $};
\draw (575.21,644.1) node [anchor=north west][inner sep=0.75pt]  [font=\scriptsize,color={rgb, 255:red, 74; green, 144; blue, 226 }  ,opacity=1 ]  {$x''$};
\draw (90.35,642.72) node [anchor=north west][inner sep=0.75pt]  [font=\scriptsize,color={rgb, 255:red, 74; green, 144; blue, 226 }  ,opacity=1 ]  {$f( x'')$};
\draw (153.09,642.6) node [anchor=north west][inner sep=0.75pt]  [font=\scriptsize,color={rgb, 255:red, 74; green, 144; blue, 226 }  ,opacity=1 ]  {$f^{2}( x'')$};
\draw (228.21,638.43) node [anchor=north west][inner sep=0.75pt]  [font=\scriptsize,color={rgb, 255:red, 74; green, 144; blue, 226 }  ,opacity=1 ]  {$f^{p-2}( x'')$};
\draw (615.9,617.85) node [anchor=north west][inner sep=0.75pt]  [font=\scriptsize,color={rgb, 255:red, 114; green, 41; blue, 50 }  ,opacity=1 ]  {$x'$};
\draw (53.45,613.18) node [anchor=north west][inner sep=0.75pt]  [font=\scriptsize,color={rgb, 255:red, 114; green, 41; blue, 50 }  ,opacity=1 ]  {$f( x')$};
\draw (125.74,613.85) node [anchor=north west][inner sep=0.75pt]  [font=\scriptsize,color={rgb, 255:red, 114; green, 41; blue, 50 }  ,opacity=1 ]  {$f^{2}( x')$};
\draw (285.41,615.6) node [anchor=north west][inner sep=0.75pt]  [font=\scriptsize,color={rgb, 255:red, 114; green, 41; blue, 50 }  ,opacity=1 ]  {$f^{p-1}( x')$};
\draw (595.39,675.53) node [anchor=north west][inner sep=0.75pt]  [font=\scriptsize,color={rgb, 255:red, 65; green, 117; blue, 5 }  ,opacity=1 ]  {$x^{*}$};
\draw (66.55,674.78) node [anchor=north west][inner sep=0.75pt]  [font=\scriptsize,color={rgb, 255:red, 65; green, 117; blue, 5 }  ,opacity=1 ]  {$f\left( x^{*}\right)$};
\draw (127.91,675) node [anchor=north west][inner sep=0.75pt]  [font=\scriptsize,color={rgb, 255:red, 65; green, 117; blue, 5 }  ,opacity=1 ]  {$f^{2}\left( x^{*}\right)$};
\draw (206.62,674.69) node [anchor=north west][inner sep=0.75pt]  [font=\scriptsize,color={rgb, 255:red, 65; green, 117; blue, 5 }  ,opacity=1 ]  {$f^{p-2}\left( x^{*}\right)$};
\draw (298.79,673.93) node [anchor=north west][inner sep=0.75pt]  [font=\scriptsize,color={rgb, 255:red, 65; green, 117; blue, 5 }  ,opacity=1 ]  {$f^{p-1}\left( x^{*}\right)$};
\draw (646.72,606.4) node [anchor=north west][inner sep=0.75pt]  [font=\normalsize,color={rgb, 255:red, 0; green, 0; blue, 0 }  ,opacity=1 ]  {$x$};
\draw (41,34.4) node [anchor=north west][inner sep=0.75pt]  [font=\normalsize,color={rgb, 255:red, 0; green, 0; blue, 0 }  ,opacity=1 ]  {$f( x)$};

\end{tikzpicture}

%% file: supp_lb_arxiv.tex
\section{Additional Proofs for Section~\ref{sec:lb}}

\subsection{Proof of Lemma~\ref{lemma:increasing-constant-osc}}\label{assec:inc-constant-osc-proof}
We restate and prove the lemma. This is the main technical lemma that we use to get the sharper depth-width tradeoffs and the improved notion of \textit{constant} approximation.

\lemmaincreasingconstantosc*
\begin{proof}
We first lower-bound the total number of oscillations that will appear an increasing $p$-cycle is present. 
Later, we show that the size of the oscillations is large as well.

Because we have an increasing cycle of itinerary $12\dots p$, we assume (wlog) that the cycle is $(x_1, \dots, x_p)$ with $x_1 < x_2 < \dots < x_p$. 
Define intervals $I_j \vcentcolon= [x_j, x_{j+1}]$ for $j \in \{1, \dots, p-1\}$. Because $f$ is continuous, we conclude that $I_{j+1} \subset f(I_{j})$ for all $j < p$ and $I_{j} \subset f(I_{p-1})$ for all $j$.
Figure~\ref{fig:inc-intervals} visualizes these relationships.

\begin{figure}
    \centering
    \input{fig/fig-inc-intervals}
    \caption{Visualizes the intervals $I_1,\dots, I_{p-1}$ defined in the proof of Lemma~\ref{lemma:increasing-constant-osc} and which intervals $f$ maps to one another when $f$ has an increasing $p$-cycle.}
    \label{fig:inc-intervals}
\end{figure}
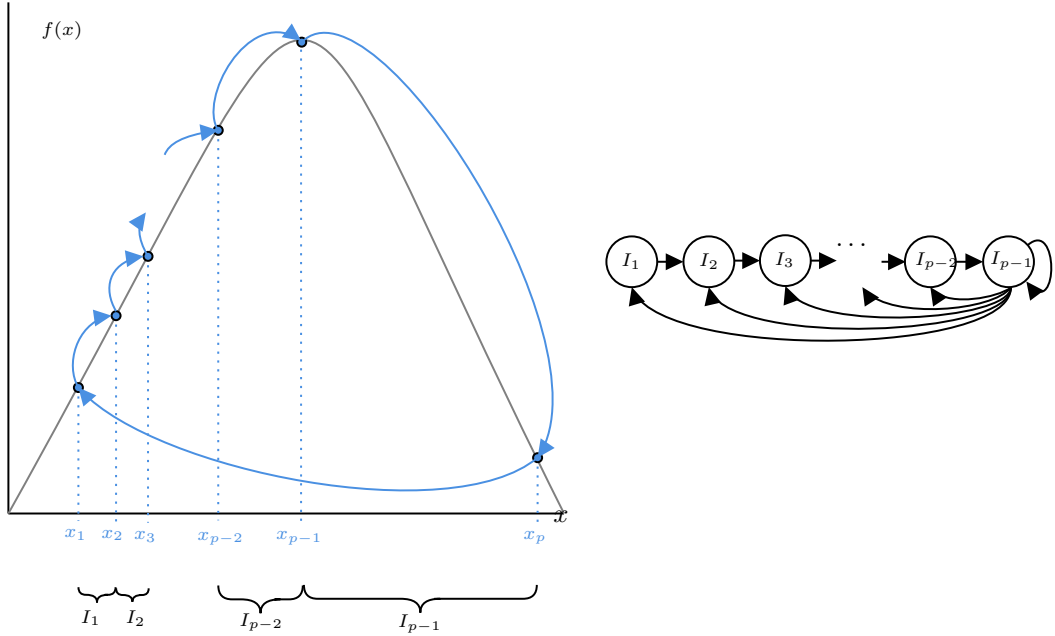

Using the methods of \cite{cnpw19}, we define $y^{(k)} \in \N^{p-1}$ such that $y^{(k)}_j$ is a lower bound on the number of times $f^k$ passes through interval $I_j$, or \[\osci{x_{j}, x_{j+1}}(f^k) \geq y^{(k)}_j.\] We can then encode the interval relationships above with $y^{(k+1)} = A_p y^{(k)}$ where $y^{({0})}$ is a vector of all ones and and $A_p \in \bit^{(p-1) \times (p-1)}$ with $(A_p)_{i, j} = \indicator{j = p-1 \text{ or } i = j + 1}$. We get the following adjacency matrix for the intervals, capturing the mapping relationships (under $f$) between them:
\[A_p = 
\begin{bmatrix}
    0 & 0 & 0 & \cdots & 0 & 1 \\
    1 & 0 & 0 & \cdots & 0 & 1 \\
    0 & 1 & 0 & \cdots & 0 & 1 \\
    0 & 0 & 1 & \cdots & 0 & 1 \\
    \vdots & \vdots & \vdots & \ddots & \vdots & \vdots \\
    0 & 0 & 0 & \cdots & 1 & 1
\end{bmatrix}.\]

We find the characteristic polynomial of $A_p$ and lower-bound $y^{(k+1)}$ with the spectral radius of $A_p$. 
We show by induction on $p \geq 3$ that
\[\det(A_p - \lambda I) = (-1)^{p-1} \paren{\lambda^{p-1} - \sum_{i=0}^{p-2} \lambda^i}.\]

For the base case $p = 3$, we have:
\begin{align*}
    \det(A_3 - \lambda I)
    &= \begin{vmatrix}
        -\lambda & 1 \\
        1 & 1 -\lambda
    \end{vmatrix}
    = \lambda^2 - \lambda - 1,
\end{align*}
which satisfies the desired form.

Now, we show the inductive step by expanding the determinant of $A_p - \lambda I$.

{\small
\begin{align*}
    \det(A_p - \lambda I)
    &= 
    \begin{vmatrix}
        -\lambda & 0 & 0 & \cdots & 0 & 1 \\
        1 & -\lambda & 0 & \cdots & 0 & 1 \\
        0 & 1 & -\lambda & \cdots & 0 & 1 \\
        0 & 0 & 1 & \cdots & 0 & 1 \\
        \vdots & \vdots & \vdots & \ddots & \vdots & \vdots \\
        0 & 0 & 0 & \cdots & -\lambda & 1 \\
        0 & 0 & 0 & \cdots & 1 & 1 - \lambda
    \end{vmatrix}
    = -\lambda 
    \begin{vmatrix}
        -\lambda & 0 & \cdots & 0 & 1 \\
        1 & -\lambda & \cdots & 0 & 1 \\
        0 & 1 & \cdots & 0 & 1 \\
        \vdots & \vdots & \ddots & \vdots & \vdots \\
        0 & 0 & \cdots & -\lambda & 1 \\
        0 & 0 & \cdots & 1 & 1 - \lambda
    \end{vmatrix} - 
    \begin{vmatrix}
        0 & 0 & \cdots & 0 & 1 \\
        1 & -\lambda & \cdots & 0 & 1 \\
        0 & 1 & \cdots & 0 & 1 \\
        \vdots & \vdots & \ddots & \vdots & \vdots \\
        0 & 0 & \cdots & -\lambda & 1 \\
        0 & 0 & \cdots & 1 & 1 - \lambda
    \end{vmatrix}.
\end{align*}}

The left determinant exactly equals $\det(A_{p-1} - \lambda I)$, which we can expand using the inductive hypothesis.
The second equals $(-1)^{p-2}$, because $p-2$ row swaps (which are elementary row operations) can be used to move the first row to the bottom and make the matrix upper-triangular with diagonals of one. 
We conclude the inductive step below.

\begin{align*}
        \det(A_p - \lambda I)
        &= -\lambda \det(A_{p-1} - \lambda I) - (-1)^{p-2} \\ 
    &= -\lambda (-1)^{p-2} \left(\lambda^{p-2} - \sum_{i=0}^{p-3} \lambda^i \right) + (-1)^{p-1}
    = (-1)^{p-1} \left(\lambda^{p-1} - \sum_{i=0}^{p-2} \lambda^i\right).
\end{align*}

We find the eigenvalues of $A_p$ by finding the roots of the polynomial
\[P(x) = \lambda^{p-1} - \sum_{i=0}^{p-2} \lambda^i = 0.\] 
Observe that there must be a root greater than $1$ because $P(1) = 2 - p < 0$ and $P(2) = 1 > 0$. Equivalently, if $\lambda \neq 1$,
\[P(x) = \lambda^{p-1} - \frac{1 - \lambda^{p-1}}{1 - \lambda} = \frac{\lambda^p - 2\lambda^{p-1} + 1}{\lambda - 1} =  0.\]

Hence, finding the largest root of $P$ is equivalent to finding the largest root of $\lambda^p - 2\lambda^{p-1} + 1$, which is $\rhoinc{p}$ by definition.

This implies that the spectral radius of $A_p$, sp$(A_p)=\rhoinc{p}>1$, and hence, we also have sp$(A_p^k)=$ sp$(A_p)^k=\rhoinc{p}^k$. Since all the elements in $A_p$ and in $A_p^k$ are non-negative, then the infinity norm of $A_p^k$ is by definition the maximum among its row sums. Since the last column of $A_p$ is the all 1's vector, the largest row sum in $A_p^k$ appears at its last row:

\[
||A_p^k||_{\infty} = \sum_{j=1}^{p-1}(A_{p}^k)_{p-1,j}
\]

We can now use the fact that the infinity norm of a matrix is larger than its spectral norm:
\[
||A_p^k||_{\infty} \ge  \rhoinc{p}^k
\]
We conclude that there exists at least one interval $I_{j^*}$ (e.g., the interval $I_{p-1}$) which is crossed at least $\rhoinc{p}^k$ times by $f^k$, so $\osci{x_{j^*}, x_{j^*+1}}(f^k) \geq \rhoinc{p}^k$.

Thus, for some $a',b'$ we get $\osci{a', b'}(f^k) \geq \rhoinc{p}^k$. But can we find $a',b'$ with large difference $b'-a'$?

Now, we show that the intervals traversed are sufficiently large, in order to lower-bound $\osci{a,b}(f^k)$ with $b-a \geq \frac{1}{18}$.
By Lemma~\ref{lemma:interval-inc-cycle}, there exists some $j$ with $x_{j+1} - x_{j} \geq \frac{1}{18}$.
It suffices to show that $f^k$ traverses the interval $I_j$ sufficiently many times.

From earlier in the proof, there exists some $j^*$ such that $f$ crosses $I_{j^*}$ at least $N\vcentcolon=\rhoinc{p}^{k}$ times. 
We conclude by showing that every other interval is traversed at least half as often as this most popular interval, which suggests that $\osci{x_{j}, x_{j+1}}(f) \geq \frac{N}{2}$.

For $A \in \mathbb{R}^{(p-1)\times (p-1)}$ as defined earlier in the section and for $y^{(k)} \vcentcolon= A^{k} \vec{1}$, we argue inductively that the elements of $y^{(k)}$ are non-decreasing and that $y^{(k)}_{p-1} \leq 2 y^{(k)}_1$. For the base case, this is trivially true for $k = 0$.

Suppose it holds for $k$.
    By construction, we have $y^{(k+1)}_1 = y^{(k)}_{p-1}$ and $y^{(k+1)}_j = y^{(k)}_{j-1} + y^{(k)}_{p-1}$ for all $j > 1$.
    By the inductive hypotheses,
    \[y^{(k+1)}_1 \leq y^{(k+1)}_2 \leq \dots \leq y^{(k+1)}_{p-1} \leq 2y^{(k+1)}_1.\]
    
Therefore, $f^k$ crosses interval $I_{j}$ at least $\frac{N}{2}$ times, and $I_j$ has width at least $\frac{1}{18}$. The claim immediately follows.
\end{proof}
    
\begin{lemma}\label{lemma:interval-inc-cycle}
    For some $p \geq 3$, consider a symmetric concave unimodal function $f$ with an increasing $p$-cycle of $x_1 < \dots < x_p$. Then, there exists $j \in [p-1]$ such that $x_{j+1} - x_j \geq \frac{1}{18}$.
\end{lemma}

\begin{proof}
By the continuity of $f$, note that $[x_1, x_p] \subset f^3([x_{p-3}, x_{p-2}])$. 
There then exists some $y_1 \in [x_{p-3}, x_{p-2}]$ such that $f^3(y_1) = y_1$, $y_2\vcentcolon= f(y_1) \in [x_{p-2}, x_{p-1}]$, and $y_3\vcentcolon= f(y_2) \in [x_{p-1}, x_p]$.
Thus, if $f$ has a maximal $p$-cycle, then $f$ also has a 3-cycle corresponding to $x_{p-3} < y_1 < y_2 < y_3 < x_p$.

We now show that $y_3 - y_1$ must be sufficiently large by concavity.
For $f$ to be concave, the following inequality must hold:
\begin{align*}
    \frac{f(y_1) - f(0)}{y_1 - 0} \geq \frac{f(y_2) - f(y_1)}{y_2 - y_1} > 0  > \frac{f(y_3) - f(y_2)}{y_3 - y_2} \geq \frac{f(1) - f(y_3)}{1 - y_3},
\end{align*}
or equivalently,
 \[\frac{y_2}{y_1} \geq \frac{y_3 - y_2}{y_2 - y_1} > 0 > -\frac{y_3 - x_1}{y_3 - y_2} \geq -\frac{y_1}{1 - y_3}.\]
 
In addition, note that $y_1 < \frac{1}{2}$ and $y_3 > \frac{1}{2}$. 
If the former were false, then $f(y_2) \leq f(y_1)$ (by unimodality), which contradicts $y_3 > y_2$. If the latter were false, then $f(y_3) > f(y_2)$, which contradicts $y_1 < y_3$.

We consider two cases and show that either way, the interval must have width at least $\frac{1}{6}$.
\begin{itemize}
    \item If $y_2 - y_1 \leq \frac{2}{5}(y_3 - y_1)$, then $\frac{y_3 - y_2}{y_2 - y_1} \geq \frac{3}{2},$
    which mandates that $y_1 \leq \frac{2y_2}{3}$ to ensure concavity. Thus,
    \[y_3 - y_1 \geq y_3 - \frac{2y_2}{3} \geq \frac{y_3}{3} \geq \frac{1}{6}.\]
    \item If $y_2 - y_1 \geq \frac{2}{5}(y_3 - y_1)$, then $\frac{y_3 - y_1}{y_3 - y_2} \geq \frac{5}{3},$ and thus $y_1 \geq \frac{5}{3}(1 - y_3)$ and $y_3 \geq 1 - \frac{3y_1}{5}$. Then,
    \[y_3 - y_1 \geq 1 - \frac{3y_1}{5} - y_1 = 1 - \frac{8y_1}{5} \geq \frac{1}{5}.\]
\end{itemize}

Thus, we must have \[\max\{x_{p-2} - x_{p-3}, x_{p-2} - x_{p-1}, x_p - x_{p-1}\} \geq \frac{1}{18}.\qedhere\]

\end{proof}

\subsection{Proof of Fact~\ref{fact:root-lb}}\label{assec:fact}
\factrootlb*
\begin{proof}

Let $P_{\mathrm{inc}, p}(\lambda) = \lambda^{p} - 2\lambda^{p-1} + 1$.

First, observe that $\rhoinc{p} < 2$, because $P_{\mathrm{inc}, p}(\lambda) > 0$ whenever $\lambda \geq 2$. 
We lower-bound $\rhoinc{p}$ by finding some $\lambda$ for each $p$ such that $P_{\mathrm{inc}, p}(\lambda) \leq 0$ or equivalently $\lambda^{p-1}(2 - \lambda) \geq 1$ for all $p\geq3$, which bounds $\rhoinc{p}$ by the Intermediate Value Theorem. 

Consider $\lambda = 2 - \frac{4}{2^p}$. Then,
\begin{align*}
    \lambda^{p-1}(2 - \lambda)
    &= \paren{2 - \frac{4}{2^p}}^{p-1} \cdot \frac{4}{2^p}
    = 2 \paren{1 - \frac{2}{2^p}}^{p-1} \\
    &\geq 2 \paren{1 - \frac{2(p-1)}{2^p}}
    = 2 - 2 \cdot \frac{p-1}{2^{p-1}} \\
    &\geq 2 - 2 \cdot \frac{1}{2} = 1.\qedhere
\end{align*}
\end{proof}

\subsection{Previous Results about Hardness of Approximating Oscillatory Functions}

We rely on prior results from~\cite{cnpw19, cnp20} to show that an iterated function $f^k$ is inapproximable by neural networks.
These results hold if $f^k$ has sufficiently many crossings of some interval.
We apply these results later with improved bounds on both the number and the size of crossings.

\cite{cnpw19} show that the classification error of $f^k$ can be bounded if there are enough oscillations.

\begin{theorem}[\citep{cnpw19}, Section 4]\label{thm:general-lb-cls}
    Consider any continuous $f: [0,1] \to [0,1]$ and any $g \in \mathcal{N}(u,\ell)$.
    Suppose there exists $a < b$ such that $\osci{a, b}(f) = \Omega(\rho^t)$ and suppose $u \leq \frac{1}{8}\rho^{k/\ell}$.
    Then, for $t = \frac{a+b}{2}$, there exists $S$ with $\abs{S} =  \frac{1}{2} \floor{\rho^k}$ samples such that 
    \[\errcls{S,t}{f^k, g}  \geq \frac{1}{2} - \frac{(2u)^\ell}{n}.\]
\end{theorem}

We adapt that claim to lower-bound the $L_\infty$ approximation of $f^k$ by $g$.

\begin{corollary}\label{cor:general-lb-linf}
    Consider any continuous $f: [0,1] \to [0,1]$ and any $g \in \mathcal{N}(u,\ell)$.
    Suppose there exists $a < b$ such that $\osci{a, b}(f) = \Omega(\rho^t)$ and suppose $u \leq \frac{1}{8}\rho^{k/\ell}$.
    Then, 
    \[\norm[\infty]{f^k - g} \geq \frac{b-a}{2}.\]
\end{corollary}
\begin{proof}
    By Theorem~\ref{thm:general-lb-cls}, there exists some $x \in [0,1]$ such that (wlog) $f^k(x) \leq a$ and $g(x) \geq \frac{a + b}{2}$.
    The conclusion for the $L_\infty$ error is immediate by definition.
\end{proof}

\cite{cnp20} give a lower-bound on the ability of a neural network $g$ to $L_1$-approximate $f^k$, provided a correspondence between the Lipschitz constant of $f$ and the rate of oscillations $\rho$.

\begin{theorem}[\cite{cnp20} Theorem 3.2]\label{thm:general-lb-l1}
    Consider any $L$-Lipschitz $f: [0,1] \to [0,1]$ and any $g \in \mathcal{N}(u,\ell)$.
    Suppose there exists $a < b$ such that $\osci{a,b}(f) = \Omega(\rho^t)$. If $L \leq \rho$ and $u \leq \frac{1}{16} \rho^{k/\ell}$, then 
    \[\norm[1]{f^k - g} = \Omega((b-a)^2).\]
\end{theorem}

The Lipschitzness assumption is extremely strict, especially because they show in their Lemma~3.1 that $L \geq \rho$ whenever $f$ has a period of odd length.

\subsection{Proof of Corollary~\ref{cor:main-l1}}\label{assec:cor-main-proof}
\cormainl*
\begin{proof}
    This theorem follows from Theorem~\ref{thm:main-l1} and Lemma~\ref{lemma:increasing-constant-osc}.
    Because $\ftent{\rho_p/2}$ is $\rho_p$-Lipschitz, it remains only to prove that there exists an increasing $p$-cycle.
    We show that 
    \[\frac{1}{2}, f\paren{\frac12},\dots, f^{p-1}\paren{\frac12}\]
    is such a cycle.
    
    By definition of the tent map, $f(\frac12) = \frac{\rhoinc{p}}2$ and $f^2(\frac12) = \rhoinc{p} (1 - \frac{\rhoinc{p}}{2})$.
    If we assume for now that $f^j(\frac12) \leq \frac12$ for all $j \in \{2, \dots, p-1\}$,
    then
    \[f^p\paren{\frac12} = \rhoinc{p}^{p-1}\paren{1 - \frac{\rhoinc{p}}2} = -\frac{1}{2}\paren{\rhoinc{p}^p - 2\rhoinc{p}^{p-1} + 1} + \frac{1}{2} = 0 + \frac{1}{2}.\]
    
    Because $f^p(\frac12) = \frac12$ and we assumed that $f^{j+1}(\frac12) = \rhoinc{p} f^{j}(\frac12)$ for $j \geq 2$ and $\rho > 1$, it must be the case that $f^j(\frac12) \leq \frac12$ for all $j \in \{2, \dots, p-1\}$.
    
    Lemma~\ref{lemma:increasing-constant-osc} thus implies that $f^k$ has $\Omega(\rhoinc{p}^k)$ crossings, which enables us to complete the proof by invoking Theorem~\ref{thm:main-l1}, since the Lipschitzness condition is met.
\end{proof}

\subsection{Proof of Lemma~\ref{lemma:odd-constant-osc}}\label{assec:lemma-odd-proof}

\lemmaoddconstantosc*

\begin{proof}
    By Theorems~2.94 and 3.11.1 of~\cite{alseda2000}, there exists a $p$-cycle of the form \[x_p < x_{p-2} < \dots < x_3 < x_1 < x_2 < x_4 < \dots < x_{p-1},\]
    which is known as a \textit{Stefan cycle}.
    The analysis of Section~3.2 of~\cite{cnp20} shows that $\osci{[x_1, x_2]}(f^{k}) \geq \rhoodd{p}^{k}$.
    Their exploitation of the relationships between intervals is visualized in Figure~\ref{fig:stefan-intervals}.
    By the continuity of $f$, applying $f$ an additional $p-1$ times gives $\osci{[x_p, x_1]}(f^{k+p-1}) \geq \rhoodd{p}^{k}$.
    Because $[x_{p-2}, x_1] \subset [x_p, x_1]$, applying $f$ one more time gives $\osci{[x_2, x_{p-1}]}(f^{k+p}) \geq \rhoodd{p}^k$.

    \begin{figure}
        \centering
        \input{fig/fig-stefan-intervals}
        \caption{Gives an example of a Stefan $p$-cycle (which is relied upon in Lemma~\ref{lemma:odd-constant-osc} and demonstrates the interval relationships). Analogous to Figure~\ref{fig:inc-intervals}.}
        \label{fig:stefan-intervals}
    \end{figure}
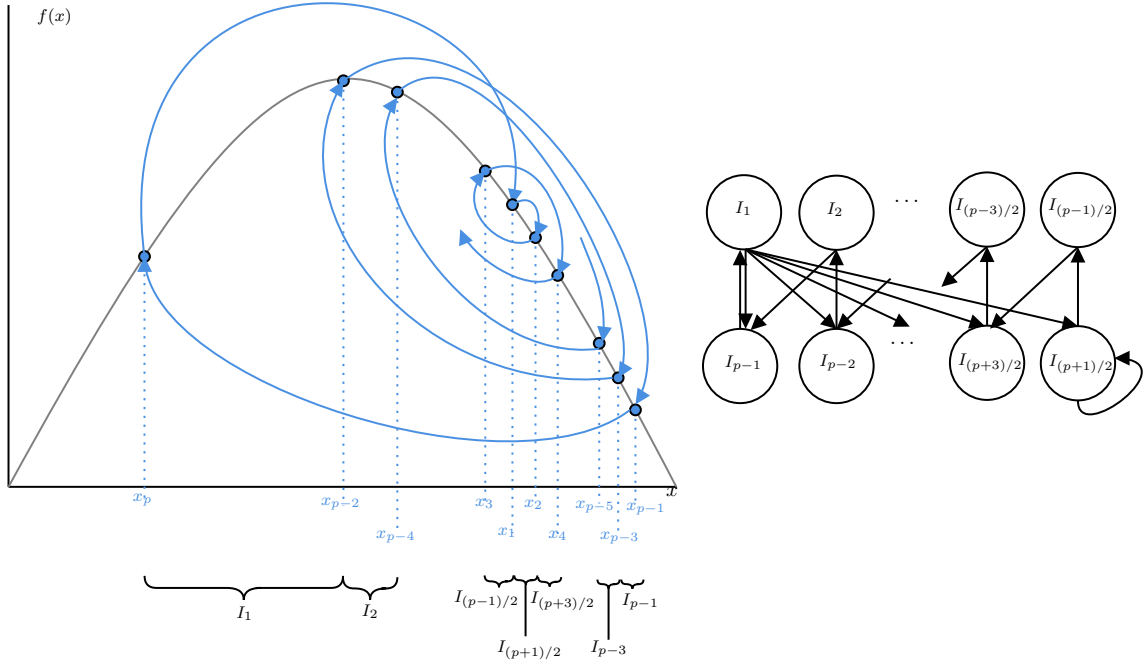
    
    Hence, by redefining $k$, we have 
    \[\max\{ \osci{[x_1, x_2]}(f^{k}), \osci{[x_2, x_{p-1}]}(f^{k}), \osci{[x_p, x_1]}(f^{k})\} \geq \rhoodd{p}^{k-p}.\]
    
    Since $[x_p, x_{p-1}]$ is the disjoint union of $[x_1, x_2]$, $[x_2, x_{p-1}]$, and $[x_p, x_1]$, there exists $[a, b] \subset [x_p, x_{p-1}]$ with $b- a \geq \frac{1}{3}(x_{p-1} - x_p)$ such that $\osci{[a,b]}(f^{k}) \geq \\rhoodd{p}^{k-p}$.
    
    The problem reduces to placing a lower bound on $x_{p-1} - x_p$.
    To do so, we derive contradictions on the concavity and symmetry of $f$. 
    Let $r = f(\frac12) \in (x_p, 1)$ be the the largest outcome of $f$, and
    let \[a= \sup_{x, x' \in [1-r, r]} \abs{\frac{f(x) - f(x')}{x - x'}}\] be the maximum absolute slope of $f$ on $[1-r,r]$. $a$ must be finite by the concavity and continuity of $f$, and if $f$ is differentiable, $a = f'(1-r) = -f'(r)$.
    Thus, $f$ is $a$-Lipschitz on that interval.
    
    Because $f([x_p, x_{p-1}]) \subseteq [x_p, r] \subset [1-r, r]$, it follows that $\abs{f^2(x) - f^2(x')} \leq a^2 \abs{x - x'}$.
    Thus, $x_2 - x_p \leq a^2(x_{p-2} - x_{p})$ and $x_2 - x_p \leq x_4 - x_p \leq a^2(x_2 - x_{p-2})$.
    Averaging the two together, we have $x_2 - x_p \leq \frac{a^2}{2}(x_2 - x_{p})$, which means $a \geq \sqrt{2}$.
    
    To satisfy concavity, the following must be true:
    \[\frac{f(1-r) - f(0)}{1-r - 0} = \frac{f(r)}{1-r} \geq a \geq \sqrt{2}.\]
    
    We rearrange the inequality and apply properties of monotonicity to lower-bound $r$ away from $\frac12$:
    \[r 
    \geq 1 - \frac{f(r)}{\sqrt{2}}  
    \geq 1 - \frac{f(x_{p-1})}{\sqrt{2}}
    = 1 - \frac{x_p}{\sqrt{2}}
    > 1 - \frac1{2\sqrt2}.\]
    
    It also must be the case for any $x \in [\frac12, 1]$, that:
    \[\abs{\frac{f(x) - f\paren{\frac12}}{x - \frac12}} \leq 2.\]
    Otherwise, the concavity of $f$ would force $f(\frac12) > 1$.
    
    We finally assemble the pieces to lower-bound the gap between $x_{p-1}$ and $x_p$:
    \begin{align*}
        x_{p-1} - x_p
        &\geq x_{p-1} - \frac12 
        \geq -\frac{1}{2} \paren{f(x_{p-1}) - f\paren{\frac12}}
        = \frac{r}{2} - \frac{x_p}{2} \\
        &> \frac{1}{2} - \frac1{4\sqrt{2}} - \frac14 
        = \frac14 - \frac1{4\sqrt{2}}
        > 0.07.\qedhere
    \end{align*}

    
    
    
    
\end{proof}

\subsection{Necessity of Symmetry and Concavity Assumptions in Theorems~\ref{thm:main-linf} and \ref{thm:main-l1}}\label{assec:sym-conc}

We demonstrate the weakness of the bounds promised by \cite{cnpw19, cnp20, bzl20} and argue that our assumptions of symmetry and concavity are necessary in order to avoid such non-vacuous bounds.
To do so, we exhibit two families of functions in Propositions~\ref{prop:need-symmetry} and ~\ref{prop:need-concavity} which contain functions with increasing $p$-cycles for every $p$ that produce large numbers of oscillations, yet are trivial to approximate because their oscillations can be made arbitrarily small.
The functions considered in both cases are unimodal and lack symmetry and concavity respectively.

These expose a fundamental shortcoming of other approaches to the hardness of neural network approximation in the aforementioned works because they all rely on showing that for every mapping $f$ meeting some condition (e.g. odd period, positive topological entropy), there exists some $[a,b] \in [0,1]$ where $\osci{a,b}$ is exponentially large, and hence no poly-size shallow neural network $g$ can obtain $L_{\infty}(f^k, g) \leq P(b-a)$ for some polynomial $P$.
However, because $[a,b]$ depends on $f$, their difference can potentially be arbitrarily small. 
The propositions show that this concern is significant and that $[a,b]$ indeed becomes arbitrarily narrow for simple 3-periodic functions. 
While \cite{cnpw19} avoid addressing this issue head-on by focusing on classification error over $L_\infty $ error, their classification lower-bounds rely on misclassification of points whose actual distance can be shrinking (see for example Figure~\ref{fig:asymmetric}).

The implications of these propositions contrast with the more robust hardness results we present in Theorems~\ref{thm:main-linf}, \ref{thm:main-l1}, \ref{thm:odd-linf}, and \ref{thm:odd-l1}, which leverage unimodality, symmetry, and concavity to ensure that the accuracy of approximation can be no better than some constant (independent on $f,p$) when the neural network $g$ is too small.
We show here that those assumptions are necessary by exhibiting functions that satisfy all but one, and become easy to $L_\infty$-approximate with small depth-2 ReLU networks.






\begin{proposition}\label{prop:need-symmetry}
For $p \geq 3$ and for sufficiently small $\eps > 0$, there exists a concave unimodal mapping $f$ with a chaotic $p$-cycle such that for any $k$, there exists $g \in \mathcal{N}(3,2)$ with
\[L_\infty(f^k, g) \leq \eps.\]
\end{proposition}

\begin{proof}
    For all $j \in [p]$, let $x_j = 1 - \frac{p-j+1}{p} \eps$.
    Define $f$ to be a piecewise-linear function with $p+1$ pieces chosen with boundaries that satisfy 
    \[f(0) = 0, f(x_1) = x_2, f(x_2) = x_3, \dots, f(x_{p-1}) = x_p, f(x_p) = x_1, f(1) = 0.\]
    
    We visualize $f$ for $p=3$ in Figure~\ref{fig:asymmetric}. $f$ is unimodal because it increases on $[0, x_{p-1}]$ and decreases on $[x_{p-1}, 1]$.
    It is concave because $f'(x)$ does not increase as $x$ grows, since 
    \[f'(x) = \begin{cases}
        \frac{1-\frac{p-1}{p}\eps}{1 - \eps} > 1 & x \in [0, x_1) \\
        1 & x \in (x_1, x_{p-1}) \\
        - p + 1 & x \in (x_{p-1}, x_p) \\
        - \frac{1 - \eps}{\eps} & x \in (x_p, 1],
    \end{cases}\]
    as long as $\frac{1 - \eps}{\eps} > p - 1$.
    
    \begin{figure}
        \centering
        \includegraphics[width=0.45\textwidth]{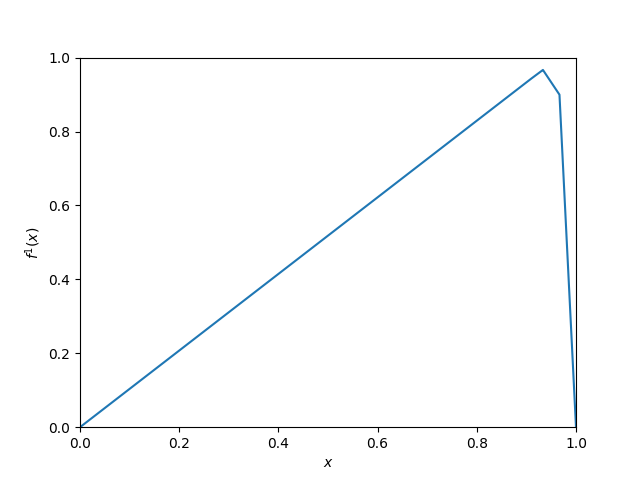}
        \includegraphics[width=0.45\textwidth]{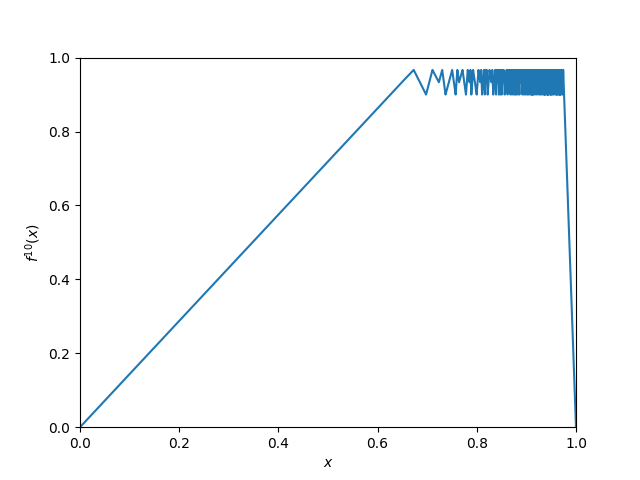}
        \caption{Plots the asymmetric function with a $p$-cycle referenced in Proposition~\ref{prop:need-symmetry} for $p=3$ and $\eps = 0.1$. While $f$ oscillates frequently, $f$ can be trivially $0.1$-approximated by three ReLUs. As $\eps \to 0$, the $L_\infty$ approximation hardness guarantees implied by \cite{cnpw19} become vacuous because the oscillations, even though they are exponentially many, they shrink in size.}
        \label{fig:asymmetric}
    \end{figure}

    We show inductively that for all $k$, there exists $a_k < b_k$ such that $f^k(a_k) = f^k(b_k) = 1 - \eps$, $f^k([a_k, b_k]) \in [1 - \eps, 1]$, and $f^k$ has exactly one linear piece for each of the intervals $[0, a_k]$ and $[b_k, 1]$.
    
    These are true for the base case $k = 1$ for $a_1 \in (0, x_1)$ and $b_1 = x_p$.
    
    If the claim holds for $k$, then there is some $a_{k+1} \in (0, a_k)$ and $b_{k+1}\in (b_k, 1)$ such that $f(a_{k+1}) = f(b_{k+1}) = a_k$.
    Then, $f^{k+1}(a_{k+1}) = f^{k+1}(b_{k+1}) = 1-\eps$ and $f^{k+1}([0, a_{k+1}]) = f^{k+1}([b_{k+1},1]) = [0, 1 - \eps]$.
    For all $x \in [0, a_{k+1}]$, $f^j(x) \leq 1 - \eps$ for all $j \leq k+1$. Hence, $f^{k+1}$ is linear on $[0, a_{k+1}]$ (and also $[b_{k+1}, 1]$.
    Because $f([x_1, x_p]) = [x_1, x_p]$, $f^{k+1}([a_{k+1}, b_{k+1}]) \subseteq [x_1, x_p] \subseteq [1 - \eps, 1]$.
    The claim then holds for $k+1$.
    
    Thus, the piecewise linear mapping $g$ with boundaries $g(0)=0$, $g(a_k) = 1-\eps$, $g(b_k) = 1-\eps$, and $g(1) = 0$ is an $\eps$-approximation of $f$.
    Because $g$ has three pieces and contains the origin, it can be exactly represented by a linear combination of four ReLUs, and hence as a depth-2 neural network of width 3.
\end{proof}

\begin{proposition}\label{prop:need-concavity}
For $p \geq 3$ and for sufficiently small $\eps > 0$, there exists a symmetric unimodal mapping $f$ with a chaotic $p$-cycle such that for any $k$, there exists $g \in \mathcal{N}(3,2)$ with
\[L_\infty(f^k, g) \leq \eps.\]
\end{proposition}
\begin{proof}
    Let $x_j = \frac{1}{2} - \frac{p-1-j}{2(p-1)}\eps$ for all $j \in [p-1]$ and $x_p = \frac{1}{2} + \frac{\eps}{2}$.
    Let $f$ be a piecewise-linear function with boundaries
    \begin{align*}
        f(0) = 0, \
        f\paren{\frac{1}{2} - \frac{\eps}{2}} = \frac{1}{2}-\frac{p-2}{p-1} \cdot \frac{\eps}{2}, \
        f\paren{\frac{1}{2} - \frac{\eps}{2(p-1)}} = \frac{1}{2}, \ f\paren{\frac{1}{2}} = \frac{1}{2} + \frac{\eps}{2}, \\
        f\paren{\frac{1}{2} + \frac{\eps}{2(p-1)}} = \frac{1}{2},  \ 
        f\paren{\frac{1}{2} + \frac{\eps}{2}} = \frac{1}{2}-\frac{p-2}{p-1} \cdot \frac{\eps}{2},\
        f(1) = 0.
    \end{align*}
    
    We visualize $f$ for $p=3$ in Figure~\ref{fig:non-concave}. Note that $f$ is symmetric and unimodal and has an increasing $p$-cycle $x_1 < \dots < x_p$. 
    It is \textit{not} concave because $f'(x) = 1$ for $x \in [x_1, x_{p-2}]$ and $f'(x) = 2(p-1)$ for $x \in [x_{p-2}, x_{p-1}]$.
    
    \begin{figure}
        \centering
        \includegraphics[width=0.45\textwidth]{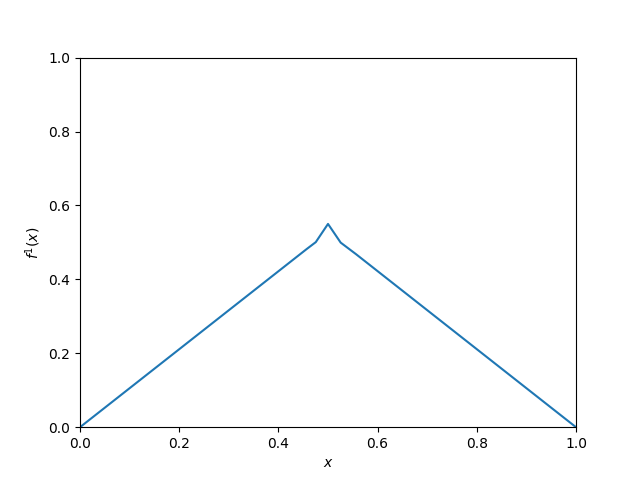}
         \includegraphics[width=0.45\textwidth]{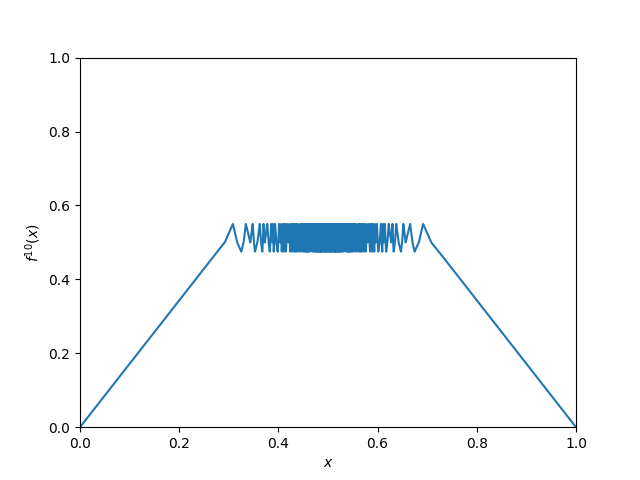}
        \caption{Another example of a function with a 3-cycle that can be $\eps$-approximated for arbitrarily small $\eps$. (Here, $\eps = 0.1$.) This function corresponds to the one in Proposition~\ref{prop:need-concavity} and the \cite{cnpw19} bounds are again vacuous for small $\eps$. Unlike Figure~\ref{fig:asymmetric}, this function is symmetric, but not concave.}
        \label{fig:non-concave}
    \end{figure}
    
    Using a very similar argument to argument from the proof of Proposition~\ref{prop:need-symmetry}, for all $k$, there exists $a_k < b_k$ such that $f^k$ is linear on $[0, a_k]$ and $[b_k, 1]$ and $f^k([a_k, b_k]) \in [\frac{1}{2} - \eps, \frac{1}{2} + \eps]$.
    As before, there exists a piecewise linear function with three pieces (which can be thought of as a depth-2 neural network of width 3) that $\eps$-approximates $f$.
\end{proof}

%% file: fig/fig-inc-intervals.tex
\tikzset{every picture/.style={line width=0.75pt}} 

\begin{tikzpicture}[x=0.75pt,y=0.75pt,yscale=-1,xscale=1]

\draw    (67.5,7.15) -- (67.5,265.14) ;
\draw    (347.74,265.14) -- (67.5,265.14) ;
\draw [color={rgb, 255:red, 128; green, 128; blue, 128 }  ,draw opacity=1 ]   (67.5,265.14) .. controls (247.57,-57.19) and (189.85,-49.89) .. (347.74,265.14) ;
\draw  [fill={rgb, 255:red, 74; green, 144; blue, 226 }  ,fill opacity=1 ] (213.17,27.23) .. controls (213.17,25.98) and (214.18,24.96) .. (215.43,24.96) .. controls (216.69,24.96) and (217.7,25.98) .. (217.7,27.23) .. controls (217.7,28.48) and (216.69,29.5) .. (215.43,29.5) .. controls (214.18,29.5) and (213.17,28.48) .. (213.17,27.23) -- cycle ;
\draw  [fill={rgb, 255:red, 74; green, 144; blue, 226 }  ,fill opacity=1 ] (332.13,236.97) .. controls (332.13,235.72) and (333.15,234.7) .. (334.4,234.7) .. controls (335.65,234.7) and (336.67,235.72) .. (336.67,236.97) .. controls (336.67,238.22) and (335.65,239.24) .. (334.4,239.24) .. controls (333.15,239.24) and (332.13,238.22) .. (332.13,236.97) -- cycle ;
\draw [color={rgb, 255:red, 74; green, 144; blue, 226 }  ,draw opacity=1 ] [dash pattern={on 0.84pt off 2.51pt}]  (215.11,26.52) -- (215.11,270.06) ;
\draw [color={rgb, 255:red, 74; green, 144; blue, 226 }  ,draw opacity=1 ] [dash pattern={on 0.84pt off 2.51pt}]  (334.4,236.97) -- (334.4,271.21) ;
\draw  [fill={rgb, 255:red, 74; green, 144; blue, 226 }  ,fill opacity=1 ] (100.53,201.56) .. controls (100.53,200.31) and (101.55,199.3) .. (102.8,199.3) .. controls (104.05,199.3) and (105.07,200.31) .. (105.07,201.56) .. controls (105.07,202.82) and (104.05,203.83) .. (102.8,203.83) .. controls (101.55,203.83) and (100.53,202.82) .. (100.53,201.56) -- cycle ;
\draw [color={rgb, 255:red, 74; green, 144; blue, 226 }  ,draw opacity=1 ] [dash pattern={on 0.84pt off 2.51pt}]  (102.8,201.56) -- (102.8,268.93) ;
\draw [color={rgb, 255:red, 74; green, 144; blue, 226 }  ,draw opacity=1 ] [dash pattern={on 0.84pt off 2.51pt}]  (121.8,165.29) -- (121.8,270.06) ;
\draw  [fill={rgb, 255:red, 74; green, 144; blue, 226 }  ,fill opacity=1 ] (119.53,165.29) .. controls (119.53,164.04) and (120.55,163.02) .. (121.8,163.02) .. controls (123.05,163.02) and (124.07,164.04) .. (124.07,165.29) .. controls (124.07,166.54) and (123.05,167.56) .. (121.8,167.56) .. controls (120.55,167.56) and (119.53,166.54) .. (119.53,165.29) -- cycle ;
\draw  [fill={rgb, 255:red, 74; green, 144; blue, 226 }  ,fill opacity=1 ] (135.65,135.35) .. controls (135.65,134.1) and (136.67,133.08) .. (137.92,133.08) .. controls (139.18,133.08) and (140.19,134.1) .. (140.19,135.35) .. controls (140.19,136.61) and (139.18,137.62) .. (137.92,137.62) .. controls (136.67,137.62) and (135.65,136.61) .. (135.65,135.35) -- cycle ;
\draw  [fill={rgb, 255:red, 74; green, 144; blue, 226 }  ,fill opacity=1 ] (171.06,71.73) .. controls (171.06,70.48) and (172.08,69.46) .. (173.33,69.46) .. controls (174.58,69.46) and (175.6,70.48) .. (175.6,71.73) .. controls (175.6,72.99) and (174.58,74) .. (173.33,74) .. controls (172.08,74) and (171.06,72.99) .. (171.06,71.73) -- cycle ;
\draw [color={rgb, 255:red, 74; green, 144; blue, 226 }  ,draw opacity=1 ] [dash pattern={on 0.84pt off 2.51pt}]  (137.92,135.35) -- (137.92,269.77) ;
\draw [color={rgb, 255:red, 74; green, 144; blue, 226 }  ,draw opacity=1 ] [dash pattern={on 0.84pt off 2.51pt}]  (173.33,71.73) -- (173.33,268.62) ;
\draw [color={rgb, 255:red, 74; green, 144; blue, 226 }  ,draw opacity=1 ]   (121.8,163.02) .. controls (115.07,151.61) and (118.23,139.22) .. (132.74,135.88) ;
\draw [shift={(135.65,135.35)}, rotate = 532.35] [fill={rgb, 255:red, 74; green, 144; blue, 226 }  ,fill opacity=1 ][line width=0.08]  [draw opacity=0] (8.93,-4.29) -- (0,0) -- (8.93,4.29) -- cycle    ;
\draw [color={rgb, 255:red, 74; green, 144; blue, 226 }  ,draw opacity=1 ]   (102.8,201.56) .. controls (96.03,190.09) and (101.83,169.95) .. (116.82,165.84) ;
\draw [shift={(119.53,165.29)}, rotate = 532.35] [fill={rgb, 255:red, 74; green, 144; blue, 226 }  ,fill opacity=1 ][line width=0.08]  [draw opacity=0] (8.93,-4.29) -- (0,0) -- (8.93,4.29) -- cycle    ;
\draw [color={rgb, 255:red, 74; green, 144; blue, 226 }  ,draw opacity=1 ]   (146.31,84.09) .. controls (148.47,78.41) and (155.44,74.51) .. (170.37,72.16) ;
\draw [shift={(173.33,71.73)}, rotate = 532.35] [fill={rgb, 255:red, 74; green, 144; blue, 226 }  ,fill opacity=1 ][line width=0.08]  [draw opacity=0] (8.93,-4.29) -- (0,0) -- (8.93,4.29) -- cycle    ;
\draw [color={rgb, 255:red, 74; green, 144; blue, 226 }  ,draw opacity=1 ]   (137.92,135.35) .. controls (132.09,125.45) and (133.11,119.63) .. (135.25,115.63) ;
\draw [shift={(136.81,113.17)}, rotate = 485.54] [fill={rgb, 255:red, 74; green, 144; blue, 226 }  ,fill opacity=1 ][line width=0.08]  [draw opacity=0] (8.93,-4.29) -- (0,0) -- (8.93,4.29) -- cycle    ;
\draw [color={rgb, 255:red, 74; green, 144; blue, 226 }  ,draw opacity=1 ]   (173.33,71.73) .. controls (163.03,56.08) and (187.91,10.64) .. (212.8,25.01) ;
\draw [shift={(215.11,26.52)}, rotate = 216.14] [fill={rgb, 255:red, 74; green, 144; blue, 226 }  ,fill opacity=1 ][line width=0.08]  [draw opacity=0] (8.93,-4.29) -- (0,0) -- (8.93,4.29) -- cycle    ;
\draw [color={rgb, 255:red, 74; green, 144; blue, 226 }  ,draw opacity=1 ]   (215.43,27.23) .. controls (252.16,-9.39) and (370.75,177.03) .. (335.51,235.25) ;
\draw [shift={(334.4,236.97)}, rotate = 304.61] [fill={rgb, 255:red, 74; green, 144; blue, 226 }  ,fill opacity=1 ][line width=0.08]  [draw opacity=0] (8.93,-4.29) -- (0,0) -- (8.93,4.29) -- cycle    ;
\draw [color={rgb, 255:red, 74; green, 144; blue, 226 }  ,draw opacity=1 ]   (334.4,236.97) .. controls (290.34,273.53) and (142.48,243.63) .. (104.46,203.41) ;
\draw [shift={(102.8,201.56)}, rotate = 409.49] [fill={rgb, 255:red, 74; green, 144; blue, 226 }  ,fill opacity=1 ][line width=0.08]  [draw opacity=0] (8.93,-4.29) -- (0,0) -- (8.93,4.29) -- cycle    ;
\draw   (102.84,302.33) .. controls (102.84,304.9) and (104.12,306.18) .. (106.69,306.18) -- (106.69,306.18) .. controls (110.36,306.18) and (112.19,307.46) .. (112.19,310.03) .. controls (112.19,307.46) and (114.03,306.18) .. (117.7,306.18)(116.04,306.18) -- (117.7,306.18) .. controls (120.27,306.18) and (121.55,304.9) .. (121.55,302.33) ;
\draw   (121.84,302.9) .. controls (121.84,305.15) and (122.96,306.28) .. (125.21,306.28) -- (125.21,306.28) .. controls (128.43,306.28) and (130.04,307.41) .. (130.04,309.66) .. controls (130.04,307.41) and (131.65,306.28) .. (134.87,306.28)(133.42,306.28) -- (134.87,306.28) .. controls (137.12,306.28) and (138.24,305.15) .. (138.24,302.9) ;
\draw   (173.37,301.47) .. controls (173.37,306.14) and (175.7,308.47) .. (180.37,308.47) -- (184.67,308.47) .. controls (191.34,308.47) and (194.67,310.8) .. (194.67,315.47) .. controls (194.67,310.8) and (198,308.47) .. (204.67,308.47)(201.67,308.47) -- (208.97,308.47) .. controls (213.64,308.47) and (215.97,306.14) .. (215.97,301.47) ;
\draw   (216.55,301.18) .. controls (216.55,305.85) and (218.88,308.18) .. (223.55,308.18) -- (265.27,308.18) .. controls (271.94,308.18) and (275.27,310.51) .. (275.27,315.18) .. controls (275.27,310.51) and (278.6,308.18) .. (285.27,308.18)(282.27,308.18) -- (327,308.18) .. controls (331.67,308.18) and (334,305.85) .. (334,301.18) ;
\draw   (369.1,138.08) .. controls (369.1,131) and (374.84,125.27) .. (381.91,125.27) .. controls (388.99,125.27) and (394.72,131) .. (394.72,138.08) .. controls (394.72,145.15) and (388.99,150.89) .. (381.91,150.89) .. controls (374.84,150.89) and (369.1,145.15) .. (369.1,138.08) -- cycle ;
\draw    (394.72,138.08) -- (404.96,138.08) ;
\draw [shift={(407.96,138.08)}, rotate = 180] [fill={rgb, 255:red, 0; green, 0; blue, 0 }  ][line width=0.08]  [draw opacity=0] (8.93,-4.29) -- (0,0) -- (8.93,4.29) -- cycle    ;
\draw    (545.28,138.08) -- (556.1,138.08) ;
\draw [shift={(559.1,138.08)}, rotate = 180] [fill={rgb, 255:red, 0; green, 0; blue, 0 }  ][line width=0.08]  [draw opacity=0] (8.93,-4.29) -- (0,0) -- (8.93,4.29) -- cycle    ;
\draw    (472.16,137.5) -- (482.26,137.5) ;
\draw [shift={(485.26,137.5)}, rotate = 180] [fill={rgb, 255:red, 0; green, 0; blue, 0 }  ][line width=0.08]  [draw opacity=0] (8.93,-4.29) -- (0,0) -- (8.93,4.29) -- cycle    ;
\draw    (507.86,138.08) -- (517.09,138.08) ;
\draw [shift={(520.09,138.08)}, rotate = 180] [fill={rgb, 255:red, 0; green, 0; blue, 0 }  ][line width=0.08]  [draw opacity=0] (8.93,-4.29) -- (0,0) -- (8.93,4.29) -- cycle    ;
\draw    (573.06,151.18) .. controls (573.06,184.86) and (402.02,187.89) .. (383.06,153.6) ;
\draw [shift={(381.91,150.89)}, rotate = 433.18] [fill={rgb, 255:red, 0; green, 0; blue, 0 }  ][line width=0.08]  [draw opacity=0] (8.93,-4.29) -- (0,0) -- (8.93,4.29) -- cycle    ;
\draw    (573.06,151.18) .. controls (573.62,175.75) and (436.31,180.82) .. (421.76,153.52) ;
\draw [shift={(420.77,150.89)}, rotate = 437.12] [fill={rgb, 255:red, 0; green, 0; blue, 0 }  ][line width=0.08]  [draw opacity=0] (8.93,-4.29) -- (0,0) -- (8.93,4.29) -- cycle    ;
\draw    (573.06,151.18) .. controls (572.23,167.84) and (475.42,173.79) .. (460.63,152.71) ;
\draw [shift={(459.35,150.31)}, rotate = 429.52] [fill={rgb, 255:red, 0; green, 0; blue, 0 }  ][line width=0.08]  [draw opacity=0] (8.93,-4.29) -- (0,0) -- (8.93,4.29) -- cycle    ;
\draw    (573.06,151.18) .. controls (563.71,163) and (513.67,167.21) .. (499.91,153.02) ;
\draw [shift={(498.21,150.89)}, rotate = 417.46000000000004] [fill={rgb, 255:red, 0; green, 0; blue, 0 }  ][line width=0.08]  [draw opacity=0] (8.93,-4.29) -- (0,0) -- (8.93,4.29) -- cycle    ;
\draw    (573.06,151.18) .. controls (565.91,156.21) and (543.41,160.51) .. (534.47,153.1) ;
\draw [shift={(532.47,150.89)}, rotate = 416.31] [fill={rgb, 255:red, 0; green, 0; blue, 0 }  ][line width=0.08]  [draw opacity=0] (8.93,-4.29) -- (0,0) -- (8.93,4.29) -- cycle    ;
\draw    (433.73,137.5) -- (443.54,137.5) ;
\draw [shift={(446.54,137.5)}, rotate = 180] [fill={rgb, 255:red, 0; green, 0; blue, 0 }  ][line width=0.08]  [draw opacity=0] (8.93,-4.29) -- (0,0) -- (8.93,4.29) -- cycle    ;
\draw   (407.96,138.08) .. controls (407.96,131) and (413.7,125.27) .. (420.77,125.27) .. controls (427.85,125.27) and (433.59,131) .. (433.59,138.08) .. controls (433.59,145.15) and (427.85,150.89) .. (420.77,150.89) .. controls (413.7,150.89) and (407.96,145.15) .. (407.96,138.08) -- cycle ;
\draw   (446.54,137.5) .. controls (446.54,130.43) and (452.27,124.69) .. (459.35,124.69) .. controls (466.42,124.69) and (472.16,130.43) .. (472.16,137.5) .. controls (472.16,144.58) and (466.42,150.31) .. (459.35,150.31) .. controls (452.27,150.31) and (446.54,144.58) .. (446.54,137.5) -- cycle ;
\draw   (519.66,138.08) .. controls (519.66,131) and (525.39,125.27) .. (532.47,125.27) .. controls (539.54,125.27) and (545.28,131) .. (545.28,138.08) .. controls (545.28,145.15) and (539.54,150.89) .. (532.47,150.89) .. controls (525.39,150.89) and (519.66,145.15) .. (519.66,138.08) -- cycle ;
\draw   (559.1,138.08) .. controls (559.1,131) and (564.83,125.27) .. (571.91,125.27) .. controls (578.98,125.27) and (584.72,131) .. (584.72,138.08) .. controls (584.72,145.15) and (578.98,150.89) .. (571.91,150.89) .. controls (564.83,150.89) and (559.1,145.15) .. (559.1,138.08) -- cycle ;
\draw    (581.7,129.59) .. controls (597.17,116.32) and (597.26,165.67) .. (582.79,150.28) ;
\draw [shift={(580.9,148.01)}, rotate = 413.25] [fill={rgb, 255:red, 0; green, 0; blue, 0 }  ][line width=0.08]  [draw opacity=0] (8.93,-4.29) -- (0,0) -- (8.93,4.29) -- cycle    ;

\draw (340.98,262.02) node [anchor=north west][inner sep=0.75pt]  [font=\normalsize,color={rgb, 255:red, 0; green, 0; blue, 0 }  ,opacity=1 ]  {$x$};
\draw (82.66,15.23) node [anchor=north west][inner sep=0.75pt]  [font=\scriptsize,color={rgb, 255:red, 0; green, 0; blue, 0 }  ,opacity=1 ]  {$f( x)$};
\draw (200.98,272.57) node [anchor=north west][inner sep=0.75pt]  [font=\scriptsize,color={rgb, 255:red, 74; green, 144; blue, 226 }  ,opacity=1 ]  {$x_{p-1}$};
\draw (325.87,272.28) node [anchor=north west][inner sep=0.75pt]  [font=\scriptsize,color={rgb, 255:red, 74; green, 144; blue, 226 }  ,opacity=1 ]  {$x_{p}$};
\draw (94.27,271.63) node [anchor=north west][inner sep=0.75pt]  [font=\scriptsize,color={rgb, 255:red, 74; green, 144; blue, 226 }  ,opacity=1 ]  {$x_{1}$};
\draw (112.98,271.91) node [anchor=north west][inner sep=0.75pt]  [font=\scriptsize,color={rgb, 255:red, 74; green, 144; blue, 226 }  ,opacity=1 ]  {$x_{2}$};
\draw (129.39,272.49) node [anchor=north west][inner sep=0.75pt]  [font=\scriptsize,color={rgb, 255:red, 74; green, 144; blue, 226 }  ,opacity=1 ]  {$x_{3}$};
\draw (161.39,272.78) node [anchor=north west][inner sep=0.75pt]  [font=\scriptsize,color={rgb, 255:red, 74; green, 144; blue, 226 }  ,opacity=1 ]  {$x_{p-2}$};
\draw (102.61,312.39) node [anchor=north west][inner sep=0.75pt]  [font=\scriptsize]  {$I_{1}$};
\draw (125.17,312.68) node [anchor=north west][inner sep=0.75pt]  [font=\scriptsize]  {$I_{2}$};
\draw (182.52,313.82) node [anchor=north west][inner sep=0.75pt]  [font=\scriptsize]  {$I_{p-2}$};
\draw (262.55,314.67) node [anchor=north west][inner sep=0.75pt]  [font=\scriptsize]  {$I_{p-1}$};
\draw (375.17,132.32) node [anchor=north west][inner sep=0.75pt]  [font=\scriptsize]  {$I_{1}$};
\draw (483.34,127.17) node [anchor=north west][inner sep=0.75pt]    {$\dotsc $};
\draw (414.46,132.32) node [anchor=north west][inner sep=0.75pt]  [font=\scriptsize]  {$I_{2}$};
\draw (452.61,131.74) node [anchor=north west][inner sep=0.75pt]  [font=\scriptsize]  {$I_{3}$};
\draw (523.28,131.6) node [anchor=north west][inner sep=0.75pt]  [font=\scriptsize]  {$I_{p-2}$};
\draw (561.14,131.32) node [anchor=north west][inner sep=0.75pt]  [font=\scriptsize]  {$I_{p-1}$};

\end{tikzpicture}

%% file: fig/fig-stefan-intervals.tex
\tikzset{every picture/.style={line width=0.75pt}} 

\begin{tikzpicture}[x=0.75pt,y=0.75pt,yscale=-0.87,xscale=0.87]

\draw    (1.5,4.28) -- (1.5,283.8) ;
\draw    (388.62,283.8) -- (1.5,283.8) ;
\draw [color={rgb, 255:red, 128; green, 128; blue, 128 }  ,draw opacity=1 ]   (1.5,283.8) .. controls (174.58,-32.31) and (224.69,-31.51) .. (388.62,283.8) ;
\draw  [fill={rgb, 255:red, 74; green, 144; blue, 226 }  ,fill opacity=1 ] (361.8,239.17) .. controls (361.8,237.43) and (363.21,236.03) .. (364.94,236.03) .. controls (366.67,236.03) and (368.07,237.43) .. (368.07,239.17) .. controls (368.07,240.9) and (366.67,242.3) .. (364.94,242.3) .. controls (363.21,242.3) and (361.8,240.9) .. (361.8,239.17) -- cycle ;
\draw [color={rgb, 255:red, 74; green, 144; blue, 226 }  ,draw opacity=1 ] [dash pattern={on 0.84pt off 2.51pt}]  (226.91,54.81) -- (226.91,311.11) ;
\draw [color={rgb, 255:red, 74; green, 144; blue, 226 }  ,draw opacity=1 ] [dash pattern={on 0.84pt off 2.51pt}]  (364.94,239.17) -- (364.94,293.46) ;
\draw  [fill={rgb, 255:red, 74; green, 144; blue, 226 }  ,fill opacity=1 ] (77.2,150.17) .. controls (77.2,148.44) and (78.6,147.04) .. (80.33,147.04) .. controls (82.06,147.04) and (83.46,148.44) .. (83.46,150.17) .. controls (83.46,151.9) and (82.06,153.3) .. (80.33,153.3) .. controls (78.6,153.3) and (77.2,151.9) .. (77.2,150.17) -- cycle ;
\draw [color={rgb, 255:red, 74; green, 144; blue, 226 }  ,draw opacity=1 ] [dash pattern={on 0.84pt off 2.51pt}]  (80.33,150.17) -- (80.33,291.55) ;
\draw  [fill={rgb, 255:red, 74; green, 144; blue, 226 }  ,fill opacity=1 ] (274.75,100.54) .. controls (274.75,98.81) and (276.15,97.41) .. (277.89,97.41) .. controls (279.62,97.41) and (281.02,98.81) .. (281.02,100.54) .. controls (281.02,102.27) and (279.62,103.67) .. (277.89,103.67) .. controls (276.15,103.67) and (274.75,102.27) .. (274.75,100.54) -- cycle ;
\draw  [fill={rgb, 255:red, 74; green, 144; blue, 226 }  ,fill opacity=1 ] (192.52,48.21) .. controls (192.52,46.48) and (193.92,45.08) .. (195.65,45.08) .. controls (197.38,45.08) and (198.78,46.48) .. (198.78,48.21) .. controls (198.78,49.94) and (197.38,51.34) .. (195.65,51.34) .. controls (193.92,51.34) and (192.52,49.94) .. (192.52,48.21) -- cycle ;
\draw  [fill={rgb, 255:red, 74; green, 144; blue, 226 }  ,fill opacity=1 ] (223.77,54.81) .. controls (223.77,53.08) and (225.17,51.68) .. (226.91,51.68) .. controls (228.64,51.68) and (230.04,53.08) .. (230.04,54.81) .. controls (230.04,56.54) and (228.64,57.94) .. (226.91,57.94) .. controls (225.17,57.94) and (223.77,56.54) .. (223.77,54.81) -- cycle ;
\draw [color={rgb, 255:red, 74; green, 144; blue, 226 }  ,draw opacity=1 ] [dash pattern={on 0.84pt off 2.51pt}]  (195.65,48.21) -- (195.65,294.89) ;
\draw [color={rgb, 255:red, 74; green, 144; blue, 226 }  ,draw opacity=1 ] [dash pattern={on 0.84pt off 2.51pt}]  (293.63,120.11) -- (293.63,310.64) ;
\draw [color={rgb, 255:red, 74; green, 144; blue, 226 }  ,draw opacity=1 ]   (80.33,150.17) .. controls (53.53,-56.72) and (317.96,-24.09) .. (294.02,117.96) ;
\draw [shift={(293.63,120.11)}, rotate = 280.63] [fill={rgb, 255:red, 74; green, 144; blue, 226 }  ,fill opacity=1 ][line width=0.08]  [draw opacity=0] (8.93,-4.29) -- (0,0) -- (8.93,4.29) -- cycle    ;
\draw   (80.38,335.45) .. controls (80.38,340.12) and (82.71,342.45) .. (87.38,342.45) -- (127.9,342.45) .. controls (134.57,342.45) and (137.9,344.78) .. (137.9,349.45) .. controls (137.9,344.78) and (141.23,342.45) .. (147.9,342.45)(144.9,342.45) -- (188.42,342.45) .. controls (193.09,342.45) and (195.42,340.12) .. (195.42,335.45) ;
\draw   (195.38,336.24) .. controls (195.38,340.57) and (197.54,342.74) .. (201.87,342.74) -- (201.87,342.74) .. controls (208.06,342.74) and (211.15,344.9) .. (211.15,349.23) .. controls (211.15,344.9) and (214.24,342.74) .. (220.42,342.74)(217.64,342.74) -- (220.42,342.74) .. controls (224.75,342.74) and (226.91,340.57) .. (226.91,336.24) ;
\draw   (277.53,335.21) .. controls (277.53,337.5) and (278.68,338.64) .. (280.97,338.64) -- (280.97,338.64) .. controls (284.24,338.64) and (285.87,339.78) .. (285.87,342.07) .. controls (285.87,339.78) and (287.5,338.64) .. (290.77,338.64)(289.3,338.64) -- (290.77,338.64) .. controls (293.06,338.64) and (294.2,337.5) .. (294.2,335.21) ;
\draw   (406.33,124.51) .. controls (406.33,112.62) and (415.97,102.98) .. (427.87,102.98) .. controls (439.76,102.98) and (449.4,112.62) .. (449.4,124.51) .. controls (449.4,136.4) and (439.76,146.05) .. (427.87,146.05) .. controls (415.97,146.05) and (406.33,136.4) .. (406.33,124.51) -- cycle ;
\draw    (428.8,146.05) -- (428.8,189.17) ;
\draw [shift={(428.8,192.17)}, rotate = 270] [fill={rgb, 255:red, 0; green, 0; blue, 0 }  ][line width=0.08]  [draw opacity=0] (8.93,-4.29) -- (0,0) -- (8.93,4.29) -- cycle    ;
\draw  [fill={rgb, 255:red, 74; green, 144; blue, 226 }  ,fill opacity=1 ] (290.5,120.11) .. controls (290.5,118.38) and (291.9,116.97) .. (293.63,116.97) .. controls (295.36,116.97) and (296.77,118.38) .. (296.77,120.11) .. controls (296.77,121.84) and (295.36,123.24) .. (293.63,123.24) .. controls (291.9,123.24) and (290.5,121.84) .. (290.5,120.11) -- cycle ;
\draw  [fill={rgb, 255:red, 74; green, 144; blue, 226 }  ,fill opacity=1 ] (303.86,139.19) .. controls (303.86,137.46) and (305.26,136.06) .. (306.99,136.06) .. controls (308.72,136.06) and (310.13,137.46) .. (310.13,139.19) .. controls (310.13,140.92) and (308.72,142.33) .. (306.99,142.33) .. controls (305.26,142.33) and (303.86,140.92) .. (303.86,139.19) -- cycle ;
\draw  [fill={rgb, 255:red, 74; green, 144; blue, 226 }  ,fill opacity=1 ] (316.74,161.14) .. controls (316.74,159.41) and (318.15,158.01) .. (319.88,158.01) .. controls (321.61,158.01) and (323.01,159.41) .. (323.01,161.14) .. controls (323.01,162.87) and (321.61,164.28) .. (319.88,164.28) .. controls (318.15,164.28) and (316.74,162.87) .. (316.74,161.14) -- cycle ;
\draw  [fill={rgb, 255:red, 74; green, 144; blue, 226 }  ,fill opacity=1 ] (351.78,220.55) .. controls (351.78,218.82) and (353.19,217.42) .. (354.92,217.42) .. controls (356.65,217.42) and (358.05,218.82) .. (358.05,220.55) .. controls (358.05,222.29) and (356.65,223.69) .. (354.92,223.69) .. controls (353.19,223.69) and (351.78,222.29) .. (351.78,220.55) -- cycle ;
\draw  [fill={rgb, 255:red, 74; green, 144; blue, 226 }  ,fill opacity=1 ] (340.81,200.51) .. controls (340.81,198.78) and (342.21,197.38) .. (343.94,197.38) .. controls (345.67,197.38) and (347.07,198.78) .. (347.07,200.51) .. controls (347.07,202.24) and (345.67,203.65) .. (343.94,203.65) .. controls (342.21,203.65) and (340.81,202.24) .. (340.81,200.51) -- cycle ;
\draw [color={rgb, 255:red, 74; green, 144; blue, 226 }  ,draw opacity=1 ] [dash pattern={on 0.84pt off 2.51pt}]  (277.89,100.54) -- (277.89,293.93) ;
\draw [color={rgb, 255:red, 74; green, 144; blue, 226 }  ,draw opacity=1 ] [dash pattern={on 0.84pt off 2.51pt}]  (306.99,139.19) -- (306.99,292.98) ;
\draw [color={rgb, 255:red, 74; green, 144; blue, 226 }  ,draw opacity=1 ] [dash pattern={on 0.84pt off 2.51pt}]  (319.88,161.14) -- (319.88,311.11) ;
\draw [color={rgb, 255:red, 74; green, 144; blue, 226 }  ,draw opacity=1 ] [dash pattern={on 0.84pt off 2.51pt}]  (343.94,200.51) -- (343.94,293.46) ;
\draw [color={rgb, 255:red, 74; green, 144; blue, 226 }  ,draw opacity=1 ] [dash pattern={on 0.84pt off 2.51pt}]  (354.92,220.55) -- (354.92,311.11) ;
\draw   (294.71,335.21) .. controls (294.71,337.04) and (295.63,337.95) .. (297.46,337.95) -- (297.46,337.95) .. controls (300.07,337.95) and (301.37,338.87) .. (301.37,340.7) .. controls (301.37,338.87) and (302.68,337.95) .. (305.29,337.95)(304.12,337.95) -- (305.29,337.95) .. controls (307.12,337.95) and (308.04,337.04) .. (308.04,335.21) ;
\draw   (308.55,336.17) .. controls (308.55,338) and (309.46,338.91) .. (311.29,338.91) -- (311.29,338.91) .. controls (313.9,338.91) and (315.21,339.82) .. (315.21,341.65) .. controls (315.21,339.82) and (316.52,338.91) .. (319.13,338.91)(317.96,338.91) -- (319.13,338.91) .. controls (320.96,338.91) and (321.87,338) .. (321.87,336.17) ;
\draw   (342.91,337.6) .. controls (342.91,339.36) and (343.79,340.24) .. (345.55,340.24) -- (345.55,340.24) .. controls (348.07,340.24) and (349.33,341.12) .. (349.33,342.89) .. controls (349.33,341.12) and (350.59,340.24) .. (353.11,340.24)(351.98,340.24) -- (353.11,340.24) .. controls (354.87,340.24) and (355.75,339.36) .. (355.75,337.6) ;
\draw   (356.75,337.6) .. controls (356.75,339.36) and (357.63,340.24) .. (359.39,340.24) -- (359.39,340.24) .. controls (361.91,340.24) and (363.17,341.12) .. (363.17,342.89) .. controls (363.17,341.12) and (364.43,340.24) .. (366.95,340.24)(365.82,340.24) -- (366.95,340.24) .. controls (368.71,340.24) and (369.59,339.36) .. (369.59,337.6) ;
\draw [color={rgb, 255:red, 74; green, 144; blue, 226 }  ,draw opacity=1 ]   (195.65,48.21) .. controls (280.99,-15.95) and (400.52,173.41) .. (366.03,234.23) ;
\draw [shift={(364.94,236.03)}, rotate = 302.88] [fill={rgb, 255:red, 74; green, 144; blue, 226 }  ,fill opacity=1 ][line width=0.08]  [draw opacity=0] (8.93,-4.29) -- (0,0) -- (8.93,4.29) -- cycle    ;
\draw [color={rgb, 255:red, 74; green, 144; blue, 226 }  ,draw opacity=1 ]   (226.91,54.81) .. controls (280.75,7.04) and (373.49,176.23) .. (356.12,218.15) ;
\draw [shift={(354.92,220.55)}, rotate = 300.93] [fill={rgb, 255:red, 74; green, 144; blue, 226 }  ,fill opacity=1 ][line width=0.08]  [draw opacity=0] (8.93,-4.29) -- (0,0) -- (8.93,4.29) -- cycle    ;
\draw [color={rgb, 255:red, 74; green, 144; blue, 226 }  ,draw opacity=1 ]   (277.89,100.54) .. controls (312.04,85.43) and (328.25,135.4) .. (320.92,158.42) ;
\draw [shift={(319.88,161.14)}, rotate = 294.3] [fill={rgb, 255:red, 74; green, 144; blue, 226 }  ,fill opacity=1 ][line width=0.08]  [draw opacity=0] (8.93,-4.29) -- (0,0) -- (8.93,4.29) -- cycle    ;
\draw [color={rgb, 255:red, 74; green, 144; blue, 226 }  ,draw opacity=1 ]   (293.63,120.11) .. controls (306.39,111.77) and (310.94,125.24) .. (307.95,136.36) ;
\draw [shift={(306.99,139.19)}, rotate = 292.25] [fill={rgb, 255:red, 74; green, 144; blue, 226 }  ,fill opacity=1 ][line width=0.08]  [draw opacity=0] (8.93,-4.29) -- (0,0) -- (8.93,4.29) -- cycle    ;
\draw [color={rgb, 255:red, 74; green, 144; blue, 226 }  ,draw opacity=1 ]   (306.99,139.19) .. controls (293.02,151.24) and (258.93,125.78) .. (276.14,102.68) ;
\draw [shift={(277.89,100.54)}, rotate = 491.74] [fill={rgb, 255:red, 74; green, 144; blue, 226 }  ,fill opacity=1 ][line width=0.08]  [draw opacity=0] (8.93,-4.29) -- (0,0) -- (8.93,4.29) -- cycle    ;
\draw [color={rgb, 255:red, 74; green, 144; blue, 226 }  ,draw opacity=1 ]   (319.88,161.14) .. controls (306.19,172.94) and (273.96,154.26) .. (264.63,137.02) ;
\draw [shift={(263.36,134.31)}, rotate = 428.2] [fill={rgb, 255:red, 74; green, 144; blue, 226 }  ,fill opacity=1 ][line width=0.08]  [draw opacity=0] (8.93,-4.29) -- (0,0) -- (8.93,4.29) -- cycle    ;
\draw [color={rgb, 255:red, 74; green, 144; blue, 226 }  ,draw opacity=1 ]   (343.94,203.65) .. controls (269.85,213.5) and (198.14,106.58) .. (225.58,60.03) ;
\draw [shift={(226.91,57.94)}, rotate = 484.31] [fill={rgb, 255:red, 74; green, 144; blue, 226 }  ,fill opacity=1 ][line width=0.08]  [draw opacity=0] (8.93,-4.29) -- (0,0) -- (8.93,4.29) -- cycle    ;
\draw [color={rgb, 255:red, 74; green, 144; blue, 226 }  ,draw opacity=1 ]   (333.15,139.09) .. controls (338.33,152.9) and (348.51,172.27) .. (346.16,197.08) ;
\draw [shift={(345.85,199.8)}, rotate = 277.42] [fill={rgb, 255:red, 74; green, 144; blue, 226 }  ,fill opacity=1 ][line width=0.08]  [draw opacity=0] (8.93,-4.29) -- (0,0) -- (8.93,4.29) -- cycle    ;
\draw [color={rgb, 255:red, 74; green, 144; blue, 226 }  ,draw opacity=1 ]   (351.78,220.55) .. controls (244.76,232.61) and (151.9,127.5) .. (194.32,50.53) ;
\draw [shift={(195.65,48.21)}, rotate = 480.6] [fill={rgb, 255:red, 74; green, 144; blue, 226 }  ,fill opacity=1 ][line width=0.08]  [draw opacity=0] (8.93,-4.29) -- (0,0) -- (8.93,4.29) -- cycle    ;
\draw [color={rgb, 255:red, 74; green, 144; blue, 226 }  ,draw opacity=1 ]   (361.8,239.17) .. controls (294.46,290.67) and (85.08,223.95) .. (80.41,152.34) ;
\draw [shift={(80.33,150.17)}, rotate = 449.6] [fill={rgb, 255:red, 74; green, 144; blue, 226 }  ,fill opacity=1 ][line width=0.08]  [draw opacity=0] (8.93,-4.29) -- (0,0) -- (8.93,4.29) -- cycle    ;
\draw    (301.24,340.7) -- (301.24,370.52) ;
\draw    (349.35,342.69) -- (349.35,372.51) ;
\draw    (428.8,146.05) -- (479.23,190.2) ;
\draw [shift={(481.49,192.17)}, rotate = 221.2] [fill={rgb, 255:red, 0; green, 0; blue, 0 }  ][line width=0.08]  [draw opacity=0] (8.93,-4.29) -- (0,0) -- (8.93,4.29) -- cycle    ;
\draw    (428.8,146.05) -- (564.73,189.67) ;
\draw [shift={(567.58,190.58)}, rotate = 197.79] [fill={rgb, 255:red, 0; green, 0; blue, 0 }  ][line width=0.08]  [draw opacity=0] (8.93,-4.29) -- (0,0) -- (8.93,4.29) -- cycle    ;
\draw    (428.8,146.05) -- (618.54,189.91) ;
\draw [shift={(621.47,190.58)}, rotate = 193.02] [fill={rgb, 255:red, 0; green, 0; blue, 0 }  ][line width=0.08]  [draw opacity=0] (8.93,-4.29) -- (0,0) -- (8.93,4.29) -- cycle    ;
\draw    (428.8,146.05) -- (521.25,189.58) ;
\draw [shift={(523.96,190.86)}, rotate = 205.22] [fill={rgb, 255:red, 0; green, 0; blue, 0 }  ][line width=0.08]  [draw opacity=0] (8.93,-4.29) -- (0,0) -- (8.93,4.29) -- cycle    ;
\draw    (425.62,192.17) -- (425.62,149.32) ;
\draw [shift={(425.62,146.32)}, rotate = 450] [fill={rgb, 255:red, 0; green, 0; blue, 0 }  ][line width=0.08]  [draw opacity=0] (8.93,-4.29) -- (0,0) -- (8.93,4.29) -- cycle    ;
\draw    (481.49,192.17) -- (481.49,149.32) ;
\draw [shift={(481.49,146.32)}, rotate = 450] [fill={rgb, 255:red, 0; green, 0; blue, 0 }  ][line width=0.08]  [draw opacity=0] (8.93,-4.29) -- (0,0) -- (8.93,4.29) -- cycle    ;
\draw    (568.78,190.3) -- (568.78,147.45) ;
\draw [shift={(568.78,144.45)}, rotate = 450] [fill={rgb, 255:red, 0; green, 0; blue, 0 }  ][line width=0.08]  [draw opacity=0] (8.93,-4.29) -- (0,0) -- (8.93,4.29) -- cycle    ;
\draw    (621.47,190.58) -- (621.47,147.73) ;
\draw [shift={(621.47,144.73)}, rotate = 450] [fill={rgb, 255:red, 0; green, 0; blue, 0 }  ][line width=0.08]  [draw opacity=0] (8.93,-4.29) -- (0,0) -- (8.93,4.29) -- cycle    ;
\draw    (621.47,233.65) .. controls (621,260.81) and (685.65,211.29) .. (645.08,209.22) ;
\draw [shift={(642.46,209.15)}, rotate = 360.1] [fill={rgb, 255:red, 0; green, 0; blue, 0 }  ][line width=0.08]  [draw opacity=0] (8.93,-4.29) -- (0,0) -- (8.93,4.29) -- cycle    ;
\draw    (481.49,146.32) -- (433.9,190.8) ;
\draw [shift={(431.7,192.85)}, rotate = 316.94] [fill={rgb, 255:red, 0; green, 0; blue, 0 }  ][line width=0.08]  [draw opacity=0] (8.93,-4.29) -- (0,0) -- (8.93,4.29) -- cycle    ;
\draw    (512.83,163.03) -- (483.69,190.13) ;
\draw [shift={(481.49,192.17)}, rotate = 317.07] [fill={rgb, 255:red, 0; green, 0; blue, 0 }  ][line width=0.08]  [draw opacity=0] (8.93,-4.29) -- (0,0) -- (8.93,4.29) -- cycle    ;
\draw    (621.47,144.73) -- (572.69,190.01) ;
\draw [shift={(570.49,192.05)}, rotate = 317.13] [fill={rgb, 255:red, 0; green, 0; blue, 0 }  ][line width=0.08]  [draw opacity=0] (8.93,-4.29) -- (0,0) -- (8.93,4.29) -- cycle    ;
\draw    (568.78,144.45) -- (544.51,165.82) ;
\draw [shift={(542.25,167.8)}, rotate = 318.65] [fill={rgb, 255:red, 0; green, 0; blue, 0 }  ][line width=0.08]  [draw opacity=0] (8.93,-4.29) -- (0,0) -- (8.93,4.29) -- cycle    ;
\draw   (404.09,213.71) .. controls (404.09,201.81) and (413.73,192.17) .. (425.62,192.17) .. controls (437.51,192.17) and (447.15,201.81) .. (447.15,213.71) .. controls (447.15,225.6) and (437.51,235.24) .. (425.62,235.24) .. controls (413.73,235.24) and (404.09,225.6) .. (404.09,213.71) -- cycle ;
\draw   (459.96,124.79) .. controls (459.96,112.9) and (469.6,103.26) .. (481.49,103.26) .. controls (493.38,103.26) and (503.02,112.9) .. (503.02,124.79) .. controls (503.02,136.68) and (493.38,146.32) .. (481.49,146.32) .. controls (469.6,146.32) and (459.96,136.68) .. (459.96,124.79) -- cycle ;
\draw   (547.24,122.92) .. controls (547.24,111.03) and (556.88,101.39) .. (568.78,101.39) .. controls (580.67,101.39) and (590.31,111.03) .. (590.31,122.92) .. controls (590.31,134.81) and (580.67,144.45) .. (568.78,144.45) .. controls (556.88,144.45) and (547.24,134.81) .. (547.24,122.92) -- cycle ;
\draw   (599.93,123.2) .. controls (599.93,111.31) and (609.57,101.67) .. (621.47,101.67) .. controls (633.36,101.67) and (643,111.31) .. (643,123.2) .. controls (643,135.09) and (633.36,144.73) .. (621.47,144.73) .. controls (609.57,144.73) and (599.93,135.09) .. (599.93,123.2) -- cycle ;
\draw   (547.24,211.84) .. controls (547.24,199.95) and (556.88,190.3) .. (568.78,190.3) .. controls (580.67,190.3) and (590.31,199.95) .. (590.31,211.84) .. controls (590.31,223.73) and (580.67,233.37) .. (568.78,233.37) .. controls (556.88,233.37) and (547.24,223.73) .. (547.24,211.84) -- cycle ;
\draw   (459.96,213.71) .. controls (459.96,201.81) and (469.6,192.17) .. (481.49,192.17) .. controls (493.38,192.17) and (503.02,201.81) .. (503.02,213.71) .. controls (503.02,225.6) and (493.38,235.24) .. (481.49,235.24) .. controls (469.6,235.24) and (459.96,225.6) .. (459.96,213.71) -- cycle ;
\draw   (599.93,212.12) .. controls (599.93,200.22) and (609.57,190.58) .. (621.47,190.58) .. controls (633.36,190.58) and (643,200.22) .. (643,212.12) .. controls (643,224.01) and (633.36,233.65) .. (621.47,233.65) .. controls (609.57,233.65) and (599.93,224.01) .. (599.93,212.12) -- cycle ;

\draw (381.36,282.2) node [anchor=north west][inner sep=0.75pt]  [font=\normalsize,color={rgb, 255:red, 0; green, 0; blue, 0 }  ,opacity=1 ,xscale=0.8,yscale=0.8]  {$x$};
\draw (16.82,4.83) node [anchor=north west][inner sep=0.75pt]  [font=\footnotesize,color={rgb, 255:red, 0; green, 0; blue, 0 }  ,opacity=1 ,xscale=0.8,yscale=0.8]  {$f( x)$};
\draw (358.31,291.43) node [anchor=north west][inner sep=0.75pt]  [font=\footnotesize,color={rgb, 255:red, 74; green, 144; blue, 226 }  ,opacity=1 ,xscale=0.8,yscale=0.8]  {$x_{p-1}$};
\draw (72.25,286.22) node [anchor=north west][inner sep=0.75pt]  [font=\footnotesize,color={rgb, 255:red, 74; green, 144; blue, 226 }  ,opacity=1 ,xscale=0.8,yscale=0.8]  {$x_{p}$};
\draw (181.37,288.53) node [anchor=north west][inner sep=0.75pt]  [font=\footnotesize,color={rgb, 255:red, 74; green, 144; blue, 226 }  ,opacity=1 ,xscale=0.8,yscale=0.8]  {$x_{p-2}$};
\draw (270.6,289.32) node [anchor=north west][inner sep=0.75pt]  [font=\footnotesize,color={rgb, 255:red, 74; green, 144; blue, 226 }  ,opacity=1 ,xscale=0.8,yscale=0.8]  {$x_{3}$};
\draw (343.05,308.33) node [anchor=north west][inner sep=0.75pt]  [font=\footnotesize,color={rgb, 255:red, 74; green, 144; blue, 226 }  ,opacity=1 ,xscale=0.8,yscale=0.8]  {$x_{p-3}$};
\draw (132.29,351.74) node [anchor=north west][inner sep=0.75pt]  [font=\footnotesize,xscale=0.8,yscale=0.8]  {$I_{1}$};
\draw (204.97,350.54) node [anchor=north west][inner sep=0.75pt]  [font=\footnotesize,xscale=0.8,yscale=0.8]  {$I_{2}$};
\draw (258.38,342.28) node [anchor=north west][inner sep=0.75pt]  [font=\footnotesize,xscale=0.8,yscale=0.8]  {$I_{( p-1) /2}$};
\draw (421.42,116.78) node [anchor=north west][inner sep=0.75pt]  [font=\footnotesize,xscale=0.8,yscale=0.8]  {$I_{1}$};
\draw (513.62,115.78) node [anchor=north west][inner sep=0.75pt]  [xscale=0.8,yscale=0.8]  {$\dotsc $};
\draw (474.51,117.38) node [anchor=north west][inner sep=0.75pt]  [font=\footnotesize,xscale=0.8,yscale=0.8]  {$I_{2}$};
\draw (470.3,205.66) node [anchor=north west][inner sep=0.75pt]  [font=\footnotesize,xscale=0.8,yscale=0.8]  {$I_{p-2}$};
\draw (415.41,206.27) node [anchor=north west][inner sep=0.75pt]  [font=\footnotesize,xscale=0.8,yscale=0.8]  {$I_{p-1}$};
\draw (213.34,307.61) node [anchor=north west][inner sep=0.75pt]  [font=\footnotesize,color={rgb, 255:red, 74; green, 144; blue, 226 }  ,opacity=1 ,xscale=0.8,yscale=0.8]  {$x_{p-4}$};
\draw (283.72,306.18) node [anchor=north west][inner sep=0.75pt]  [font=\footnotesize,color={rgb, 255:red, 74; green, 144; blue, 226 }  ,opacity=1 ,xscale=0.8,yscale=0.8]  {$x_{1}$};
\draw (298.99,289.96) node [anchor=north west][inner sep=0.75pt]  [font=\footnotesize,color={rgb, 255:red, 74; green, 144; blue, 226 }  ,opacity=1 ,xscale=0.8,yscale=0.8]  {$x_{2}$};
\draw (313.07,307.93) node [anchor=north west][inner sep=0.75pt]  [font=\footnotesize,color={rgb, 255:red, 74; green, 144; blue, 226 }  ,opacity=1 ,xscale=0.8,yscale=0.8]  {$x_{4}$};
\draw (328.74,290.2) node [anchor=north west][inner sep=0.75pt]  [font=\footnotesize,color={rgb, 255:red, 74; green, 144; blue, 226 }  ,opacity=1 ,xscale=0.8,yscale=0.8]  {$x_{p-5}$};
\draw (283.4,371.81) node [anchor=north west][inner sep=0.75pt]  [font=\footnotesize,xscale=0.8,yscale=0.8]  {$I_{( p+1) /2}$};
\draw (302.6,343.23) node [anchor=north west][inner sep=0.75pt]  [font=\footnotesize,xscale=0.8,yscale=0.8]  {$I_{( p+3) /2}$};
\draw (337.44,371.38) node [anchor=north west][inner sep=0.75pt]  [font=\footnotesize,xscale=0.8,yscale=0.8]  {$I_{p-3}$};
\draw (355.73,344.18) node [anchor=north west][inner sep=0.75pt]  [font=\footnotesize,xscale=0.8,yscale=0.8]  {$I_{p-1}$};
\draw (511.24,198.09) node [anchor=north west][inner sep=0.75pt]  [xscale=0.8,yscale=0.8]  {$\dotsc $};
\draw (548.99,115.38) node [anchor=north west][inner sep=0.75pt]  [font=\footnotesize,xscale=0.8,yscale=0.8]  {$I_{( p-3) /2}$};
\draw (602.87,115.19) node [anchor=north west][inner sep=0.75pt]  [font=\footnotesize,xscale=0.8,yscale=0.8]  {$I_{( p-1) /2}$};
\draw (603.28,205.87) node [anchor=north west][inner sep=0.75pt]  [font=\footnotesize,xscale=0.8,yscale=0.8]  {$I_{( p+1) /2}$};
\draw (550.59,205.07) node [anchor=north west][inner sep=0.75pt]  [font=\footnotesize,xscale=0.8,yscale=0.8]  {$I_{( p+3) /2}$};

\end{tikzpicture}

%% file: supp_doubling.tex
\section{Additional Proofs for Section~\ref{sec:phase}}\label{asec:phase}

\subsection{Preliminaries}

Before reintroducing and proving the theorems about the doubling and chaotic regime, we introduce topological entropy and define VC-dimension.

\subsubsection{Topological Entropy}
Topological entropy is a well-known measure of function complexity in dynamical systems that measures the ``bumpiness'' of a mapping. 
Like we do with chaotic itineraries, \cite{bzl20}~draw analogies between the neural network approximability of $f^k$ and the topological entropy of $f$.
We do not give a rigorous definition of topological entropy, but we include a well known result connecting topological entropy to the number of monotone pieces (not constant-sized crossings), which is stated as Lemma~3 of the aforementioned work.

\begin{lemma}\label{lemma:htop}[\cite{mw80, young81}]
    If $f:[0,1]\to[0,1]$ is continuous and piece-wise monotone, then the topological entropy of $f$ satisfies the following:
    \[\htop(f) = \lim_{k \to \infty} \frac{1}{k} \log \oscm(f^k).\]
\end{lemma}

\subsubsection{VC-Dimension}
We capture the complexity of the mappings produced by repeated application of $f$, by measuring the capability of a family of iterates to fit arbitrarily-labeled samples with the VC-dimension. 
For some threshold parameter $t \in (0, 1)$, we first define a hypothesis class  that we use to cast this family of iterated functions as Boolean-valued.

\begin{definition}
    For some unimodal $f: [0,1]\to[0,1]$ and threshold $t \in (0,1)$, let \[\hyp{f,t} := \{\thres{t}{f^k}: k \in \mathbb{N}\}\]
    be the Boolean-valued hypothesis class of classifiers of composed functions.
\end{definition}

The following is the standard definition of the VC-dimension:

\begin{definition}[\cite{vc13}]
    For some hypothesis class $\mathcal{H}$ containing functions $[0,1] \to \{0,1\}$, we say that $\mathcal{H}$ \emph{shatters} samples $x_1, \dots, x_d \in [0,1]$ if for every labeling of the samples $\sigma_1, \dots, \sigma_d \in \bit$, there exists some $h\in \mathcal{H}$ such that $h(x_i) = \sigma_i$ for all $i \in [d]$. 
    The \emph{VC-dimension} of $\mathcal{H}$, $\vc(\mathcal{H})$ is the maximum $d$ such that there exists $x_1, \dots, x_d \in [0,1]$ that $\mathcal{H}$ shatters.
\end{definition}

$\vc(\hyp{f,t})$ will be a useful measurement of complexity of the mapping $f$, which as we show is tighly connected with the notion of periodicity and oscillations. 
Notably, this is a measurement of the complexity of iterated maps and is \textit{not} a typical formulation of VC-dimension for neural networks, since those typically would consider a fixed depth and a fixed width, but variable values for the weights, rather than fixed $f$ and variable $k$.



\subsection{Proof for Theorem~\ref{thm:properties-doubling} and \ref{thm:properties-chaotic}}

\thmpropertiesdoubling*
\begin{proof}
    Claim 1 follows from a somewhat involved argument in Appendix~\ref{assec:doubling-monotone} that uses an inductive argument to compare the behavior of a mapping with a maximal $p$-cycle to one with a maximal $\frac{p}{2}$-cycle.
    By categorizing intervals of $[0,1]$ based on how $f^k$ behaves on that interval, we analyze how $f^{k+1}$ in turn behaves, which leads to a bound on the monotone pieces $\oscm(f^k)$.
    
    Claim 2 is a simple consequence of Claim 1, by using the fact that a ReLU network can piecewise approximate each monotone piece of $f^k$.
    This argument appears in Appendix~\ref{assec:doubling-nn}.
    
    Claim 3 follows easily from Claim 1 and Lemma~\ref{lemma:htop}. We note that this derivation about the topological entropy and the periodicity of $f$ is a known fact in the dynamical systems community.
    
    Claim 4 relies on another recursive argument that frames VC-dimension in terms of the possible trajectories of $f^k(x)$ for fixed $x$ and changing $k$. 
    We characterize these trajectories by making use of Regular Expressions and by bounding the corresponding VC dimension in Appendix~\ref{assec:doubling-vc}.
\end{proof}

\thmpropertieschaotic*
\begin{proof}
    Claims 1 and 2 are immediate implications Theorems~1.5 and 1.6 of~\cite{cnpw19}.
    Claim 3 follows by applying Lemma~\ref{lemma:htop} to Claim 1 (again this derivation about the topological entropy is basic in the literature on dynamical systems).
    
    The most interesting part of the theorem is the last claim. We prove Claim 4 in Appendix~\ref{asec:chaotic} by showing that the VC-dimension of the class is at least $d$ for all $d \in \N$.
    The argument relies on the existence of an infinite number of cycles of other lengths, as guaranteed by Sharkovsky's Theorem.
\end{proof}

\subsection{Proof of Theorem~\ref{thm:properties-doubling}, Claim 1}\label{assec:doubling-monotone}

We restate Claim~1 of the theorem as the following proposition and prove it.

\begin{proposition}[Claim 1 of Theorem~\ref{thm:properties-doubling}]\label{prop:properties-doubling-osc}
    Suppose $f$ is a symmetric unimodal mapping whose maximal cycle is of length $p = 2^q$. Then, for any $k \in \N$, $M(f^k) = O((4k)^{q+1})$.
\end{proposition}

In order to bound the number of times $f$ oscillates based on its power-of-two periods, we categorize $f$ by its cyclic behavior and the bound the number of local maxima and minima $f$ has based on its characterization.

\begin{definition}[Category]
    For $q \geq 0$ and $z \in \bit$, let $\category{q,z}$ contain the set of all symmetric unimodal functions $f$ such that (1) $f$ has a $2^q$-cycle, (2) $f$ does not have a $2^{q+1}$-cycle, and (3) $\thres{1/2}{f^{2^q}(\frac{1}{2})} = z$.
\end{definition}


We abuse notation to let $\oscm(\category{q,z}^k) = \max_{f \in \category{q,z}} \oscm(f^k)$.
Thus, for $f$ given in the theorem statement with a $2^q$-cycle, but not a $2^{q+1}$-cycle, our final bound is obtained by 
\[\oscm(f^m) \leq \max\{\oscm(\category{q,0}^m), \oscm(\category{q,1}^m)\}.\]
We let $\oscm(f, a, b)$ represent the number of monotone pieces of $f$ on the sub-interval $[a, b] \subset [0, 1]$. 

We build a large-scale inductive argument by first bounding base cases $\oscm(\category{0,0}^k)$ and $\oscm(\category{0, 1}^k)$. 
Then, we relate $\oscm(\category{q, z}^k)$ to $\oscm(\category{q-1, 1-z}^k)$ to get the desired outcome. 

Before beginning the proof, we state a slight refinement of the part of the theorem, which takes into account the newly-introduced categories, from which the claim follows.
\begin{proposition}\label{prop:osc-category-bound}
For any $k \in \N$, $q\geq 0$, and $z \in \bit$, 
\[\oscm(\category{q, z}^k) \leq \begin{cases}
   2(3q)^k & \text{$q$ is even, $z = 0$, or $q$ is odd, $z = 1$}\\
   2(3q)^{k+1} & \text{$q$ is even, $z = 1$, or $q$ is odd, $z = 0$}.
\end{cases}\]
\end{proposition}

Thus, proving Proposition~\ref{prop:osc-category-bound} is sufficient to prove Proposition~\ref{prop:properties-doubling-osc}.
The remainder of the section proves Proposition~\ref{prop:osc-category-bound}.

\subsubsection{Special Case Proof for $q =1$}\label{sec:ub-special-case}
We show that $\oscm(\category{0,0}^k) = 2$ and $\oscm(\category{0,1}^k) = 2k$.

For $f_r$ as defined above, we characterize the number of oscillations that are added by increasing $r$ past $\frac{1}{2}$, where super-stability of a fixed point exists.
Figure \ref{fig:ub-specific} illustrates those results.

\begin{figure}
    \centering
    \includegraphics[width=0.49\textwidth]{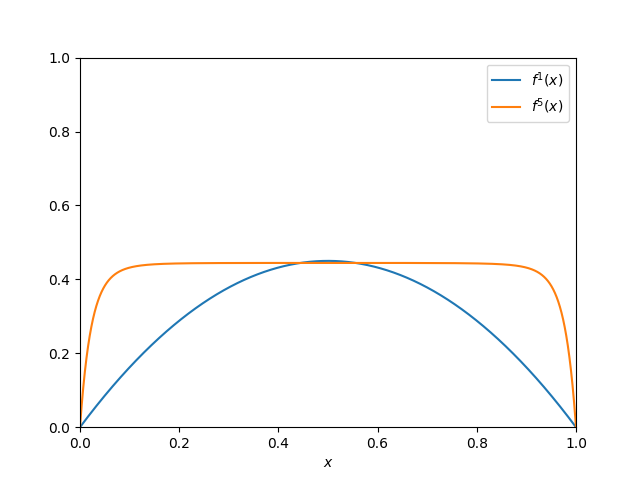}
    \includegraphics[width=0.49\textwidth]{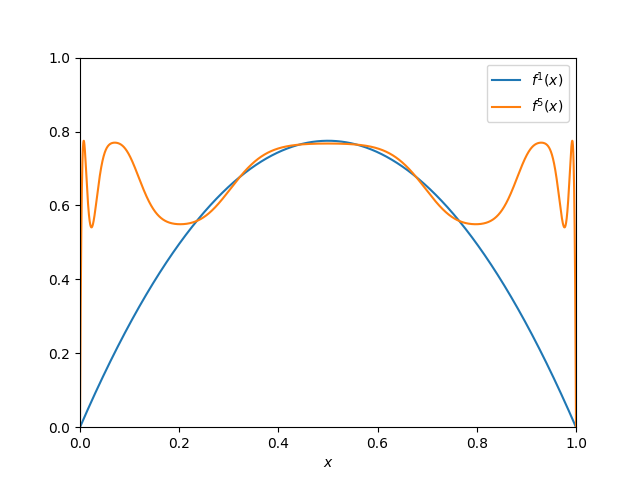}
    \caption{The base case results of  Proposition~\ref{prop:osc-category-bound} demonstrate the number of oscillations of $f^k$ increases when $f$ moves from $\category{0, 0}$ to $\category{0, 1}$.
    The plots show $f$ and $f^5$ for $f \in \category{0, 0}$ ($f = \flog{0.45}$) on the left and $f \in \category{0,1}$ ($f =\flog{0.775}$) on the right.}
    \label{fig:ub-specific}
\end{figure}

To analyze the oscillation patterns of $f^k$, we define several ``building blocks,'' which  represent disjoint pieces of $f^k$.
That is, the interval $[0,1]$ can be partitioned into several sub-intervals, each of which has $f^k$ follow certain simple behavior that we categorize.
We argue that any iterate can be decomposed into those pieces and then show how applying $f$ to $f^{k}$ modifies the pieces in order to analyze $f^{k+1}$. 
Here are the function pieces that we analyze, which map interval $[a, b] \subseteq [0, 1]$ to $[0, 1]$:
\begin{definition}
    For any $f: [0, 1] \rightarrow [0, 1]$ and for any $[a, b] \subseteq [0, 1]$, $f$ is referred to on interval $[a, b]$ as:
    \begin{itemize}
        \item a \textbf{increasing crossing piece} $\ic$ if $f$ is strictly increasing on $[a, b]$ and has $f(a) = 0$, $f(b) > \frac{1}{2}$, and $f'(b) > 0$;
        \item a \textbf{decreasing crossing piece} $\dc$ if $f$ is strictly decreasing on $[a,b]$ and has $f(a) > \frac{1}{2}$, $f(b) = 0$, and $f(a) < 0$;
        \item a \textbf{up peak} $\up$ if there exists some $c \in (a, b)$ that maximizes $f$ on $[a,b]$, $f$ is strictly increasing on $[a, c)$, $f$ is strictly decreasing on $(c, b]$, and $f(x) > \frac{1}{2}$ for all $x \in [a, b]$;
        \item a \textbf{up valley} $\uv$ if there exists some $c \in (a, b)$ that minimizes $f$ on $[a,b]$, $f$ is strictly decreasing on $[a, c)$, $f$ is strictly increasing on $(c, b]$, and $f(x) > \frac{1}{2}$ for all $x \in [a, b]$; and
        \item a \textbf{down peak} $\dpk$ if there exists some $c \in (a, b)$ that maximizes $f$ on $[a,b]$, $f$ is strictly increasing on $[a, c)$, $f$ is strictly decreasing on $(c, b]$, and $f(x) \leq \frac{1}{2}$ for all $x \in [a, b]$.
    \end{itemize}
    If there exists a sequence of intervals $J_1, \dots, J_m$ such that $f$ is piece $\eta_i$ on $J_i$, then we represented $f$ with the string $\eta_1 \dots \eta_m$.
\end{definition}

\begin{figure}
    \centering
    \includegraphics[width=0.6\textwidth]{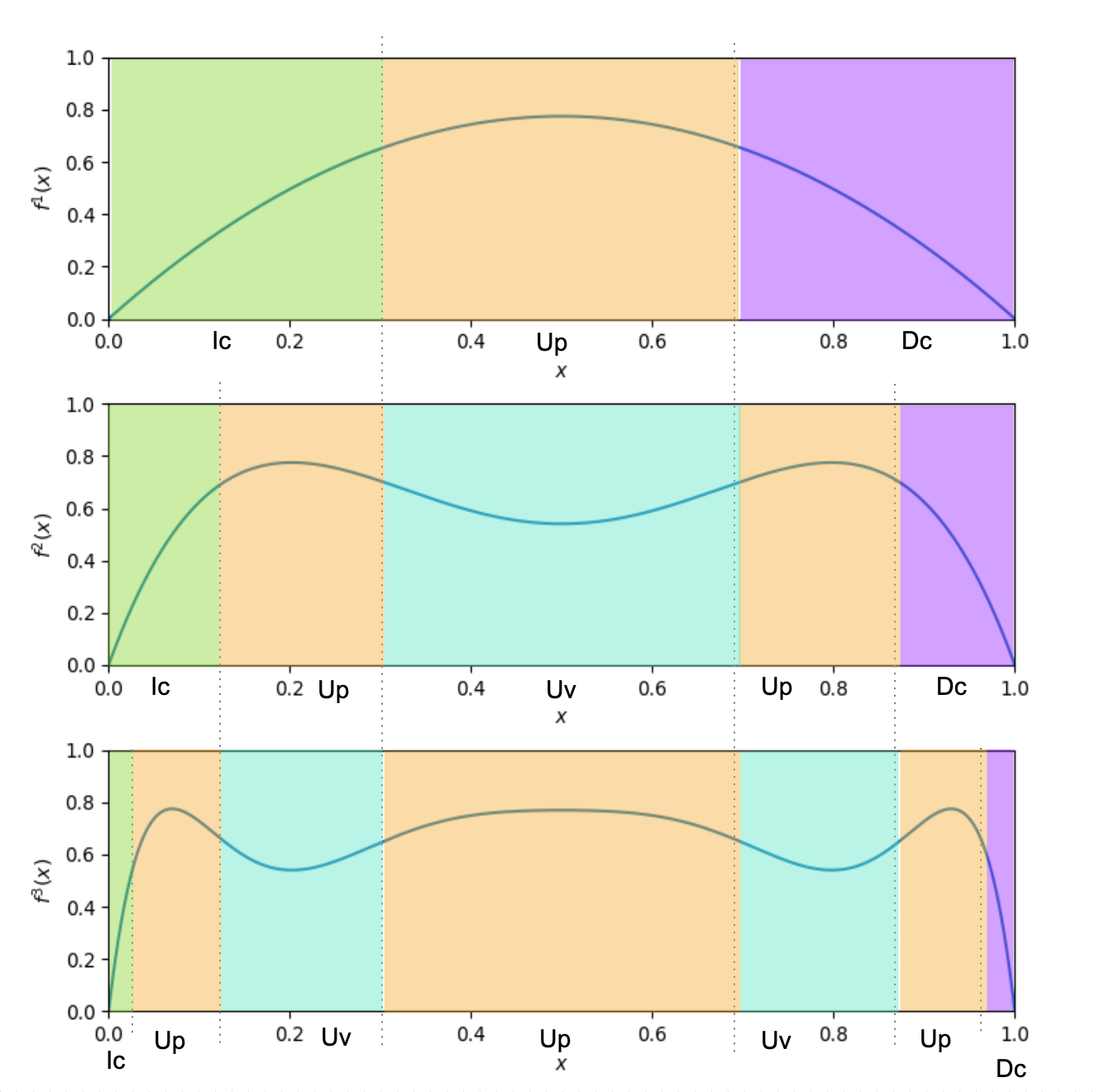}
    \caption{For $f\in\category{0,1}$ ($f = \flog{0.775}$), visualizes the decomposition of $f$, $f^2$, and $f^3$ into $\ic\up\dc$, $\ic\up\uv\up\dc$, and $\ic(\up\uv)^2\up\dc$ respectively.}
    \label{fig:ub-specific-classification}
\end{figure}

We specify an invariant for each part of the theorem, such that proving the invariant is sufficient to prove the proposition:
\begin{enumerate}
    \item If $f \in \category{0, 0}$, then $f^k$ is a down peak on $[0,1]$ for all $k$, and $f^k$ has two monotone pieces.
    \item If $f \in \category{0, 1}$, $f$ is represented by $\ic(\up\uv)^{k-1}\up\dc$. That is, $[0, 1]$ can be partitioned into $2k + 1$ subsequent intervals $J_1, \dots J_{2k+1}$ such that $f^k$ is an increasing crossing piece on $J_1$, a decreasing crossing piece on $J_{2k+1}$ (if $k \neq 0$), an up peak on $J_{2j}$ for $j \in \{1, \dots, k\}$, and a up valley on $J_{2j + 1}$ for $j \in \{1, \dots, k-1\}$.
    Hence, $f^k$ has $k$ distinct maxima and $2k$ monotone pieces.
    Figure \ref{fig:ub-specific-classification} illustrates this invariant.
\end{enumerate}


\textbf{Base Case:}
\begin{enumerate}
    \item For $f\in\category{0,0}$, $f^1 = f$ is trivially a down peak on $[0, 1]$ by the definition of $\category{0,0}$, since $\frac{1}{2}$ maximizes $f$.
    \item For $f\in\category{0,1}$, $f$ can be represented by $\ic\up\dc$. That is, $[0, 1]$ can be decomposed into intervals $I_1$, $I_2$, and $I_3$, on which $f_r$ is an increasing crossing piece, an up peak, and a decreasing crossing piece respectively.
\end{enumerate}

\textbf{Inductive Step:}

We examine what happens to each function piece when $f$ is applied to it. We can use the following analysis, along with the inductive hypothesis to show that $f^{k+1}$ can be decomposed as we expect it to be.
\begin{enumerate}
    \item Examining the \textbf{down peak} proves first invariant for the case when $f\in \category{0,0}$. Because $f$ strictly increases on $[0, \frac{1}{2}]$ and because $f([0, 1]) \subseteq [0, \frac{1}{2}]$ if $f^k$ is a down peak, $f \circ f^k$ also supports a down peak on $[0, 1]$.
    
    Because we inductively assume that $f^k$ is a low peak on $[0, 1]$, it then follows that $f^{k+1}$ is also a down peak on $[0, 1]$.

    \item 
    We first prove a claim, which implies that $f$ has no down peaks for $f \in \category{0,1}$. Let $\xmax = f(\frac{1}{2})$,
    \begin{claim}
    If $f \in \category{0,1}$, then $f((\frac{1}{2}, \xmax]) \subseteq (\frac{1}{2}, \xmax]$.
    \end{claim}
    \begin{proof}
    Because $\frac{1}{2}$ maximizes $f$, $f(x) \leq \xmax$ for all $x \in [\frac{1}{2}, \xmax]$.
    Since $f$ monotonically decreases, on $[\frac{1}{2}, \xmax]$, the claim can only be false if $f(\xmax) < \frac{1}{2}$.
    We show by contradiction that this is impossible.
    
    Because $f$ is continuous and monotonically increases on $[0, \frac{1}{2}]$ and ranges from $0$ to $\xmax \geq \frac{1}{2}$, there exists some $x' \leq \frac{1}{2}$ such that $f(x') = \frac{1}{2}$ and $f^2(x') = \xmax$.
    
    Let $g(x) = f^2(x) - x$.
    By assumption, $g(\frac{1}{2}) = f(\xmax) - \frac{1}{2} < 0$. 
    By definition of $x'$, $g(x') = \frac{1}{2} - x' \geq 0$.
    Because $g$ is continuous, the Intermediate Value Theorem implies the existence of $x'' \in [x', \frac{1}{2})$ such that $g(x'') = 0$ and $f^2(x'') = x''$.
    Since $f$ has no two-cycles, it must be the cause that $f(x'') = x''$ and $x'' = \frac{1}{2}$.
    However, this contradicts our finding that $x'' < \frac{1}{2}$, which means that $f(\xmax) \geq \frac{1}{2}$ and the claim holds.
    \end{proof}
    
    Now, we proceed with analyzing each of the function pieces on some interval $[a, b] \subseteq [0, 1]$ when $f \in \category{0,1}$. The transformations are visualized in Figure \ref{fig:ub-specific-classification}. 
    
    \begin{itemize}
        \item \textbf{Increasing crossing piece}: If $f^k$ has an $\ic$ on $[a,b]$, then $f^{k+1}$ can be represented by $\ic\up$ on $[a,b]$. 
        
        There exist $c$ and $d$ such that $a < d < c < \frac{1}{2} < b$, $f^k(c) = \frac{1}{2}$, and $f^k(d) = c $. Then, $[a, \frac{1}{2}(c + d)]$ supports an increasing crossing piece on $f \circ f^k$---because $f(f^k(a)) = 0$, $f(f^k(\frac{1}{2}(c + d))) > \frac{1}{2}$, and $f \circ f^k$ is strictly increasing on that interval since $f$ is increasing before reaching $\frac{1}{2}$. $[0.5(c + d), b]$ supports a high peak---because $c$ is a local maxima on $f \circ f^k$, and $f \circ f^k$ is strictly increasing before $c$ and strictly decreasing after $c$.
        \item \textbf{Decreasing crossing piece}: For the same arguments, $f^{k+1}$ can represented by $\up\dc$ on $[a,b]$ if $f^k$ is represented by $\dc$ on $[a,b]$.
        \item \textbf{Up peak}: Because $f$ strictly decreases for $x > \frac{1}{2}$ and because $f^k([a, b]) \subseteq (\frac{1}{2}, \xmax]$ if $\up$ represents $f^k$ on $[a, b]$, $c$ becomes a local minimum for $f \circ f^k$, and $f^{k+1}$ is a high valley $\uv$ on $[a,b]$.
        \item \textbf{Up valley}:  Because $f$ strictly decreases for $x > \frac{1}{2}$ and because $f^k([a, b]) \subseteq (\frac{1}{2}, \xmax]$ if $\uv$ represents $f^k$ on $[a, b]$, $c$ becomes a local maximum for $f \circ f^k$, and $f^{k+1}$ is a high peak $\up$ on $[a,b]$.
    \end{itemize}
    
    Now, consider the inductive hypothesis. Because $f^k$ can be represented by $\ic(\up\uv)^{k-1}\up\dc$, applying the above transformations to each piece implies that $f^{k+1}$ can be represented by $\ic(\up\uv)^{k}\up\dc$. Hence, the inductive argument goes through.
\end{enumerate}

\subsubsection{General Case Proof}\label{sec:ub-general-case}
The argument proceeds inductively.
We show that if we have some $f \in \category{q,k}$, then we can find some other function $h \in \category{q-1,1-k}$ and characterize the behavior of $f$ in terms of the behavior of $h$.

Since we assume that $q \geq 1$, there will always exist some $x^* > \frac{1}{2}$ that is a fixed point of $f$.\footnote{Sharkovsky's Theorem yields this by showing that the existence of a $2^q$-cycle implies the existence of any $2^j$-cycle, for all $j \in \{0, \dots, q-1\}$. $x^* > \frac12$ by our assumption that a 2-cycle $x_1 < x_2$ exists. It must be true that $x_2 > \frac{1}{2}$; otherwise, $f(x_2) > x_2 > x_1$, which breaks the cycle. Because $f(\frac{1}{2}) > \frac{1}{2}$ and $f(x_2) < x_2$, there exists $x^* \in (\frac{1}{2}, x_2)$ such that $f(x^*) = x^*$ by the Intermediate Value Theorem.}
By symmetry, $f(1- x^*) = x^*$.
Let $\phi:[0,1] \to [1-x^*, x^*]$ be a decreasing isomorphism with $\phi(x) = x^* - x(2x^* - 1)$, and let
\[h = \phi^{-1} \circ f^2 \circ \phi.\]
$h$ is a useful construct, because its behavior resembles simpler versions of $f$, with fewer cycles and oscillations. We use properties of $h$ to relate pieces of $f^k$ to those of $h^{k/2}$. 
We illustrate this recursive and fractal-like behavior in Figure \ref{fig:ub-general-hr}.

Note that $h^k = \phi^{-1} \circ f^{2k} \circ \phi$.

\begin{figure}
    \centering
    \includegraphics[width=0.8\textwidth]{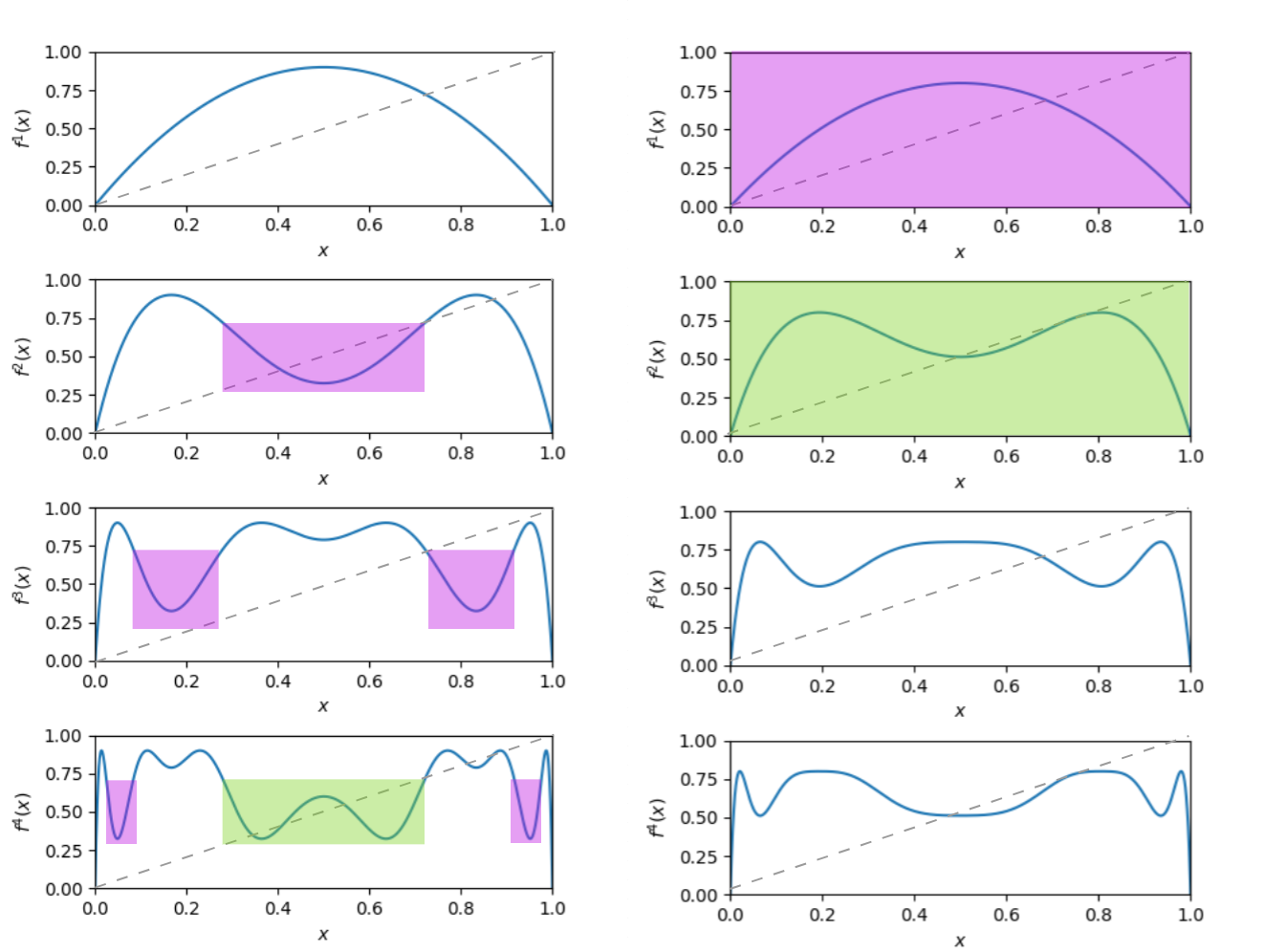}
    \caption{Visualizes the analogy between mappings in $\category{q, z}$ and $\category{q-1,1-z}$. The left plots the first 4 iterates of $f= \flog{0.9} \in \category{4, 1}$ (has a maximal 4-cycle with $f^4(\frac{1}{2}) > \frac{1}{2}$), while the right plots those of $f = \flog{0.85} \in \category{2,0}$ (has a maximal 2-cycle with $f^2(\frac{1}{2}) < \frac{1}{2}$. The purple highlighted regions on the left behave qualitatively similar to $\flog{0.85}$, while the green regions are similar to $\flog{0.85}^2$.}
    \label{fig:ub-general-hr}
\end{figure}

\begin{lemma}\label{lemma:doubling-h}
    $h$ is a symmetric unimodal mapping with $h \in \category{q-1, 1-z}$.
\end{lemma}
\begin{proof}
We verify the conditions for $f$ to be unimodal mapping.
\begin{enumerate}
    \item $h$ is continuous and piece-wise differentiable on $[0, 1]$ because $f^2$ is, and $h$ is merely a linear transformation of $f^2$.
    
    \item $h(0) = h(1) = 0$. $h((0, 1))$ is strictly positive because $f((1- x^*, x^*)) = (x^*, \xmax)$, $f^2((1- x^*, x^*)) = (f(\xmax), x^*)$, and $f(\xmax) < f(x^*) = x^*$ by $f$ being decreasing on $[\frac{1}{2}, 1]$.
    
    \item $h$ is uniquely maximized by $\frac{1}{2}$ because $\frac{1}{2}$ minimizes $f^2$ on the interval $[1 - x^*, x^*]$. 
    $f$ maps both $[1-x^*, \frac{1}{2}]$ and $[\frac{1}{2}, x^*]$ onto $[x^*, \xmax]$ and is increasing and decreasing on the respective intervals.
    Because $f$ maps $[x^*, \xmax]$ onto $[f(\xmax), x^*]$ and $f(\xmax) < x^*$ and is decreasing on $[x^*, \xmax]$, $f^2$ is increasing on $[1-x^*, \frac{1}{2}]$ and decreasing on $[\frac{1}{2}, x^*]$.
    
    Thus, $h$ is maximized by $\frac{1}{2}$, increases before $\frac{1}{2}$, and decreases after $\frac{1}{2}$.
        
    \item We must also show that $h$ is well-defined, which entails proving that $h(x) \leq 1$ for all $x \in [0, 1]$. 
    Suppose that were not the case. 
    Then, $h(\frac{1}{2}) > 1$, and there exists some $x' \leq \frac{1}{2}$ with $h(x') = 1$.
    There also exists some $x^{**} \in [1-x', 1]$ with $h(x^{**}) = x^{**}$ by the Intermediate Value Theorem.

    Let $g(x) = h^3(x) - x$ and note that $g$ is continuous on $[0, x']$. 
    Observe that $g(1-x^{**}) = 2x^{**} - 1 > 0$ and $g(x') = -x' < 0$.
    Thus, there exists $x'' \in [1 - x^{**}, x']$ with $g(x'') = 0$.
    Because $h$ is increasing on $[0, x']$ and $x' > 1 - x^{**}$, it must be the case that $h(x') > x^{**} > x'$.
    Thus, $x^{**}$ is not a fixed point and must be on a 3-cycle in $h$.
    
    However, if $x^{**}$ is on a 3-cycle in $h$, then $\phi(x^{**})$ must be part of a 6-cycle in $f$.
    This contradicts the assumption that $f$ cannot have a $2^{q+1}$-cycle, because Sharkovsky's Theorem states that a 6-cycle implies a $2^{q+1}$-cycle.

\end{enumerate}

We show that $h$ is symmetric.
    \begin{align*}
        h(x)
        &= \phi^{-1}(f^2(\phi(x))) 
        = \phi^{-1}( f^2(1 - \phi(x)))
        = \phi^{-1}(f^2(1 - x^* + x(2x^* - 1))) \\
        &= \phi^{-1}(f^2(x^* - (1-x)(2x^* - 1)))
        = \phi^{-1}(f^2(\phi(1-x))) = h(1-x).
    \end{align*}
    
If $f^{2^{q}}(\frac{1}{2}) \geq \frac{1}{2}$, then $h^{2^{q-1}}(\frac{1}{2}) \leq \frac{1}{2}$, and if $f^{2^{q}}(\frac{1}{2}) \leq \frac{1}{2}$, then $h^{2^{q-1}}(\frac{1}{2}) \geq \frac{1}{2}$. Thus, $\thres{1/2}{h^{2^{q-1}}(\frac{1}{2})} = \thres{1/2}{f^{2^q}(\frac{1}{2})}$.
By Lemma~\ref{lemma:hr-super-stable}, $h$ has a $2^{q-1}$-cycle and does not have a $2^{q}$-cycle. Thus, $h \in \category{q-1, 1-z}$.
\end{proof}

\begin{lemma}\label{lemma:hr-super-stable}
    For $p \in \mathbb{Z}_{+}$, $h$ has a $p$-cycle if and only if $f$ has a $2p$-cycle.
\end{lemma}
\begin{proof}
    Suppose $x_1, \dots, x_p$ is a $p$-cycle for $h$. Then,
    $\phi(x_1), \dots ,\phi(x_p)$
    is a $p$-cycle for $f^2$. 
    If $x_1, \dots, x_p$ are distinct, then so must be $\phi(x_1), \dots ,\phi(x_p)$, since $\phi$ is an isomorphism.
    Thus, 
    \[\phi(x_1), f(\phi(x_1)), \dots, \phi(x_p), f(\phi(x_p))\]
    is a $2p$-cycle for $f$.
    
    Conversely, if $x_1, \dots, x_{2p}$ is a $2p$-cycle for $f$, then $x_1, x_3, \dots, x_{2p-1}$ is a $p$-cycle for $f^2$ and 
    \[\phi^{-1}(x_1), \dots, \phi^{-1}(x_{2p})\]
    is a $p$-cycle for $h$.
\end{proof}

We proceed with a proof similar in structure to the one in the last section, where we divide each $f^{k}$ into intervals and monitor the evolution of each as $k$ increases. We define the classes of the pieces of some 1-dimensional map $f^{k}$ on interval $[a, b]$ below. We visualize these classes in Figure \ref{fig:ub-general-classification}.
\begin{itemize}
    \item $f^{k}$ is an \textbf{approach} $\app$ on $[a, b]$ if $f$ is strictly increasing, $f^{k}(a) = 0$, and $f^{k}(b) = 1-x^*$.
    \item Similarly, $f^{k}$ is a \textbf{departure} $\dep$ on $[a, b]$ if $f^k$ is strictly decreasing, $f^k(a) = 1-x^*$, and $f^k(b) = 0$.
    \item $f^{k}$ is an \textbf{$i$-Left Valley} $\lv_i$ on $[a, b]$ if $f^{k}:[a, b] \rightarrow [f(\xmax), x^*]$ and if there exists some strictly increasing and bijective $\sigma: [a, b] \rightarrow [1 - x^*, x^*]$ such that $f^{k} = \phi \circ h^i \circ \phi^{-1} \circ \sigma$ on $[a, b]$. Note that $f^{k}(a) = f^{k}(b) = x^*$---unless $i = 0$, in which case $f^k(a) = 1-x^*$ and $f_r^k(b) = x^*$.
    \item $f^{k}$ is analogously a \textbf{$i$-Right Valley} $\rv_i$ if the same condition holds, except that $\sigma$ is strictly decreasing.
    \item $f^{k}$ is an \textbf{$i$-Left Peak} $\lp_i$ on $[a, b]$ if $f^{k-1}$ is $\lv_{i-1}$ on $[a, b]$. It follows that $f^{k}:[a, b] \rightarrow [x^*, \xmax]$, that there exists some $c \in [a, b]$ such that $f^{k}(c) = \xmax$ (because $\frac{1}{2} \in [f(\xmax), x^*]$), and that $f^{k}(a) = f^{k}(b) = x^*$.
    \item $f^{k}$ is an \textbf{$i$-Right Peak} $\rp_i$ on $[a, b]$ if $f^{k-1}$ is $\rv_{i-1}$ on $[a, b]$. The same claims hold as $\lp_i$.
 \end{itemize}
 
 \begin{figure}
     \centering
     \includegraphics[width=0.6\textwidth]{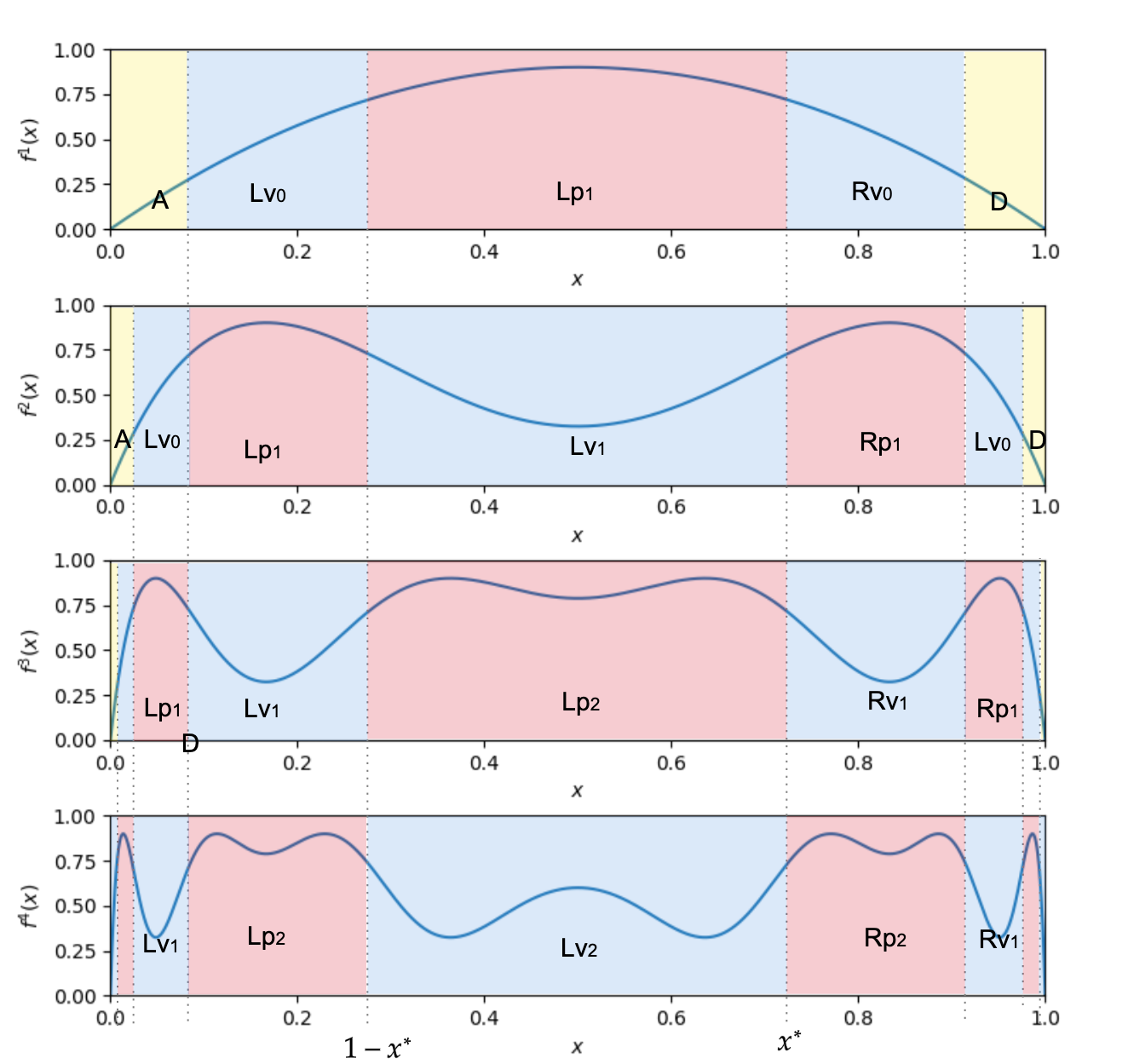}
     \caption{Similar to Figure~\ref{fig:ub-specific-classification}, visualizes the classifications of $f, f^2, f^3, f^4$ for $f=\flog{0.9} \in \category{4,1}$, and demonstrates that the decompositions are $\app\lv_0\lp_0\rv_0\dep$, $\app\lv_0\lp_1\lv_1\rp_1\rv_0\dep$, $\app\lv_0\lp_1\lv_1\lp_2\rv_1\rp_1\rv_0\dep$, and $\app\lv_0\lp_1\lv_1\lp_2\lv_2\rp_2\rv_1\rp_1\rv_0\dep$ respectively.}
     \label{fig:ub-general-classification}
 \end{figure}
 
 Now, the proof of the number of oscillations proceeds in two steps. (1) We analyze how each of the above pieces evolves with each application of $f$. (2) We show how many maxima and minima each translates to.
 
\begin{lemma}\label{lemma-piece-decomposition}
     When $f \in \category{q,z}$ for $q \geq 1$ and for all $k \in \mathbb{Z}_{+}$, $f^k$ can be decomposed into $2k + 3$ pieces $\eta_1, \dots, \eta_{2k + 3}$ such that
     \[\text{$\eta_i$ is }
     \begin{cases}
        \text{$\app$ if $i = 1$} \\
        \text{$\lv_j$ if $i = 2j + 2$ for $j \in \{0, 1, \dots, \floor{k / 2}\}$} \\
        \text{$\lp_j$ if $i = 2j + 1$ for $j \in \{1, \dots, \floor{(k+1)/2}\}$} \\
        \text{$\rv_j$ if $i = 2k - 2j + 2$ for $j \in \{0, 1, \dots, \floor{(k-1)/ 2}\}$} \\
        \text{$\rp_j$ if $i = 2k - 2j + 3$ for $j \in \{1, \dots, \floor{k/2}\}$} \\
        \text{$\dep$ if $i = 2k + 3$} \\
     \end{cases}\]
     
     That is, if $k$ is even, then $f$ can be represented by \[\app \lv_0 \lp_1 \lv_1 \dots \lv_{k/2 - 1} \lp_{k/2} \lv_{k/2} \rp_{k/2} \rv_{k/2 - 1} \dots \rv_1 \rp_2 \rv_0\dep.\]
     If $k$ is odd, then $f$ is represented by
     \[\app \lv_0 \lp_1 \lv_1 \dots \lp_{(k-1)/2} \lv_{(k-1)/2}\lp_{(k+1)/2} \rv_{(k-1)/2} \rp_{(k - 1)/2} \dots \rv_1 \rp_2 \rv_0\dep.\]
\end{lemma}
\begin{proof}
    This lemma is proved inductively. $f$ can be decomposed into the pieces $\app\lv_0 \lp_1\rv_0 \dep$.
    \begin{itemize}
        \item 
        By unimodality and symmetry, $f$ is strictly increasing on $[0, \frac{1}{2})$ and strictly decreasing on $(\frac{1}{2}, 1]$. 
        There exists some $x_1$ such that $[0, x_1]$ is strictly increasing and $f(x_1) = 1 - x^*$ (because $1 - x^* < x^* < \xmax$). Thus, $f$ is $\app$ on $[0, x_1]$.
        Similarly, $[1-x_1, 1]$ is strictly decreasing and $f(1- x_1) = 1 - x^*$, which implies that $f$ is $\dep$ on $[1-x_1, 1]$.
        
        \item Note that $x_1 < 1 - x^* < x^* < 1-x_1$, and $f$ is increasing on $[x_1, 1 - x^*]$ and decreasing on $[x^*, 1-x_1]$. 
        
        Because $[x_1, 1 - x^*]$ is monotone, there exists continuous and increasing $\sigma: [x_1, 1 - x^*] \rightarrow [1 - x^*, x^*]$ such that $f(x) = \sigma(x)$. Since $h^0$ is the identity map, it trivially also holds that $f(x) = \phi(h^0(\phi^{-1}(\sigma(x))))$. Because $f(x_1) = 1 - x^*$ and $f(x^*) = x^*$, it follows that $f$ is $\lv_0$ on $[x_1, 1- x^*]$.
        
        By a similar argument, $f$ is $\rv_0$ on $[x^*, 1-x_1]$, with the only difference being that $\sigma$ needs to be strictly decreasing for it to hold.
        \item $[1-x^*, x^*]$ is $\lp_1$ because $[1-x^*, x^*]$ is $\lv_0$ on the identity map $f^0$. This trivially holds using the identity $\sigma$ map.
    \end{itemize}
    
    Now, we prove the inductive step, which can be summed up by the following line:
    \[\app \rightarrow \app \lv_0; \dep \rightarrow \rv_0\dep; \lv_j \rightarrow \lp_{j+1}; \lp_j \rightarrow \lv_j; \rv_j \rightarrow \rp_{j+1}; \rp_j \rightarrow \rv_j.\]
    We show each part of the relationship as follows:
    \begin{itemize}
        \item If $f^k$ is $\app$ on $[0, b]$, then there exists some $c \in (0, b)$ such that $f^{k+1}(c) = 1 - x^*$ because $f^k$ is an isomorphism between $[0, b]$ and $[0, x^*]$.
        
        It follows that $f^{k+1}$ is $\app$ on $[0, c]$ because $f^{k+1}$ is strictly increasing on the interval from $0$ to $1 - x^*$.
        
        $[c, b]$ is $\lv_0$ because there must exist some increasing $\sigma$ such that $f^{k+1}(x) = \sigma(x)$ on that interval. Thus, it follows that $f^{k+1} = \phi\circ h^0 \circ \phi^{-1}\circ \sigma$ on $[0,b]$.
        
        \item The same argument holds for $\dep$. If $f^k$ is $\dep$ on $[a, 1]$, then there exists $c \in (a, 1)$ such that $[a, c]$ is $\rv_0$ and $[c, 1]$ is $\dep$.
        
        \item If $f^k$ is $\lv_j$ on $[a, b]$, then $f^{k+1}$ is $\lp_{j+1}$ on the same interval by the definition of $\lp_{j+1}$.
        
        \item Similarly, if $f^k$ is $\rv_j$ on $[a, b]$, then $f^{k+1}$ is $\rp_{j+1}$ on the same interval by the definition of $\rp_{j+1}$.
        
        \item If $f^k$ is $\lp_j$ on $[a, b]$, then $f^{k-1}$ is $\lv_{j-1}$ and hence $f^{k-1}$ maps to $[f(\xmax), x^*]$ on the interval. Therefore, there exists $\sigma$ such that $f^{k-1} = \phi \circ h^{j-1} \circ \phi^{-1} \circ\sigma$ on the interval. We use the properties of $h$ to show that $f^{k+1}$ is $\lv_{j}$ on $[a, b]$. Note that $f^2 = \phi \circ h \circ \phi^{-1}$ on $[f(\xmax), x^*]$.
        \begin{align*}
            f^{k+1}
            &= f^2 \circ f^{k-1}
            = \phi \circ h\circ \phi^{-1} \circ \phi \circ h^{j-1}\circ\phi^{-1}\circ\sigma
            = \phi\circ h^{j}\circ\phi^{-1}\circ\sigma
        \end{align*}
        Thus, $f^{k+1}$ satisfies the condition to be $\lv_j$.
        \item By a identical argument, if $f^k$ is $\rp_j$ on $[a, b]$, then $f^{k+1}$ is $\rv_{j}$.
    \end{itemize}
    The remainder of this argument follows by applying the above transition rules for each piece to the inductive hypothesis about the ordering of pieces in $f^k$ to obtain the ordering for $f^{k+1}$.
\end{proof}

Now, we determine how many local maxima and minima are contained in each type of piece. Let $\maxima(f)$ and $\minima(f)$ represent the number of local maxima and minima respectively on mapping $f$ on interval $[0,1]$. We bound the total number of monotone pieces with these bounds by using $\oscm(f) = 2 \maxima(f)$. We similarly abuse notation to bound the number of maxima and minima in a category with $\maxima(\category{q,z}^k)$ and $\minima(\category{q,z}^k)$, and in the interval $[a,b]$ with $\maxima(f, a,b)$ and $\minima(f, a,b)$. 

By the base case in the previous section $\maxima(\category{0,0}^k) = 1$, $\minima(\category{0,0}^k) = 2$, $\maxima(\category{0,1}^k) = k,$ and $\minima(\category{0,1}^k) = k+1$. We obtain recurrences to represent $\maxima(\category{q,z}^k)$ and $\minima(\category{q,z}^k)$.

For each part, we rely on the following facts: If $\sigma$ is a strictly increasing bijection, then $\maxima(f \circ \sigma, a, b) = \maxima(f,a ,b)$. If $\sigma$ is strictly decreasing, then $\minima(f \circ \sigma, a, b) = \maxima(f,a ,b)$. (The reverse are true for minima of $f$.)

We analyze each type of piece individually, considering what happens when $f$ has some kind of piece on interval $[a, b]$.

\begin{itemize}
    \item Because $\app$ and $\dep$ segments are strictly increasing or decreasing, $\maxima(f, a, b) = 0$ when $f$ has either piece on $[a,b]$. $\minima(f, a,b) = 1$ because segments that support $\app$ contain $0$ and segments with $\dep$ have $1$, each of which $f$ maps to 0.
    \item Because each $\lv_i$ segment of $f^k$ on $[a,b]$ can be represented as $\phi \circ h^{i} \circ \phi^{-1} \circ \sigma$, and because $\phi$ is strictly decreasing, $\maxima(f_r^k, a, b) = \minima(h^i)$ and $\minima(f_r^k, a, b) = \maxima(h^i)$.
    By Lemma~\ref{lemma:doubling-h}, $h \in \category{q-1, 1-z}$, $\maxima(f_r^k, a, b) \leq \minima(\category{q-1, 1-z}^i)$ and $\minima(f_r^k, a, b) \leq \maxima(\category{q-1, 1-z}^i)$.
    
    The same analysis holds for each $\rv_i$ segment. 
    
    \item 
    Consider an $\lp_i$ segment of $f^k$ on $[a, b]$, which has output spanning the interval $[x^*, \xmax]$. Because $x^* > \frac{1}{2}$, $f$ is strictly decreasing on the domain $[x^*, \xmax]$. Thus, $f^{k+1}$ must satisfy $\maxima(f_r^{k+1},a, b) = \minima(f_r^{k},a, b)$ and $\minima(f_r^{k+1},a, b) = \maxima(f_r^{k},a, b)$.
    
    Note by the definition of $\lp_i$ that $[a,b]$ must also support an $\lv_{i-1}$ segment on $f^{k-1}$ and an $\lv_{i}$ segment on $f^{k+1}$. 
    From the previous bullet, the $\lv_i$ segment must have at most $\minima(\category{q-1, 1-z}^i)$ maxima and $\maxima(\category{q-1, 1-z}^i)$ minima. 
    Because there must be a one-to-one correspondence between minima of $f^{k+1}$ and maxima of $f^{k}$ on the interval and vice versa, the $\lp_i$ segment has $\maxima(f_r^k, a, b) \leq \maxima(\category{q-1, 1-z}^i)$ and $\minima(f_r^k, a, b) \leq \minima(\category{q-1, 1-z}^i)$.
    
    The same analysis hold for each $\rp_i$ segment.
\end{itemize}

Therefore, we can construct a recurrence relationship for the number of maxima and minima for $f_r^k$ based on the sequences found in Lemma \ref{lemma-piece-decomposition}.

\begin{align*}
   \maxima(\category{q,z}^k)
    &\leq \underbrace{\sum_{i=0}^{\floor{k/2}} \minima(\category{q-1, 1-z}^k)}_{\lv_i} 
    + \underbrace{\sum_{i=0}^{\floor{(k-1)/2}} \minima(\category{q-1, 1-z}^k)}_{\rv_i} 
    \\&\quad+ \underbrace{\sum_{i=1}^{\floor{(k+1)/2}} \maxima(\category{q-1, 1-z}^k)}_{\lp_i} 
    + \underbrace{\sum_{i=1}^{\floor{k/2}} \maxima(\category{q-1, 1-z}^k)}_{\rp_i} \\
    \minima(\category{q,z}^k)
    &= \underbrace{2}_{\app \& \dep} 
    + \underbrace{\sum_{i=0}^{\floor{k/2}} \maxima(\category{q-1, 1-z}^k)}_{\lv_i} 
    + \underbrace{\sum_{i=0}^{\floor{(k-1)/2}} \maxima(\category{q-1, 1-z}^k)}_{\rv_i} 
    \\&\quad+ \underbrace{\sum_{i=1}^{\floor{(k+1)/2}} \minima(\category{q-1, 1-z}^k)}_{\lp_i} 
    + \underbrace{\sum_{i=1}^{\floor{k/2}} \minima(\category{q-1, 1-z}^k)}_{\rp_i} 
\end{align*}

We bound $\maxima(\category{q,z}^k)$ and $\minima(\category{q,z}^k)$ by induction to prove Proposition~\ref{prop:osc-category-bound}.
We use the following inductive assumption over all $k$, $q$, and $z$, which suffices to prove the claim:
\[\maxima(\category{q, z}^k), \minima(\category{q, z}^k) \leq \begin{cases}
   (4q)^k & \text{$q$ is even, $z = 0$, or $q$ is odd, $z = 1$}\\
   (4q)^{k+1} & \text{$q$ is even, $z = 1$, or $q$ is odd, $z = 0$}.
\end{cases}\]

By the previous section, the claim holds for $q = 0$ and all $k$ and $z$, which gives the base case.

Moving forward, we assume that the claim holds for all values of $q'$ with $q' \leq q$ and any $k$ and $z$. 
We prove that it holds for $q+1$ with any choices of $k$ and $z$.

We show that the bound holds for $\minima(\category{q+1, z}^k)$ when $q+1$ is even and $z = 1$, or $q+1$ is odd and $z = 0$. The other cases are nearly identical. Since the bounds are trivial for $k = 1$, we prove them below for $k \geq 2$.

\begin{align*}
    \minima(\category{q+1, z}^k)
    &\leq 2 + \sum_{i=0}^{\floor{k/2}} (4i)^{q+1} + \sum_{i=0}^{\floor{(k-1)/2}} (4i)^{q+1} + \sum_{i=0}^{\floor{(k+1)/2}} (4i)^{q+1} + \sum_{i=0}^{\floor{k/2}} (4i)^{q+1} \\
    &\leq 4 \cdot \frac{k}{2} \cdot (2k)^{q+1} + (2(k+1))^{q+1} \leq (2k)^{q+2} + (3k)^{q+1} \leq (4k)^{q+2}.
\end{align*}

\subsection{Proof of Theorem~\ref{thm:properties-doubling}, Claim 2}\label{assec:doubling-nn}

We restate the claim:
\begin{proposition}[Claim 2 of Theorem~\ref{thm:properties-doubling}]\label{prop:properties-doubling-approx}
    Suppose $f$ is a symmetric unimodal mapping whose maximal cycle is of length $p = 2^q$. 
    For any $k \in \N$, there exists $g\in\mathcal{N}(u, 2)$ with width $u = O((4k)^{q+1}/\eps)$ such that $L_\infty{f^k,g} \leq \epsilon$. Moreover, if $f = \ftent{r}$, then there exists $g$ of width $O((4k)^{q+1})$ with $g = f^k$.
\end{proposition}

\begin{proof}
    
This part follows the bound on monotone pieces of $f^k$ given in Proposition~\ref{prop:properties-doubling-osc} and a simple neural network approximation bound.

\begin{lemma}
    Consider some continuous $f: [0,1] \to [0,1]$ with $\oscm(f) \leq m$.
    For any $\eps \in (0, 1)$, there exists $g \in \mathcal{N}(u,2)$ of width $u =O(\frac{m}{\eps})$ such that $L_\infty(f, g) \leq \eps.$
\end{lemma}

\begin{proof}
    A monotone function mapping to $[0,1]$ can be $\eps$-approximated by a piecewise-linear function with $O(\frac{1}{\eps})$ pieces, and hence, a 2-layer ReLU network of width $O(\frac{1}{\eps})$.
    
    Every monotone piece can be approximated as such, which means that $g$ has width $O(\frac{m}{\eps})$.
\end{proof}

For the case where $f = \ftent{r}$ for some $r$, it is always true that $\abs{\frac{d}{dx} \ftent{r}^k(x)} = (2r)^k$, except when $x$ is a local maximum or minimum. 
Thus, every monotone piece of $f^k$ is linear, and $f$ can be exactly expressed with a piecewise linear function with $O((4q)^{k+1})$ pieces, and also a ReLU neural network of width $((4q)^{k+1})$.
\end{proof}

\subsection{Proof of Theorem~\ref{thm:properties-doubling}, Claim 4}\label{assec:doubling-vc}

Recall that for unimodal $f: [0,1]\to[0,1]$ and threshold $t \in (0,1)$, \[\hyp{f,t} := \{\thres{t}{ f^k}: k \in \mathbb{N}\}\]
is the hypothesis class under consideration.

\begin{proposition}[Claim 4 of Theorem~\ref{thm:properties-doubling}]\label{prop:properties-doubling-vc}
    Suppose $f$ is a symmetric unimodal mapping whose maximal cycle is of length $p = 2^q$. 
    For any $t \in (0, 1)$, $\vc(\hyp{f, t}) \leq 18p^2$.
\end{proposition}

This proof is involved and requires some setup and new definitions. 

\subsubsection{Notation}
Let $\boolseq$ represent all countable infinite sequences of Boolean values, and let $\boolstr$ represent all finite sequences (including the empty sequence).

For $y \in \boolseq$, let $y_{i:j} = (y_i, \dots, y_{j}) \in \bit^{j-i+1}$ and $y_{i:} = (y_i, y_{i+1}, \dots) \in \boolseq$.
For $w \in \bit^n, w' \in \bit^{n'}$, let $ww' = w \circ w' \in \bit^{n+n'}$ be their concatenation.
Let $w^j = w \circ w \circ \dots \circ w \in \bit^{jn}$.

\subsubsection{Iterated Boolean-Valued Functions, Regular Expressions, and VC-Dimension}
Before we give the main result, we give a way to upper-bound the VC-dimension of countably infinite hypothesis classes $\hc = \{h_1, h_2, \dots, \} \subseteq ([0,1] \to \{0, 1\})$. 
For some $x \in \mathcal{X}$, define $\seq{\hc}: [0,1] \to \boolseq$ as $\seq{\hc}(x) = (h_i(x))_{i \in \N}$. We by $\hc$ over all choices of $x \in [0,1]$:
\[\Seq{\hc} = \{\seq{\hc}(x): \ x \in [0,1]\} \subset \boolseq.\]

With this notation, $\hc$ shatters $d$ points if and only if there exist $\ys1,\dots,\ys{d} \in \Seq{\hc}$ such that $\absl{\{(\ys1_j, \dots, \ys{n}_j): j \in \N\}} = 2^d$.
We equivalently say that $\ys1,\dots,\ys{d}$ are shattered.

Here's where the idea of Regular Expressions (Regexes) comes in. If we can show all elements in $\Seq{\hc}$ are represented by some infinite-length Regex, then we can upper-bound the number of points $\hc$ can shatter, which is necessary to bound the expressive capacity of unimodal functions with recursive properties. 

To that end, we first introduce a different notion of shattering. Then, we'll give an upper-bound for the VC-dimension of $\hc$ when we have a Regex for $\Seq{\mathcal{H}}$. 

\begin{definition}
    We say that $\hc$ (or $\Seq{\hc}$) \textbf{weakly shatters} $d$ points if there exist $\ws{1}\circ \ys1, \dots , \ws{d}\circ \ys{d} \in \Seq{\mathcal{H}}$ for $\ws1,\dots, \ws{d} \in \boolstr$ such  that $\ys1,\dots, \ys{d}$ are shattered.
    Let the \textbf{weak VC-dimension} of $\hc$ represent the maximum number of points $\hc$ can weakly shatter and denote it $\vcw(\hc) = \vcw(\Seq{\hc})$.
\end{definition}

Using this notation, we can extend our notion of weak VC-dimension to any subset of $\boolseq$, whether or not it corresponds to a hypothesis class. If $\hc \subset S \subset \boolseq$, then $\vcw(\hc) \leq \vcw(S)$.

Note that if $\hc$ shatters $d$ points, then it also trivially weakly shatters $d$ points. We can get this by taking $w_1 = \dots, w_d$ to be the empty strings. Thus, the $\vc(\hc) \leq \vcw(\hc)$.

A Regex is a recursively defined subset of $\boolseq$ that can be represented by a string. We describe how a Regex $R \subseteq \boolseq$ can be defined below.

\begin{itemize}
    \item One way to define a Regex is with a repeating sequence $w^{\infty}$ for $w \in \bit^n$. That is, \[w^{\infty} = \{y \in \boolseq: y_{in+1: (i+1)n} = w, \forall i \in \N\}.\] For instance, $(011)^{\infty} = \{(0,1,1,0,1,1,0,1,1,\dots)\}$.
    \item For $w \in \bit^n$, if $R$ is a Regex, then $w R$ is also a Regex. This means satisfying sequences must start with $w$ and then the remainder of the bits must satisfy $R$. 
    \[wR = \{y \in \boolseq: y_{1:n} = w, y_{n+1:} \in R\}.\]
    \item $w^* R$ is also a Regex, where $w^*$ represents any number of recurrences of the finite sequence $s$. That is,
    \[w^* R =\cup_{j=0}^\infty w^j R.\]
    \item If $R'$ is also a Regex, then so is $R \cup R'$.
    \item If $R'$ is also a Regex, then so is $R \oplus R'$, where the odd entries of sequences in $R \oplus R'$ concatenated together must be in $R$ and the even entries must be in $R'$.
    \[R \oplus R' = \{y \in \boolseq: y_{1,3,5, \dots} \in R, y_{2,4,6,\dots} \in R'\}.\]
\end{itemize}

Now, we can create a recursive upper-bound on the number of points $\hc$ can weakly shatter. 
To do so, we assume that $\hc \subseteq R$ for some Regex $R$ and bound the weak VC dimension of $R$.

\begin{lemma}\label{lemma-regex}
 Consider infinite-length Regexes $R, R', R''$ and $w \in \bit^n$.
 \begin{enumerate}
     \item If $R = w^{\infty}$, then $\vcw(R) \leq \log_2 n$.
     \item If $R = wR'$, then $\vcw(R) \leq \vcw(R') + \log_2 n + 1$.
     \item If $R = w^* R'$, then $\vcw(R) \leq \vcw(R') + \log_2 n + 1$.
     \item If $R = R' \cup R''$, then $\vcw(R) \leq \vcw(R') + \vcw(R'')$.
     \item If $R = R' \oplus R''$, then $\vcw(R) \leq 4 \max(\vcw(R'), \vcw(R''))+2$.
 \end{enumerate}
\end{lemma}
\begin{proof}
\begin{enumerate}
    \item If $R = w^{\infty}$, then the set $Y = \{y: w\circ y \in w^\infty, w \in \boolstr\}$ contains at most $n$ elements. 
    Hence, 
    \[\abs{\{(\ys1_j,\dots,\ys{d}_j): j\in\N\}} \leq n\]
    for any fixed $\ys1,\dots,\ys{d}\in Y$, and
    no more than $d= \log_2 n$ points can be weakly shattered.
    
    \item Suppose $R$ weakly shatters $d$ points, so $\ys1,\dots, \ys{d}$ are shattered for some $\ws1\circ\ys1,\dots,\ws{d}\circ \ys{d} \in R$.
    If $Y = \{(\ys1_j, \dots, \ys{d}_j): j \in \N\}$ and $Y_n = \{(\ys1_j, \dots, \ys{d}_j): j \leq n\}$, then $\absl{Y} = 2^d$ and $\absl{Y_n} \leq n$.
    There exists some $v \in \bit^{1+ \log_2 n}$ such that $v \circ \sigma \in Y \setminus Y_n$ for all $\sigma \in \bit^{d -1- \log_2 n}$. 
    Therefore, \[\absl{\{(\ys{2+\log_2 n}_j, \dots, \ys{d}_j): j > n\}} = 2^{d - 1 - \log_2n},\] and there exist $d -1- \log_2 n$ points that can be weakly shattered by $R'$, since none of the labelings with $w$ are necessary.
    
    
    \item Once again, suppose $R$ weakly shatters $d$ points, $\ys1,\dots, \ys{d}$ for $\ws1\circ\ys1,\dots,\ws{d}\circ \ys{d} \in R$.
    Because each $\ws{i} \circ \ys{i} \in w^*R'$, there exists an index $\ell_i$ such that $(\ws{i} \circ \ys{i})_{1:\ell_i} = w^{\ell_i / n}$ and $(\ws{i} \circ \ys{i})_{\ell_i+1:} \in R'$. 
    Without loss of generality, assume $\ys1,\dots, \ys{d}$ are ordered such that $\ell_i - \abs{\ws{i}}$ decreases.
    That is, the first $1+\log_2 n$ sequences are the ones that ``leave $w^*$ last.''
    Let $\ell^* := \ell_{1+\log_2 n} - \abs{\ws{1+\log_2 n}}$.
    Define $Y$ and $Y_{\ell^*}$ analogously to the previous part and note that $\absl{Y} = 2^d$. Because $Y_{\ell^*}$ corresponds only to labelings where the first $1 + \log_2 n$ elements come from subsets of $w^\infty$, there exists some $v \in \bit^{1+ \log_2 n}$ such that $v \circ \sigma \in Y \setminus Y_{\ell^*}$ for all $\sigma \in \bit^{d - 1  -\log_2 n}$.
    As before, there exist $d -1- \log_2 n$ points that can be weakly shattered by $R'$
    
    \item There is no set of $\vcw(R')+1$ and $\vcw(R'') +1$ points that can be weakly shattered by $R'$ and $R''$ respectively. 
    Any $\vcw(R')+\vcw(R'')+1$ points in $R$ must have at either $\vcw(R')+1$ points in $R'$ or $\vcw(R'')+1$ points in $R''$.
    Thus, at least one subset cannot be shattered.
    
    
    \item 
    Suppose without loss of generality that $d := \vcw(R') \geq \vcw(R'')$.
    Consider any $\ws{1}\circ\ys1,\dots, \ws{d}\circ\ys{4d+3} \in R$.
    WLOG, assume that $\abs{\ws{1}}, \dots, \abs{\ws{2d+2}}$ are even, which implies that $\ws{1}_{\text{odd}}\circ \ys{1}_{\text{odd}}, \dots \ws{2d+2}_{\text{odd}}\circ \ys{2d+2}_{\text{odd}} \in R'$ and $\ws{1}_{\text{even}}\circ \ys{1}_{\text{even}}, \dots \ws{2d+2}_{\text{even}}\circ \ys{2d+2}_{\text{even}} \in R''$. 
    Therefore,
    \begin{align*}
        \abs{\{(\ys1_j, \dots, \ys{4d+3}_j): j \in \N\}} 
        &\leq 2^{2d+1}\abs{\{(\ys1_j, \dots, \ys{2d+2}_j): j \in \N\}} \\
        &\leq 2^{2d+1}\paren{\abs{\{(\ys1_j, \dots, \ys{2d+2}_j): j \in \N_{\text{odd}}\}} + \abs{\{(\ys1_j, \dots, \ys{2d+1}_j): j \in \N_{\text{even}}\}}} \\
        &\leq 2^{2d+1}\cdot 2 \sum_{i=0}^{d} {2d+2 \choose i} 
        < 2^{2d+2}\cdot 2^{2d+1} 
        = 2^{4d+3}.
    \end{align*}
    The last line follows by the Sauer Lemma. 
    Thus, $R$ cannot shatter $4d+3$ points if $R'$ and $R''$ cannot shatter $d$ points.
    
\end{enumerate}
\end{proof}

Here's an example of how to apply our regex rules:
\begin{align*}
    \vcw(1^* 0(01)^{\infty} \cup 10^{\infty})
    &\leq \vcw(1^* 0(01)^{\infty}) + \vcw(10^{\infty}) \\
    &\leq 1 + \vcw(0(01)^{\infty}) + 1 + \vcw(0^\infty) \\
    &\leq 2 + 1 + \vcw((01)^\infty) \\
    &\leq 3 + 1 = 4.
\end{align*}

\subsubsection{Proof of the Proposition~\ref{prop:properties-doubling-vc}}
Recall that we consider the hypothesis class \[\hc_{f,t} := \{\thres{t}{  f^k}: k \in \N\}\]
for symmetric unimodal $f$ and $t \in (0,1)$.

To build up the argument, we first bound the VC-dimension for two simple cases. 
\begin{itemize}
    \item 
    First, we consider the case when $f$ has no fixed point. 
    Thus, for all $x \in (0, 1]$, $f(x) < x$, which means that the sequence $f(x), f^2(x), \dots$ is decreasing. 

    If the threshold $t$ is $0$ or is greater than $f(\frac{1}{2})$, then the sequence will be all 0's or 1's, which will imply that $\vc(\hc_{f, t}) = 0$. Thus, the only interesting thresholds are $t \in (0, f(\frac{1}{2})]$. Because the sequence is decreasing, $\Seq{\hc_{f,t}} = 1^* 0^{\infty}$. From Lemma \ref{lemma-regex}, $\vc(\hc_{f,t}) \leq \vcw(\hc_{f, t})  \leq 1$.
    
    \item 
    Let $x_1 < \dots < x_m$ be all the fixed points of $f$. Suppose $x_m \leq \frac{1}{2}$. By symmetry, for all $j \in [m]$, $f(1-x_j) = x_j$.
    
    To analyze this function, we partition $[0, 1]$ into $2m + 2$ intervals: $I_0 = [0, x_1), I_0' = (1-x_1, 1]$, $I_m = [x_m, \frac12]$, $I_m' = (\frac12, 1-x_m]$, $I_j = [x_j, x_{j+1})$, and $I_j' = (1-x_{j+1},1- x_j]$ for all $j \in \{1, \dots, m-1\}$ (visualized in Figure \ref{fig:simple-fixed-point}).
    
    \begin{figure}
        \centering
        \input{fig/fig-simple-fixed-point}
        \caption{A plot of $f$ with fixed point $m=2$ fixed points---both less than $\frac{1}{2}$---subdivided into intervals. The relationships of which intervals $f$ maps onto one another are also visualized.}
        \label{fig:simple-fixed-point}
    \end{figure}
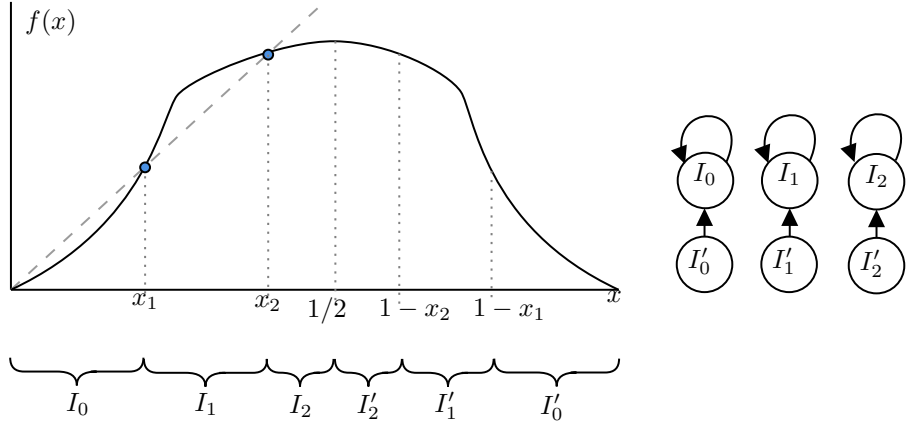
    Because $f$ is unimodal and because the edges of all intervals map to fixed points, for all $j \in \{0, \dots, m\}$, $f(I_j') = f(I_j) = I_j$. In this case, it must be the case that $q = 0$ because $f$ cannot have a 2-cycle. Such a cycle is impossible because it would have to be contained entirely in some $I_j$. In those intervals, it must be the case that either $\forall x \in I_j$, $f(x) \geq x$, or $\forall x \in I_j$, $f(x) \leq x$ (if this were not the case, then this would imply the existence of a fixed point other than $x_j$ in $I_j$). Thus, cyclic behavior within an interval is impossible.
    
    Thus, we can construct a Regex to represent the itinerary of any $x \in [0, 1]$: $ \bigcup_{j=0}^{m} I_j^{\infty}$.\footnote{
    This is a massive abuse of notation, but we use the same Regex notation to denote the intervals that are traversed as we use to denote the values of Boolean sequence.} 
    Now, we consider all possible locations of threshold $t$:
    \begin{itemize}
        \item If $t \in I_j$, such that $f(x) \geq x$ for $x \in I_j$, then $\Seq{\hc_{f,t}} \subseteq 0^* 1^{\infty} \cup 0^{\infty} \cup 1^{\infty}$. By Lemma \ref{lemma-regex}, $\vcw(\hc_{f,t}) \leq 1$.
        \item If $t \in I_j$, such that $f(x) \leq x$ for $x \in I_j$, then $\Seq{\hc_{f,t}} \subseteq 1^* 0^{\infty} \cup 0^{\infty} \cup 1^{\infty}$. By Lemma~\ref{lemma-regex}, $\vcw(\hc_{f,t}) \leq 1$.
        \item If $t \in \bigcup_{j=0}^m I_j'$, then $\Seq{\hc_{f,t}} = 0^{\infty}$, and $\vcw(\hc_{f,t}) = 0$.
    \end{itemize}

\end{itemize}

Now, we give a lemma, which relates the VC-dimension of complex functions to that of simpler ones.
Let $\category{q}$ refer to the family of symmetric unimodal functions that have a $2^q$-cycle but not a $2^{q+1}$-cycle.

\begin{lemma}\label{lemma:vc-recursive}
    For any $f \in \category{q}$ with fixed point $x^* > \frac12$ and any $t \in [0,1]$, \[\vcw(\hc_{f,t}) \leq 4 \max_{g \in \category{q-1}, t'\in [0,1]}\vcw(\hc_{g, t'})+10.\]
\end{lemma}
\begin{proof}
    
Consider some such $f$. Let $x_1< \dots< x_m$ be the fixed points of $f$ where $ x_m=x^* > \frac12$. Because $\frac12$ maximizes $f$, $f(\frac12) \geq x_m > \frac12$. This fixed point must the only fixed point no smaller than $\frac12$; the existence of another such fixed point would contradict the fact that $f$ is decreasing on $(\frac12, 1]$. Thus, $x_1, \dots, x_{m-1} < \frac12$.

We build a recursive relationship by considering $f^2$ and relating some its output on some segments of $[0, 1]$ to other maps with smaller $q$. For now, we instead attempt to upper-bound the VC-dimension of $\mathcal{H}_{f^2, t}$.

For all $j \in [m]$, unimodality implies that $x_j$ and $1-x_j$ are the only points that map to $x_j$ and that the following ordering holds.
\[0 < x_1 < \dots, < x_{m-1} < 1-x_m < \frac12 < x_m < 1-x_{m-1} < \dots, 1-x_1 < 1.\]

By the Intermediate Value Theorem, there exists some $x_m' \in (x_m, 1-x_{m-1})$ such that $f(x_m') = f(1-x_m') = 1-x_m$ and $f^2(x_m') =f^2(1-x_m')= x_m$.

We define intervals as follows:
\begin{itemize}
    \item $I_0 = [0, x_1)$ and $I_0' = (1-x_1, 1]$.
    \item For all $j \in [m-2]$, $I_j = [x_j, x_{j+1})$ and $I_{j}' = (1-x_{j+1}, 1-x_j]$.
    \item $I_{m-1} = [x_{m-1}, 1-x_{m}')$ and $I_{m-1}' = (x_{m}', 1-x_{m-1}]$.
    \item $I_{m} = [1-x_m', 1-x_m)$ and $I_m' = (x_m, x_m']$.
    \item $I_{m+1} = [1-x_m, \frac12)$, and $I_{m+1}' = [\frac12, x_{m}]$.
\end{itemize}

For any $j \in \{0, \dots, m+1\}$, $f$ is increasing on all intervals $I_j$ and decreasing on $I_j'$. 
By symmetry, $f(I_j) = f(I_j')$.
For all $j \in \{0, \dots, m-2\}$, $f(I_j) = I_j$. $f(I_{m-1}) = I_{m-1} \cup I_m$, $f(I_m) = I_{m+1} \cup I_{m+1}'$, and $f(I_{m+1}) \subseteq I_{m}'$, because $f(\frac12) \in [x_m, x_m')$.\footnote{
This must be the case for the assumptions to be met. If $f(\frac12) < x_m$, then $x_m$ cannot be a fixed point because $\frac12$ maximizes $f$. If $f(\frac12) > x_m'$, then there exists a 3-cycle with points in $I_{m+1}, I_{m-1}', I_m$, which contradicts the assumption that we only have power-of-two cycles.}

From there, we obtain additional properties for $f^2$: $f^2(I_{m-1}) = I_{m-1} \cup I_{m} \cup I_{m+1} \cup I_{m+1}'$, $f^2(I_m) \subseteq I_{m}'$, and $f^2(I_{m+1}) \subset I_{m+1} \cup I_{m+1}'$. This suggests that there is recurrent structure that we can take advantage of to count all of the patterns.

Let $J_{m+1} := I_{m+1} \cup I_{m+1}'$. We create a Regex to track the behavior of iterates $f^2$, which we visualize in Figure~\ref{fig:inductive-fixed-point}:
\[\bigcup_{j=0}^{m-2} I_j^{\infty}\ \cup\  I_{m-1}^* I_m I_m^{\prime \infty}\ \cup\ I_m^{\prime \infty} \ \cup\  I_{m-1}^* J_{m+1}^{\infty}.\]

When an iterate of $f^2$ gets ``stuck'' in one of $I_0, I_1, \dots, I_{m-1}$, it must either be at a fixed point, be strictly increasing, or be strictly decreasing. To suggest otherwise would imply the existence of another fixed point in those intervals, because $f^2$ is monotonically increasing or decreasing in all of those and either all $x$ yield $f^2(x) \geq x$ or $f^2(x) \leq x$.

For the remaining intervals, one might notice in Figure \ref{fig:inductive-fixed-point} that zooming in on the intervals $I_{m}$, $J_{m+1}$, and $I_{m}'$ for $f^2$ gives what looks like unimodal maps.\footnote{We use similar techniques here to those used in Section \ref{sec:ub-general-case}.} We take advantage of that structure to bound the complexity of the 0/1 Regexes for those intervals. We can formalize this by defining symmetric unimodal mappings $h_m$ and $h_{m+1}$ and bijective monotonic mappings $\phi_m: I_{m}' \to (0,1]$ (increasing) and $\phi_{m+1}: J_{m+1} \to [0,1]$ (decreasing) such that:
\begin{itemize}
    \item For $x \in I_m$, $f^2(x) = \phi_m^{-1} \circ h_m \circ \phi_m(1-x).$
    \item For $x \in I_m'$, $f^2(x) = \phi_m^{-1} \circ h_m \circ \phi_m(x).$
    \item For $x \in J_{m+1}$, $f^2(x) = \phi_{m+1}^{-1} \circ h_{m+1} \circ \phi_{m+1}(x).$
\end{itemize}

\begin{figure}
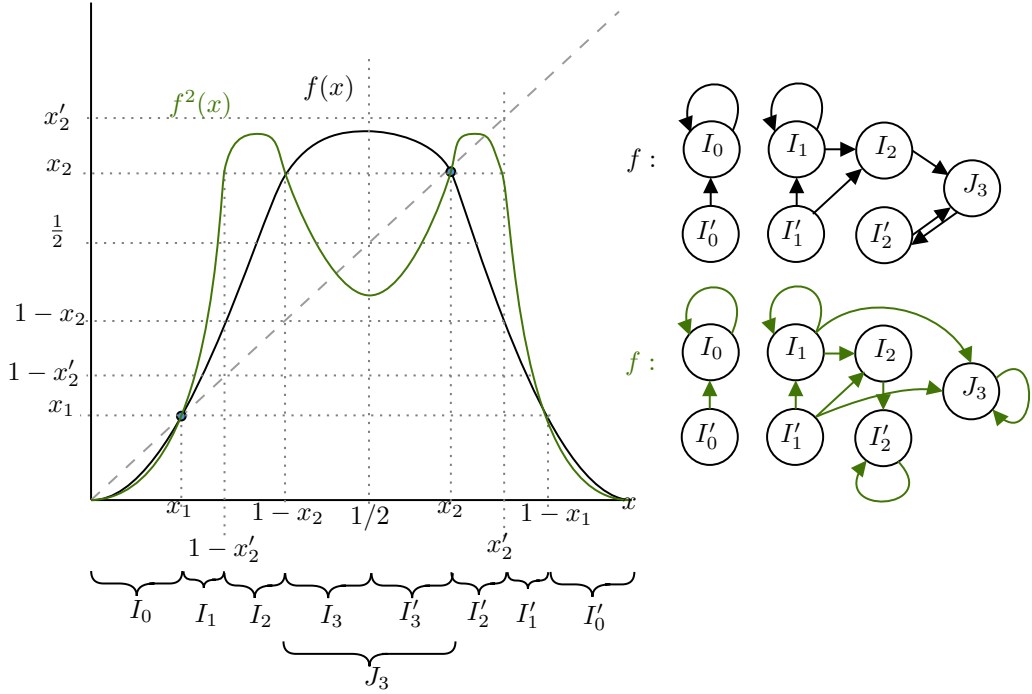

    \centering
    \include{fig/fig-inductive-fixed-point}
    \caption{Like Figure~\ref{fig:simple-fixed-point}, plot of $f$ and $f^2$ with $m = 3$ fixed points with $x_m > \frac{1}{2}$ and visualizes the mappings between intervals.}
    \label{fig:inductive-fixed-point}
\end{figure}

Because $f$ cannot have a cycle of length $2^{q+1}$, $h_m$ and $h_{m+1}$ may not have cycles of length $2^q$. 
Thus, we can reason inductively about how iterates behave when they're trapped in those intervals.

We do another case analysis of the 0/1 Regexes induced by different choices of $t$.
\begin{itemize}
    \item If $t \in I_j$ for $j \in \{0, \dots, m-1\}$, then $\Seq{\mathcal{H}_{f^2, t}} \subseteq 0^\infty \cup 1^{\infty} \cup 0^* 1^{\infty} \cup 1^* 0^{\infty}$ because a sequence of iterates only crosses $t$ if it enters the correct interval $I_j$, where the iterate then will be stuck and must monotonically increase or decrease. By Lemma \ref{lemma-regex}, $\vcw(\hc_{f^2, t}) \leq 2$.
    
    \item If $t \in I_j'$ for $j \in \{0, \dots, m-1\}$, then $\Seq{\mathcal{H}_{f^2, t}} = 0^{\infty}$, and $\vcw(\hc_{f^2, t}) = 0$.
    
    \item If $t \in I_{m}$, then $\Seq{\mathcal{H}_{f^2, t}} = 0^{\infty} \cup 1^{\infty} \cup 0^* 1^{\infty}$. Then, $\vcw(\hc_{f^2, t}) \leq 1$.
    
    \item If $t \in J_{m+1}$, then $\Seq{\hc_{f^2, t}} = 0^{\infty} \cup 0^* 1^{\infty} \cup 0^* J_{m+1}^{\infty}$. Because $h_{m+1}$ has at most a cycle of length $2^{q-1}$, we have that \[\vcw(\hc_{f^2, t}) \leq 2 + \max_{t'}\vcw(\hc_{h_{m+1}, t'}).\]
    
    \item If $t \in I_{m}'$, then $\Seq{{\hc}_{f^2, t}} = 0^{\infty} \cup 0^* I_m^{\prime \infty}$. 
    This gives us that \[\vcw(\hc_{f^2, t}) \leq 1 + \max_{t'} \vcw(\hc_{h_m, t'}).\]
\end{itemize}

To get $\vcw(\hc_{f, t})$, notice that $\Seq{\hc_{f, t}} = \Seq{\hc_{f^2, t}} \oplus \Seq{\hc_{f^2, t}'}$, where $\hc_{f^2, t}'$ refers to the outcome of all odd iterates of $f$. 
We show that $\Seq{\hc_{f^2, t}'} \subseteq \Seq{\hc_{f^2, t}}$ because the latter could induce all sequences produced by the former by starting with some $x'$ such that $f^2(x') = f(x)$. 
Thus, by Lemma \ref{lemma-regex}, 
\begin{align*}
    \vcw(\hc_{f, t}) 
    & \leq  4\max_{t'}\vcw(\hc_{f^2,t'}) + 2  \\
    & \leq 4 \max(2 + \max_{t'}\vcw(\hc_{h_{m+1}, t'}),1 + \max_{t'} \vcw(\hc_{h_m, t'})) + 2 \\
    &\leq 4\max_{g \in \category{q-1}, t'} \vcw(\hc_{g,t'}) + 10.\qedhere
\end{align*}
\end{proof}

Now, we prove a bound on the VC-dimension for arbitrary $q$ by induction with Lemma \ref{lemma:vc-recursive} to show that for $\vc(\mathcal{H}_{f, t}) \leq 18 \cdot 4^q$.

This holds when $q = 0$. There are two possible cases for the fixed point of such an $f$. If the the largest fixed point is smaller than $\frac12$, then, by the simple cases explored at the beginning, $\vcw(\mathcal{H}_{f, t}) \leq 1$. 
Otherwise, we apply Lemma \ref{lemma:vc-recursive} along with the the other simple case---which tells us what happens when there are no fixed point---to get that $\vcw(\hc_{f, t}) \leq 4 (1) + 10 = 14$. This trivially satisfies the proposition.

For the inductive step for arbitrary $q$, we iteratively apply Lemma~\ref{lemma:vc-recursive} to obtain the final bound.
\begin{align*}
    \vc(\mathcal{H}_{f, t})
    &\leq 4\max_{g \in \category{q-1}, t'} \vcw(\hc_{g,t'}) + 10 \\
    &\leq 4^q \max_{g \in \category{0}, t'} \vcw(\hc_{g,t'}) + 10 \sum_{i=0}^{q-1}4^i \\
    &\leq 14 \cdot 4^q + \frac{10}{3} 4^q \leq 18 \cdot 4^q.
\end{align*}

\subsection{Proof of Theorem~\ref{thm:properties-chaotic}, Claim 4}\label{asec:chaotic}

\begin{proposition}\label{proposition:odd-vc-dim}
    Suppose $f$ is a symmetric unimodal function with a $2^q m$-cycle for odd $m$. 
    Then for 
    \[K = \exp\paren{O\paren{q + d\log(d+m)}},\]
    $\vc(\hc_{f,K})\geq d$ for
    $\hc_{f,K} = \left\{\thres{1/2}{ f^k}: k \in [K]\right\}.$
\end{proposition}
The claim holds by this proposition, since the VC-dimension of $\hc_f$ is larger than every $d$ and hence must be infinite.
\begin{proof}
    The proof of this claim relies on the existence of a lemma that describes a characteristic of odd-period cycles of unimodal functions.
    
    \begin{lemma}\label{lemma:odd-cycle-left}
         Let $f$ be a symmetric unimodal function with some odd cycle $ x_1, x_2, \dots, x_{m}$ of length $m > 1$ such that $f(x_i) = x_{i+1}$ and $f(x_{m}) = x_1$. Then, there exists some $i$ such that $x_i < \frac{1}{2}$ and $f(x_i) \geq \frac{1}{2}$.
    \end{lemma}
    \begin{proof}
        To prove the claim, it suffices to show that the following two cases are impossible: (1) $x_1, \dots, x_{m} < \frac{1}{2}$ and (2) $x_1, \dots, x_{m} \geq \frac{1}{2}$.
        
        \begin{enumerate}
            \item Suppose $x_1, \dots, x_{m} < \frac{1}{2}$. 
            By unimodality $x_j < x_{j'}$ implies that $f(x_j) < f(x_{j'})$.
            If $x_1$ is the smallest element of the cycle, then $f(x_1) > x_1$.
            For any other $x_j$, $f(x_j) > x_1$, which means that $x_1$ cannot be part of a cycle, which contradicts the odd cycle.
            
            \item Suppose instead that $x_1, \dots, x_{m} \geq \frac{1}{2}$. 
            
            For this to be the case, $f(\frac12) > \frac12$ by unimodality. This fact paired with $f(1) < 1$ implies the existence of some $x^* \in (\frac12, 1)$ with $f(x^*) = x^*$. Because $f$ is decreasing on $[1/2, 1]$, $f([1/2, x^*)) \subseteq (x^*, 1]$ and $f((x^*, 1]) = [0, x^*)$. 
            
            If $x_1 \in [\frac12, x^*)$, then $x_2 \in (x^*, 1]$, and $x_3 \in [\frac12, x^*)$. 
            If apply this fact repeatedly, the oddness of $m$ implies that $x_{m} \in [\frac12, x^*)$ and $x_1 \in (x^*, 1]$, a contradiction.\qedhere
        \end{enumerate}
    \end{proof}
    
    We show that $\vc(\hc_{f^{2^q}, K/2^q}) > d$.
    If $f$ has a cycle of length $2^q \cdot m$, then $f^{2^q}$ has a cycle of length $m$. By Sharkovskii's Theorem, for all odd $m' > m$, $f^{2^q}$ also has a cycle of length $m'$. Let $p_1 < \dots < p_d$ be the smallest prime numbers greater than $m$. According to Lemma \ref{lemma:primes-in-interval}, $p_d \leq \left(\frac{K}{2^q}\right)^{1/d}$ for 
    \[K =2^q \paren{O(\max(d\log d, m)}^d = \exp\paren{O\paren{q + d\log(d+m)}}.\]
    
    For $j \in [m]$, let $x^{(j)}$ be the point guaranteed by Lemma~\ref{lemma:odd-cycle-left} with $f^{2^q \cdot p_j}(x^{(j)}) = x^{(j)}$, $x^{(j)} < \frac{1}{2}$, and $f^{2^q}(x^{(j)}) \geq \frac{1}{2}$. Therefore, it follows that $f^{2^q \cdot \ell p_j}(x^{(j)}) < \frac{1}{2}$ and $f^{2^q (\ell p_j + 1)}(x^{(j)}) \geq \frac{1}{2}$ for all $\ell \in \mathbb{Z}_{\geq 0}$.
    
    To show that $\hc_{f^{2^q}}$ shatters $x^{(1)}, \dots, x^{(d)}$, we show that for any labeling $\sigma \in \bit^d$, there exists $h \in \hc_{f^{2^q},K/2^q}$ such that $h(x^{(j)}) = \sigma_j$.
    \begin{itemize}
        \item If $\sigma = (0, \dots, 0)$, then consider $f^{2^q \cdot k}$, where $k = \prod_{j=1}^n p_j$. 
        Then, for all $j$, $f^{2^q \cdot k}(x^{(j)}) < \frac{1}{2}$. 
        Because $k \leq p_d^d \leq \frac{K}{2^q}$, there exists some $h \in \hc_{f^{2^q}, K/2^q}$ that assigns zero to every $x^{(j)}$.
        
        \item Similarly, if $\sigma = (1, \dots, 1)$, we instead consider $f^{2^q \cdot k}$ for $k = 1 + \prod_{j=1}^n p_j$. 
        Now, for all $j$, $f^{2^q \cdot k}(x^{(j)}) \geq \frac{1}{2}$, and $k \leq p_d^d \leq \frac{K}{2^q}$, which means there exists satisfactory $h \in \hc_{f^{2^q}, K/2^q}$.
        
        \item Otherwise, assume WLOG that $(\sigma_1, \dots, \sigma_{\ell}) = (0, \dots, 0)$ and $(\sigma_{\ell+1}, \dots, \sigma_d) = (1, \dots, 1)$ for $\ell \in (1, d)$. 
        We satisfy the claim for $f^{2^q \cdot k}$ if we choose some $k$ with $k = q_1 \prod_{i=1}^{\ell} p_i = 1 + q_2 \prod_{i=\ell+1}^d p_i$, for some $q_1, q_2 \in \mathbb{Z}_+$. 
        
        We find $q_1 \in [\prod_{i=\ell+1}^d p_i]$ and $q_2 \in [ \prod_{i=1}^{\ell} p_i ]$ by choosing them such that:
        \begin{align*}
            q_1 \prod_{i=1}^{\ell} p_i &\equiv 1 \pmod{\prod_{i=\ell+1}^d p_i} \\
            q_2  \prod_{i=\ell+1}^d p_i &\equiv -1 \pmod{\prod_{i=1}^{\ell} p_i} .
        \end{align*}
        
        This is possible because $p_1, \dots, p_d$ are prime, and $\gcd\left(\prod_{i=1}^{\ell} p_i, \prod_{i=\ell+1}^d p_i\right) = 1$.
        
        Because $k \leq \prod_{i=1}^d p_i \leq p_d^d \leq \frac{K}{2^q}$, there must exist some satisfactory $h \in \mathcal{H}_{f^{2^q}, K/2^q}$.\qedhere
    \end{itemize}
\end{proof}

\begin{lemma}\label{lemma:primes-in-interval}
    For $m \geq 3$ and any $d \geq 0$, there exist $d$ primes such that $m \leq p_1 < \dots < p_d$ for \[p_d = O(\max(d\log d, m)).\]
\end{lemma}
\begin{proof}
    Let $\pi(x) = \abs{\{y \in [x]: y \text{ is prime}\}}$ be the number of primes no larger than $x$.
    By the Prime Number Theorem, 
    \[\frac{x}{\log(x) + 2} \leq \pi(x) \leq \frac{x}{\log(x) - 4},\]
    for all $x \geq 55$ \citep{rosser41}.
    Thus, for some $m' = O(\max(d\log d, m))$,
    the number of prime numbers smaller than $m'$ is
    \[\Omega\paren{\frac{d\log d}{\log(d\log d)} + \frac{m}{\log m}} = \Omega\paren{d +\frac{m}{\log m}}, \]
    and the number between $m$ and $m'$ is $\Omega(d).$
    Thus, $p_d \leq m'$.
\end{proof}

%% file: fig/fig-simple-fixed-point.tex
\tikzset{every picture/.style={line width=0.75pt}} 

\begin{tikzpicture}[x=0.75pt,y=0.75pt,yscale=-1,xscale=1]

\draw    (80.1,10) -- (80.1,155.3) ;
\draw    (386.59,155.3) -- (80.1,155.3) ;
\draw [color={rgb, 255:red, 128; green, 128; blue, 128 }  ,draw opacity=1 ] [dash pattern={on 0.84pt off 2.51pt}]  (147.77,93.52) -- (147.77,159.74) ;
\draw   (79.73,189.58) .. controls (79.73,194.25) and (82.06,196.58) .. (86.73,196.58) -- (103.33,196.58) .. controls (110,196.58) and (113.33,198.91) .. (113.33,203.58) .. controls (113.33,198.91) and (116.66,196.58) .. (123.33,196.58)(120.33,196.58) -- (139.92,196.58) .. controls (144.59,196.58) and (146.92,194.25) .. (146.92,189.58) ;
\draw   (147.17,190.21) .. controls (147.17,194.88) and (149.5,197.21) .. (154.17,197.21) -- (168.21,197.21) .. controls (174.88,197.21) and (178.21,199.54) .. (178.21,204.21) .. controls (178.21,199.54) and (181.54,197.21) .. (188.21,197.21)(185.21,197.21) -- (202.26,197.21) .. controls (206.93,197.21) and (209.26,194.88) .. (209.26,190.21) ;
\draw   (416.37,100.14) .. controls (416.37,92.4) and (422.65,86.13) .. (430.38,86.13) .. controls (438.12,86.13) and (444.39,92.4) .. (444.39,100.14) .. controls (444.39,107.87) and (438.12,114.15) .. (430.38,114.15) .. controls (422.65,114.15) and (416.37,107.87) .. (416.37,100.14) -- cycle ;
\draw   (458.88,100.14) .. controls (458.88,92.4) and (465.15,86.13) .. (472.89,86.13) .. controls (480.62,86.13) and (486.9,92.4) .. (486.9,100.14) .. controls (486.9,107.87) and (480.62,114.15) .. (472.89,114.15) .. controls (465.15,114.15) and (458.88,107.87) .. (458.88,100.14) -- cycle ;
\draw   (458.56,142.48) .. controls (458.56,134.74) and (464.83,128.47) .. (472.57,128.47) .. controls (480.31,128.47) and (486.58,134.74) .. (486.58,142.48) .. controls (486.58,150.22) and (480.31,156.49) .. (472.57,156.49) .. controls (464.83,156.49) and (458.56,150.22) .. (458.56,142.48) -- cycle ;
\draw   (415.9,142.48) .. controls (415.9,134.74) and (422.17,128.47) .. (429.91,128.47) .. controls (437.65,128.47) and (443.92,134.74) .. (443.92,142.48) .. controls (443.92,150.22) and (437.65,156.49) .. (429.91,156.49) .. controls (422.17,156.49) and (415.9,150.22) .. (415.9,142.48) -- cycle ;
\draw   (502.64,100.76) .. controls (502.64,93.03) and (508.91,86.75) .. (516.65,86.75) .. controls (524.39,86.75) and (530.66,93.03) .. (530.66,100.76) .. controls (530.66,108.5) and (524.39,114.78) .. (516.65,114.78) .. controls (508.91,114.78) and (502.64,108.5) .. (502.64,100.76) -- cycle ;
\draw   (502.64,143.11) .. controls (502.64,135.37) and (508.91,129.1) .. (516.65,129.1) .. controls (524.39,129.1) and (530.66,135.37) .. (530.66,143.11) .. controls (530.66,150.85) and (524.39,157.12) .. (516.65,157.12) .. controls (508.91,157.12) and (502.64,150.85) .. (502.64,143.11) -- cycle ;
\draw    (429.91,128.47) -- (429.91,117.15) ;
\draw [shift={(429.91,114.15)}, rotate = 450] [fill={rgb, 255:red, 0; green, 0; blue, 0 }  ][line width=0.08]  [draw opacity=0] (8.93,-4.29) -- (0,0) -- (8.93,4.29) -- cycle    ;
\draw [color={rgb, 255:red, 155; green, 155; blue, 155 }  ,draw opacity=1 ] [dash pattern={on 4.5pt off 4.5pt}]  (80.1,155.3) -- (235,13) ;
\draw    (80.1,155.3) .. controls (149.28,122.9) and (156.36,66.7) .. (163.92,57.26) .. controls (171.48,47.81) and (216.34,29.39) .. (243.73,29.87) .. controls (271.12,30.34) and (300.4,45.45) .. (307.01,55.37) .. controls (313.63,65.29) and (317.88,123.37) .. (386.59,155.3) ;
\draw  [fill={rgb, 255:red, 74; green, 144; blue, 226 }  ,fill opacity=1 ] (145.29,93.52) .. controls (145.29,92.15) and (146.4,91.04) .. (147.77,91.04) .. controls (149.14,91.04) and (150.25,92.15) .. (150.25,93.52) .. controls (150.25,94.89) and (149.14,96) .. (147.77,96) .. controls (146.4,96) and (145.29,94.89) .. (145.29,93.52) -- cycle ;
\draw  [fill={rgb, 255:red, 74; green, 144; blue, 226 }  ,fill opacity=1 ] (207.37,36.62) .. controls (207.37,35.25) and (208.48,34.14) .. (209.85,34.14) .. controls (211.22,34.14) and (212.33,35.25) .. (212.33,36.62) .. controls (212.33,37.99) and (211.22,39.1) .. (209.85,39.1) .. controls (208.48,39.1) and (207.37,37.99) .. (207.37,36.62) -- cycle ;
\draw [color={rgb, 255:red, 128; green, 128; blue, 128 }  ,draw opacity=1 ] [dash pattern={on 0.84pt off 2.51pt}]  (209.85,36.62) -- (209.85,163.52) ;
\draw [color={rgb, 255:red, 128; green, 128; blue, 128 }  ,draw opacity=1 ] [dash pattern={on 0.84pt off 2.51pt}]  (243.73,29.87) -- (243.73,163.52) ;
\draw [color={rgb, 255:red, 128; green, 128; blue, 128 }  ,draw opacity=1 ] [dash pattern={on 0.84pt off 2.51pt}]  (275.96,36.15) -- (275.96,163.04) ;
\draw [color={rgb, 255:red, 128; green, 128; blue, 128 }  ,draw opacity=1 ] [dash pattern={on 0.84pt off 2.51pt}]  (322.5,95.41) -- (322.5,161.63) ;
\draw   (209.04,190.21) .. controls (209.04,194.84) and (211.36,197.16) .. (215.99,197.16) -- (215.99,197.16) .. controls (222.6,197.16) and (225.91,199.48) .. (225.91,204.11) .. controls (225.91,199.48) and (229.22,197.16) .. (235.84,197.16)(232.86,197.16) -- (235.84,197.16) .. controls (240.47,197.16) and (242.79,194.84) .. (242.79,190.21) ;
\draw   (243.51,190.21) .. controls (243.51,194.84) and (245.83,197.16) .. (250.46,197.16) -- (250.46,197.16) .. controls (257.08,197.16) and (260.39,199.48) .. (260.39,204.11) .. controls (260.39,199.48) and (263.7,197.16) .. (270.31,197.16)(267.33,197.16) -- (270.31,197.16) .. controls (274.94,197.16) and (277.26,194.84) .. (277.26,190.21) ;
\draw   (277.51,190.21) .. controls (277.51,194.88) and (279.84,197.21) .. (284.51,197.21) -- (290.53,197.21) .. controls (297.2,197.21) and (300.53,199.54) .. (300.53,204.21) .. controls (300.53,199.54) and (303.86,197.21) .. (310.53,197.21)(307.53,197.21) -- (316.54,197.21) .. controls (321.21,197.21) and (323.54,194.88) .. (323.54,190.21) ;
\draw   (323.89,190.06) .. controls (323.89,194.73) and (326.22,197.06) .. (330.89,197.06) -- (345.36,197.06) .. controls (352.03,197.06) and (355.36,199.39) .. (355.36,204.06) .. controls (355.36,199.39) and (358.69,197.06) .. (365.36,197.06)(362.36,197.06) -- (379.82,197.06) .. controls (384.49,197.06) and (386.82,194.73) .. (386.82,190.06) ;
\draw    (472.89,128.47) -- (472.89,117.15) ;
\draw [shift={(472.89,114.15)}, rotate = 450] [fill={rgb, 255:red, 0; green, 0; blue, 0 }  ][line width=0.08]  [draw opacity=0] (8.93,-4.29) -- (0,0) -- (8.93,4.29) -- cycle    ;
\draw    (516.65,129.1) -- (516.65,117.78) ;
\draw [shift={(516.65,114.78)}, rotate = 450] [fill={rgb, 255:red, 0; green, 0; blue, 0 }  ][line width=0.08]  [draw opacity=0] (8.93,-4.29) -- (0,0) -- (8.93,4.29) -- cycle    ;
\draw    (440.87,90.91) .. controls (454.61,60.22) and (411.03,60.19) .. (419.18,89.09) ;
\draw [shift={(420.09,91.86)}, rotate = 249.54000000000002] [fill={rgb, 255:red, 0; green, 0; blue, 0 }  ][line width=0.08]  [draw opacity=0] (8.93,-4.29) -- (0,0) -- (8.93,4.29) -- cycle    ;
\draw    (526.35,89.97) .. controls (540.09,59.27) and (496.51,59.24) .. (504.66,88.15) ;
\draw [shift={(505.57,90.91)}, rotate = 249.54000000000002] [fill={rgb, 255:red, 0; green, 0; blue, 0 }  ][line width=0.08]  [draw opacity=0] (8.93,-4.29) -- (0,0) -- (8.93,4.29) -- cycle    ;
\draw    (482.9,90.44) .. controls (496.64,59.75) and (453.06,59.71) .. (461.21,88.62) ;
\draw [shift={(462.12,91.38)}, rotate = 249.54000000000002] [fill={rgb, 255:red, 0; green, 0; blue, 0 }  ][line width=0.08]  [draw opacity=0] (8.93,-4.29) -- (0,0) -- (8.93,4.29) -- cycle    ;

\draw (379.22,154.08) node [anchor=north west][inner sep=0.75pt]  [font=\normalsize,color={rgb, 255:red, 0; green, 0; blue, 0 }  ,opacity=1 ]  {$x$};
\draw (85.82,12.15) node [anchor=north west][inner sep=0.75pt]  [font=\normalsize,color={rgb, 255:red, 0; green, 0; blue, 0 }  ,opacity=1 ]  {$f( x)$};
\draw (139.53,155.53) node [anchor=north west][inner sep=0.75pt]  [font=\normalsize,color={rgb, 255:red, 0; green, 0; blue, 0 }  ,opacity=1 ]  {$x_{1}$};
\draw (106.48,205.93) node [anchor=north west][inner sep=0.75pt]    {$I_{0}$};
\draw (171.25,205.83) node [anchor=north west][inner sep=0.75pt]    {$I_{1}$};
\draw (422.75,89.75) node [anchor=north west][inner sep=0.75pt]    {$I_{0}$};
\draw (465.25,89.75) node [anchor=north west][inner sep=0.75pt]    {$I_{1}$};
\draw (509.49,90.85) node [anchor=north west][inner sep=0.75pt]    {$I_{2}$};
\draw (201.4,156.95) node [anchor=north west][inner sep=0.75pt]  [font=\normalsize,color={rgb, 255:red, 0; green, 0; blue, 0 }  ,opacity=1 ]  {$x_{2}$};
\draw (227.98,157.61) node [anchor=north west][inner sep=0.75pt]  [font=\normalsize,color={rgb, 255:red, 0; green, 0; blue, 0 }  ,opacity=1 ]  {$1/2$};
\draw (264.18,158.42) node [anchor=north west][inner sep=0.75pt]  [font=\normalsize,color={rgb, 255:red, 0; green, 0; blue, 0 }  ,opacity=1 ]  {$1-x_{2}$};
\draw (311.74,158.31) node [anchor=north west][inner sep=0.75pt]  [font=\normalsize,color={rgb, 255:red, 0; green, 0; blue, 0 }  ,opacity=1 ]  {$1-x_{1}$};
\draw (218.87,206.87) node [anchor=north west][inner sep=0.75pt]    {$I_{2}$};
\draw (252.97,205.93) node [anchor=north west][inner sep=0.75pt]    {$I_{2} '$};
\draw (292.17,205.46) node [anchor=north west][inner sep=0.75pt]    {$I_{1} '$};
\draw (345.53,206.4) node [anchor=north west][inner sep=0.75pt]    {$I_{0} '$};
\draw (420.01,132.72) node [anchor=north west][inner sep=0.75pt]    {$I_{0} '$};
\draw (462.52,132.72) node [anchor=north west][inner sep=0.75pt]    {$I_{1} '$};
\draw (506.75,133.83) node [anchor=north west][inner sep=0.75pt]    {$I_{2} '$};

\end{tikzpicture}

%% file: fig/fig-inductive-fixed-point.tex
\tikzset{every picture/.style={line width=0.75pt}} 

\begin{tikzpicture}[x=0.75pt,y=0.75pt,yscale=-0.9,xscale=0.9]

\draw    (87.27,6.15) -- (87.27,285.18) ;
\draw    (390.36,285.18) -- (87.27,285.18) ;
\draw [color={rgb, 255:red, 128; green, 128; blue, 128 }  ,draw opacity=1 ] [dash pattern={on 0.84pt off 2.51pt}]  (137.84,238.09) -- (137.84,290.03) ;
\draw   (87.34,324.96) .. controls (87.34,329.63) and (89.67,331.96) .. (94.34,331.96) -- (102.85,331.96) .. controls (109.52,331.96) and (112.85,334.29) .. (112.85,338.96) .. controls (112.85,334.29) and (116.18,331.96) .. (122.85,331.96)(119.85,331.96) -- (131.35,331.96) .. controls (136.02,331.96) and (138.35,329.63) .. (138.35,324.96) ;
\draw   (139.13,325.13) .. controls (139.13,328.24) and (140.68,329.79) .. (143.79,329.79) -- (143.79,329.79) .. controls (148.22,329.79) and (150.44,331.34) .. (150.44,334.44) .. controls (150.44,331.34) and (152.66,329.79) .. (157.1,329.79)(155.1,329.79) -- (157.1,329.79) .. controls (160.21,329.79) and (161.76,328.24) .. (161.76,325.13) ;
\draw   (419.34,88.45) .. controls (419.34,79.78) and (426.37,72.75) .. (435.04,72.75) .. controls (443.72,72.75) and (450.75,79.78) .. (450.75,88.45) .. controls (450.75,97.13) and (443.72,104.16) .. (435.04,104.16) .. controls (426.37,104.16) and (419.34,97.13) .. (419.34,88.45) -- cycle ;
\draw   (466.99,88.45) .. controls (466.99,79.78) and (474.02,72.75) .. (482.69,72.75) .. controls (491.37,72.75) and (498.4,79.78) .. (498.4,88.45) .. controls (498.4,97.13) and (491.37,104.16) .. (482.69,104.16) .. controls (474.02,104.16) and (466.99,97.13) .. (466.99,88.45) -- cycle ;
\draw   (466.63,135.93) .. controls (466.63,127.25) and (473.67,120.22) .. (482.34,120.22) .. controls (491.02,120.22) and (498.05,127.25) .. (498.05,135.93) .. controls (498.05,144.6) and (491.02,151.64) .. (482.34,151.64) .. controls (473.67,151.64) and (466.63,144.6) .. (466.63,135.93) -- cycle ;
\draw   (418.81,135.93) .. controls (418.81,127.25) and (425.84,120.22) .. (434.51,120.22) .. controls (443.19,120.22) and (450.22,127.25) .. (450.22,135.93) .. controls (450.22,144.6) and (443.19,151.64) .. (434.51,151.64) .. controls (425.84,151.64) and (418.81,144.6) .. (418.81,135.93) -- cycle ;
\draw   (516.05,89.16) .. controls (516.05,80.49) and (523.08,73.45) .. (531.76,73.45) .. controls (540.43,73.45) and (547.46,80.49) .. (547.46,89.16) .. controls (547.46,97.83) and (540.43,104.87) .. (531.76,104.87) .. controls (523.08,104.87) and (516.05,97.83) .. (516.05,89.16) -- cycle ;
\draw   (516.05,136.63) .. controls (516.05,127.96) and (523.08,120.93) .. (531.76,120.93) .. controls (540.43,120.93) and (547.46,127.96) .. (547.46,136.63) .. controls (547.46,145.31) and (540.43,152.34) .. (531.76,152.34) .. controls (523.08,152.34) and (516.05,145.31) .. (516.05,136.63) -- cycle ;
\draw    (434.51,120.22) -- (434.51,107.16) ;
\draw [shift={(434.51,104.16)}, rotate = 450] [fill={rgb, 255:red, 0; green, 0; blue, 0 }  ][line width=0.08]  [draw opacity=0] (8.93,-4.29) -- (0,0) -- (8.93,4.29) -- cycle    ;
\draw [color={rgb, 255:red, 155; green, 155; blue, 155 }  ,draw opacity=1 ] [dash pattern={on 4.5pt off 4.5pt}]  (87.27,285.18) -- (383.12,11.22) ;
\draw    (87.27,285.18) .. controls (137.69,286.36) and (180.58,128.69) .. (195.33,104.27) .. controls (210.09,79.85) and (228.77,78.32) .. (242.14,78.32) .. controls (255.52,78.32) and (280.17,83.48) .. (288.9,101.03) .. controls (297.64,118.58) and (341.98,285.32) .. (390.36,285.18) ;
\draw  [fill={rgb, 255:red, 74; green, 144; blue, 226 }  ,fill opacity=1 ] (135.39,238.09) .. controls (135.39,236.73) and (136.49,235.63) .. (137.84,235.63) .. controls (139.2,235.63) and (140.3,236.73) .. (140.3,238.09) .. controls (140.3,239.44) and (139.2,240.54) .. (137.84,240.54) .. controls (136.49,240.54) and (135.39,239.44) .. (135.39,238.09) -- cycle ;
\draw  [fill={rgb, 255:red, 74; green, 144; blue, 226 }  ,fill opacity=1 ] (286.45,101.03) .. controls (286.45,99.67) and (287.55,98.57) .. (288.9,98.57) .. controls (290.26,98.57) and (291.36,99.67) .. (291.36,101.03) .. controls (291.36,102.38) and (290.26,103.48) .. (288.9,103.48) .. controls (287.55,103.48) and (286.45,102.38) .. (286.45,101.03) -- cycle ;
\draw [color={rgb, 255:red, 128; green, 128; blue, 128 }  ,draw opacity=1 ] [dash pattern={on 0.84pt off 2.51pt}]  (243.02,37.84) -- (243.02,290.5) ;
\draw [color={rgb, 255:red, 128; green, 128; blue, 128 }  ,draw opacity=1 ] [dash pattern={on 0.84pt off 2.51pt}]  (288.9,101.03) -- (288.9,286.29) ;
\draw [color={rgb, 255:red, 128; green, 128; blue, 128 }  ,draw opacity=1 ] [dash pattern={on 0.84pt off 2.51pt}]  (343.33,238.09) -- (343.33,292.36) ;
\draw   (162.03,326.15) .. controls (162.03,330.79) and (164.35,333.11) .. (168.99,333.11) -- (168.99,333.11) .. controls (175.62,333.11) and (178.94,335.43) .. (178.94,340.07) .. controls (178.94,335.43) and (182.25,333.11) .. (188.88,333.11)(185.9,333.11) -- (188.88,333.11) .. controls (193.52,333.11) and (195.84,330.79) .. (195.84,326.15) ;
\draw   (195.1,364.43) .. controls (195.1,369.1) and (197.43,371.43) .. (202.1,371.43) -- (233.3,371.43) .. controls (239.97,371.43) and (243.3,373.76) .. (243.3,378.43) .. controls (243.3,373.76) and (246.63,371.43) .. (253.3,371.43)(250.3,371.43) -- (284.49,371.43) .. controls (289.16,371.43) and (291.49,369.1) .. (291.49,364.43) ;
\draw    (482.69,120.22) -- (482.69,107.16) ;
\draw [shift={(482.69,104.16)}, rotate = 450] [fill={rgb, 255:red, 0; green, 0; blue, 0 }  ][line width=0.08]  [draw opacity=0] (8.93,-4.29) -- (0,0) -- (8.93,4.29) -- cycle    ;
\draw    (446.8,78.11) .. controls (462.28,43.53) and (412.84,43.67) .. (422.63,76.57) ;
\draw [shift={(423.5,79.17)}, rotate = 249.54000000000002] [fill={rgb, 255:red, 0; green, 0; blue, 0 }  ][line width=0.08]  [draw opacity=0] (8.93,-4.29) -- (0,0) -- (8.93,4.29) -- cycle    ;
\draw    (493.92,77.58) .. controls (509.4,43) and (459.96,43.14) .. (469.75,76.05) ;
\draw [shift={(470.62,78.64)}, rotate = 249.54000000000002] [fill={rgb, 255:red, 0; green, 0; blue, 0 }  ][line width=0.08]  [draw opacity=0] (8.93,-4.29) -- (0,0) -- (8.93,4.29) -- cycle    ;
\draw [color={rgb, 255:red, 128; green, 128; blue, 128 }  ,draw opacity=1 ] [dash pattern={on 0.84pt off 2.51pt}]  (195.97,101.49) -- (195.97,286.76) ;
\draw [color={rgb, 255:red, 128; green, 128; blue, 128 }  ,draw opacity=1 ] [dash pattern={on 0.84pt off 2.51pt}]  (317.44,101.96) -- (81.14,101.96) ;
\draw [color={rgb, 255:red, 128; green, 128; blue, 128 }  ,draw opacity=1 ] [dash pattern={on 0.84pt off 2.51pt}]  (344.73,237.68) -- (81.14,237.68) ;
\draw [color={rgb, 255:red, 128; green, 128; blue, 128 }  ,draw opacity=1 ] [dash pattern={on 0.84pt off 2.51pt}]  (318.77,184.76) -- (82.61,184.76) ;
\draw [color={rgb, 255:red, 128; green, 128; blue, 128 }  ,draw opacity=1 ] [dash pattern={on 0.84pt off 2.51pt}]  (318.77,56.44) -- (318.77,305.6) ;
\draw [color={rgb, 255:red, 128; green, 128; blue, 128 }  ,draw opacity=1 ] [dash pattern={on 0.84pt off 2.51pt}]  (162.06,100.2) -- (162.06,307.2) ;
\draw   (196.11,326.15) .. controls (196.11,330.82) and (198.44,333.15) .. (203.11,333.15) -- (210.15,333.15) .. controls (216.82,333.15) and (220.15,335.48) .. (220.15,340.15) .. controls (220.15,335.48) and (223.48,333.15) .. (230.15,333.15)(227.15,333.15) -- (237.18,333.15) .. controls (241.85,333.15) and (244.18,330.82) .. (244.18,326.15) ;
\draw   (244.45,326.15) .. controls (244.45,330.82) and (246.78,333.15) .. (251.45,333.15) -- (257.21,333.15) .. controls (263.88,333.15) and (267.21,335.48) .. (267.21,340.15) .. controls (267.21,335.48) and (270.54,333.15) .. (277.21,333.15)(274.21,333.15) -- (282.97,333.15) .. controls (287.64,333.15) and (289.97,330.82) .. (289.97,326.15) ;
\draw   (289.22,325.64) .. controls (289.22,329.86) and (291.33,331.97) .. (295.55,331.97) -- (295.55,331.97) .. controls (301.58,331.97) and (304.6,334.08) .. (304.6,338.3) .. controls (304.6,334.08) and (307.62,331.97) .. (313.65,331.97)(310.94,331.97) -- (313.65,331.97) .. controls (317.88,331.97) and (319.99,329.86) .. (319.99,325.64) ;
\draw   (320.77,325.64) .. controls (320.77,328.75) and (322.32,330.3) .. (325.42,330.3) -- (325.42,330.3) .. controls (329.86,330.3) and (332.08,331.85) .. (332.08,334.95) .. controls (332.08,331.85) and (334.3,330.3) .. (338.73,330.3)(336.73,330.3) -- (338.73,330.3) .. controls (341.84,330.3) and (343.39,328.75) .. (343.39,325.64) ;
\draw   (344.78,325.47) .. controls (344.78,330.14) and (347.11,332.47) .. (351.78,332.47) -- (358.25,332.47) .. controls (364.92,332.47) and (368.25,334.8) .. (368.25,339.47) .. controls (368.25,334.8) and (371.58,332.47) .. (378.25,332.47)(375.25,332.47) -- (384.72,332.47) .. controls (389.39,332.47) and (391.72,330.14) .. (391.72,325.47) ;
\draw   (565.29,110.34) .. controls (565.29,101.66) and (572.32,94.63) .. (581,94.63) .. controls (589.67,94.63) and (596.7,101.66) .. (596.7,110.34) .. controls (596.7,119.01) and (589.67,126.05) .. (581,126.05) .. controls (572.32,126.05) and (565.29,119.01) .. (565.29,110.34) -- cycle ;
\draw [color={rgb, 255:red, 65; green, 117; blue, 5 }  ,draw opacity=1 ]   (87.27,285.18) .. controls (157.18,283.36) and (158.81,115.97) .. (162.06,100.2) .. controls (165.32,84.42) and (171.42,79.34) .. (181.6,79.85) .. controls (191.77,80.35) and (191.51,87.53) .. (195.97,101.49) .. controls (200.42,115.46) and (222.41,170.86) .. (243.67,170.41) .. controls (264.93,169.95) and (285.39,113.93) .. (288.9,101.03) .. controls (292.42,88.12) and (290.48,80.35) .. (300.65,79.85) .. controls (310.83,79.34) and (310.83,81.88) .. (317.44,101.21) .. controls (324.06,120.55) and (324.56,286.92) .. (390.36,285.18) ;
\draw [color={rgb, 255:red, 128; green, 128; blue, 128 }  ,draw opacity=1 ] [dash pattern={on 0.84pt off 2.51pt}]  (333.72,215.29) -- (84.71,215.29) ;
\draw [color={rgb, 255:red, 128; green, 128; blue, 128 }  ,draw opacity=1 ] [dash pattern={on 0.84pt off 2.51pt}]  (304.21,141.01) -- (84.14,141.01) ;
\draw [color={rgb, 255:red, 128; green, 128; blue, 128 }  ,draw opacity=1 ] [dash pattern={on 0.84pt off 2.51pt}]  (319.48,70.93) -- (83.18,70.93) ;
\draw    (498.4,88.45) -- (513.05,88.45) ;
\draw [shift={(516.05,88.45)}, rotate = 180] [fill={rgb, 255:red, 0; green, 0; blue, 0 }  ][line width=0.08]  [draw opacity=0] (8.93,-4.29) -- (0,0) -- (8.93,4.29) -- cycle    ;
\draw    (547.46,89.16) -- (565.55,101.32) ;
\draw [shift={(568.04,103)}, rotate = 213.92000000000002] [fill={rgb, 255:red, 0; green, 0; blue, 0 }  ][line width=0.08]  [draw opacity=0] (8.93,-4.29) -- (0,0) -- (8.93,4.29) -- cycle    ;
\draw    (547.46,136.63) -- (566.67,122.75) ;
\draw [shift={(569.1,121)}, rotate = 504.15] [fill={rgb, 255:red, 0; green, 0; blue, 0 }  ][line width=0.08]  [draw opacity=0] (8.93,-4.29) -- (0,0) -- (8.93,4.29) -- cycle    ;
\draw    (572.81,123.65) -- (549.3,140.43) ;
\draw [shift={(546.86,142.18)}, rotate = 324.46000000000004] [fill={rgb, 255:red, 0; green, 0; blue, 0 }  ][line width=0.08]  [draw opacity=0] (8.93,-4.29) -- (0,0) -- (8.93,4.29) -- cycle    ;
\draw    (491.8,124.7) -- (517.13,101.32) ;
\draw [shift={(519.33,99.29)}, rotate = 497.29] [fill={rgb, 255:red, 0; green, 0; blue, 0 }  ][line width=0.08]  [draw opacity=0] (8.93,-4.29) -- (0,0) -- (8.93,4.29) -- cycle    ;
\draw   (418.81,202.29) .. controls (418.81,193.61) and (425.84,186.58) .. (434.51,186.58) .. controls (443.19,186.58) and (450.22,193.61) .. (450.22,202.29) .. controls (450.22,210.96) and (443.19,217.99) .. (434.51,217.99) .. controls (425.84,217.99) and (418.81,210.96) .. (418.81,202.29) -- cycle ;
\draw   (466.46,202.29) .. controls (466.46,193.61) and (473.49,186.58) .. (482.16,186.58) .. controls (490.84,186.58) and (497.87,193.61) .. (497.87,202.29) .. controls (497.87,210.96) and (490.84,217.99) .. (482.16,217.99) .. controls (473.49,217.99) and (466.46,210.96) .. (466.46,202.29) -- cycle ;
\draw   (466.1,249.76) .. controls (466.1,241.09) and (473.14,234.05) .. (481.81,234.05) .. controls (490.49,234.05) and (497.52,241.09) .. (497.52,249.76) .. controls (497.52,258.44) and (490.49,265.47) .. (481.81,265.47) .. controls (473.14,265.47) and (466.1,258.44) .. (466.1,249.76) -- cycle ;
\draw   (418.28,249.76) .. controls (418.28,241.09) and (425.31,234.05) .. (433.98,234.05) .. controls (442.66,234.05) and (449.69,241.09) .. (449.69,249.76) .. controls (449.69,258.44) and (442.66,265.47) .. (433.98,265.47) .. controls (425.31,265.47) and (418.28,258.44) .. (418.28,249.76) -- cycle ;
\draw   (515.52,202.99) .. controls (515.52,194.32) and (522.55,187.29) .. (531.23,187.29) .. controls (539.9,187.29) and (546.93,194.32) .. (546.93,202.99) .. controls (546.93,211.67) and (539.9,218.7) .. (531.23,218.7) .. controls (522.55,218.7) and (515.52,211.67) .. (515.52,202.99) -- cycle ;
\draw   (515.52,250.47) .. controls (515.52,241.79) and (522.55,234.76) .. (531.23,234.76) .. controls (539.9,234.76) and (546.93,241.79) .. (546.93,250.47) .. controls (546.93,259.14) and (539.9,266.17) .. (531.23,266.17) .. controls (522.55,266.17) and (515.52,259.14) .. (515.52,250.47) -- cycle ;
\draw [color={rgb, 255:red, 65; green, 117; blue, 5 }  ,draw opacity=1 ]   (433.98,234.05) -- (433.98,220.99) ;
\draw [shift={(433.98,217.99)}, rotate = 450] [fill={rgb, 255:red, 65; green, 117; blue, 5 }  ,fill opacity=1 ][line width=0.08]  [draw opacity=0] (8.93,-4.29) -- (0,0) -- (8.93,4.29) -- cycle    ;
\draw [color={rgb, 255:red, 65; green, 117; blue, 5 }  ,draw opacity=1 ]   (482.16,234.05) -- (482.16,220.99) ;
\draw [shift={(482.16,217.99)}, rotate = 450] [fill={rgb, 255:red, 65; green, 117; blue, 5 }  ,fill opacity=1 ][line width=0.08]  [draw opacity=0] (8.93,-4.29) -- (0,0) -- (8.93,4.29) -- cycle    ;
\draw [color={rgb, 255:red, 65; green, 117; blue, 5 }  ,draw opacity=1 ]   (446.27,191.94) .. controls (461.75,157.36) and (412.31,157.5) .. (422.1,190.41) ;
\draw [shift={(422.97,193)}, rotate = 249.54000000000002] [fill={rgb, 255:red, 65; green, 117; blue, 5 }  ,fill opacity=1 ][line width=0.08]  [draw opacity=0] (8.93,-4.29) -- (0,0) -- (8.93,4.29) -- cycle    ;
\draw [color={rgb, 255:red, 65; green, 117; blue, 5 }  ,draw opacity=1 ]   (493.39,191.42) .. controls (508.88,156.83) and (459.43,156.97) .. (469.22,189.88) ;
\draw [shift={(470.09,192.47)}, rotate = 249.54000000000002] [fill={rgb, 255:red, 65; green, 117; blue, 5 }  ,fill opacity=1 ][line width=0.08]  [draw opacity=0] (8.93,-4.29) -- (0,0) -- (8.93,4.29) -- cycle    ;
\draw   (564.76,224.17) .. controls (564.76,215.5) and (571.79,208.46) .. (580.47,208.46) .. controls (589.14,208.46) and (596.17,215.5) .. (596.17,224.17) .. controls (596.17,232.85) and (589.14,239.88) .. (580.47,239.88) .. controls (571.79,239.88) and (564.76,232.85) .. (564.76,224.17) -- cycle ;
\draw [color={rgb, 255:red, 65; green, 117; blue, 5 }  ,draw opacity=1 ]   (594.51,216.83) .. controls (619.81,193.08) and (618.43,260.94) .. (594.3,236.38) ;
\draw [shift={(592.4,234.3)}, rotate = 409.32] [fill={rgb, 255:red, 65; green, 117; blue, 5 }  ,fill opacity=1 ][line width=0.08]  [draw opacity=0] (8.93,-4.29) -- (0,0) -- (8.93,4.29) -- cycle    ;
\draw [color={rgb, 255:red, 65; green, 117; blue, 5 }  ,draw opacity=1 ]   (539.45,263.42) .. controls (565.78,291.3) and (499.5,295.01) .. (520.21,265.25) ;
\draw [shift={(521.98,262.89)}, rotate = 488.77] [fill={rgb, 255:red, 65; green, 117; blue, 5 }  ,fill opacity=1 ][line width=0.08]  [draw opacity=0] (8.93,-4.29) -- (0,0) -- (8.93,4.29) -- cycle    ;
\draw [color={rgb, 255:red, 65; green, 117; blue, 5 }  ,draw opacity=1 ]   (531.23,218.7) -- (531.23,231.76) ;
\draw [shift={(531.23,234.76)}, rotate = 270] [fill={rgb, 255:red, 65; green, 117; blue, 5 }  ,fill opacity=1 ][line width=0.08]  [draw opacity=0] (8.93,-4.29) -- (0,0) -- (8.93,4.29) -- cycle    ;
\draw [color={rgb, 255:red, 65; green, 117; blue, 5 }  ,draw opacity=1 ]   (493.39,191.42) .. controls (518.55,165.74) and (565.14,177.42) .. (579.25,205.78) ;
\draw [shift={(580.47,208.46)}, rotate = 247.69] [fill={rgb, 255:red, 65; green, 117; blue, 5 }  ,fill opacity=1 ][line width=0.08]  [draw opacity=0] (8.93,-4.29) -- (0,0) -- (8.93,4.29) -- cycle    ;
\draw [color={rgb, 255:red, 65; green, 117; blue, 5 }  ,draw opacity=1 ]   (493.39,238.54) .. controls (513.61,229.94) and (542.53,220.38) .. (562.06,223.63) ;
\draw [shift={(564.76,224.17)}, rotate = 193.55] [fill={rgb, 255:red, 65; green, 117; blue, 5 }  ,fill opacity=1 ][line width=0.08]  [draw opacity=0] (8.93,-4.29) -- (0,0) -- (8.93,4.29) -- cycle    ;
\draw [color={rgb, 255:red, 65; green, 117; blue, 5 }  ,draw opacity=1 ]   (497.87,202.99) -- (512.52,202.99) ;
\draw [shift={(515.52,202.99)}, rotate = 180] [fill={rgb, 255:red, 65; green, 117; blue, 5 }  ,fill opacity=1 ][line width=0.08]  [draw opacity=0] (8.93,-4.29) -- (0,0) -- (8.93,4.29) -- cycle    ;
\draw [color={rgb, 255:red, 65; green, 117; blue, 5 }  ,draw opacity=1 ]   (493.39,238.54) -- (518.72,215.16) ;
\draw [shift={(520.92,213.12)}, rotate = 497.29] [fill={rgb, 255:red, 65; green, 117; blue, 5 }  ,fill opacity=1 ][line width=0.08]  [draw opacity=0] (8.93,-4.29) -- (0,0) -- (8.93,4.29) -- cycle    ;

\draw (383.05,281.95) node [anchor=north west][inner sep=0.75pt]  [font=\normalsize,color={rgb, 255:red, 0; green, 0; blue, 0 }  ,opacity=1 ]  {$x$};
\draw (204.48,45.36) node [anchor=north west][inner sep=0.75pt]  [font=\normalsize,color={rgb, 255:red, 0; green, 0; blue, 0 }  ,opacity=1 ]  {$f( x)$};
\draw (127.77,284.23) node [anchor=north west][inner sep=0.75pt]  [font=\normalsize,color={rgb, 255:red, 0; green, 0; blue, 0 }  ,opacity=1 ]  {$x_{1}$};
\draw (107.11,339.67) node [anchor=north west][inner sep=0.75pt]    {$I_{0}$};
\draw (143.83,341.53) node [anchor=north west][inner sep=0.75pt]    {$I_{1}$};
\draw (427.51,77.97) node [anchor=north west][inner sep=0.75pt]    {$I_{0}$};
\draw (475.16,77.97) node [anchor=north west][inner sep=0.75pt]    {$I_{1}$};
\draw (524.76,79.21) node [anchor=north west][inner sep=0.75pt]    {$I_{2}$};
\draw (230.79,285.95) node [anchor=north west][inner sep=0.75pt]  [font=\normalsize,color={rgb, 255:red, 0; green, 0; blue, 0 }  ,opacity=1 ]  {$1/2$};
\draw (279.55,284.41) node [anchor=north west][inner sep=0.75pt]  [font=\normalsize,color={rgb, 255:red, 0; green, 0; blue, 0 }  ,opacity=1 ]  {$x_{2}$};
\draw (326.13,285.5) node [anchor=north west][inner sep=0.75pt]  [font=\normalsize,color={rgb, 255:red, 0; green, 0; blue, 0 }  ,opacity=1 ]  {$1-x_{1}$};
\draw (174.27,342.21) node [anchor=north west][inner sep=0.75pt]    {$I_{2}$};
\draw (296.17,339.67) node [anchor=north west][inner sep=0.75pt]    {$I_{2} '$};
\draw (323.64,340.69) node [anchor=north west][inner sep=0.75pt]    {$I_{1} '$};
\draw (361.29,341.7) node [anchor=north west][inner sep=0.75pt]    {$I_{0} '$};
\draw (424.75,126.15) node [anchor=north west][inner sep=0.75pt]    {$I_{0} '$};
\draw (472.4,126.15) node [anchor=north west][inner sep=0.75pt]    {$I_{1} '$};
\draw (521.99,127.39) node [anchor=north west][inner sep=0.75pt]    {$I_{2} '$};
\draw (175.48,285.35) node [anchor=north west][inner sep=0.75pt]  [font=\normalsize,color={rgb, 255:red, 0; green, 0; blue, 0 }  ,opacity=1 ]  {$1-x_{2}$};
\draw (60.49,92.36) node [anchor=north west][inner sep=0.75pt]  [font=\normalsize,color={rgb, 255:red, 0; green, 0; blue, 0 }  ,opacity=1 ]  {$x_{2}$};
\draw (61.57,228) node [anchor=north west][inner sep=0.75pt]  [font=\normalsize,color={rgb, 255:red, 0; green, 0; blue, 0 }  ,opacity=1 ]  {$x_{1}$};
\draw (42.45,173.29) node [anchor=north west][inner sep=0.75pt]  [font=\normalsize,color={rgb, 255:red, 0; green, 0; blue, 0 }  ,opacity=1 ]  {$1-x_{2}$};
\draw (307.33,301.55) node [anchor=north west][inner sep=0.75pt]  [font=\normalsize,color={rgb, 255:red, 0; green, 0; blue, 0 }  ,opacity=1 ]  {$x_{2} '$};
\draw (139.9,303.63) node [anchor=north west][inner sep=0.75pt]  [font=\normalsize,color={rgb, 255:red, 0; green, 0; blue, 0 }  ,opacity=1 ]  {$1-x_{2} '$};
\draw (258.03,339.67) node [anchor=north west][inner sep=0.75pt]    {$I_{3} '$};
\draw (213.96,342.21) node [anchor=north west][inner sep=0.75pt]    {$I_{3}$};
\draw (238.38,376.91) node [anchor=north west][inner sep=0.75pt]    {$J_{3}$};
\draw (574,100.39) node [anchor=north west][inner sep=0.75pt]    {$J_{3}$};
\draw (129.25,53.77) node [anchor=north west][inner sep=0.75pt]  [font=\normalsize,color={rgb, 255:red, 0; green, 0; blue, 0 }  ,opacity=1 ]  {$\textcolor[rgb]{0.25,0.46,0.02}{f}\textcolor[rgb]{0.25,0.46,0.02}{^{2}}\textcolor[rgb]{0.25,0.46,0.02}{(}\textcolor[rgb]{0.25,0.46,0.02}{x}\textcolor[rgb]{0.25,0.46,0.02}{)}$};
\draw (38.84,204.8) node [anchor=north west][inner sep=0.75pt]  [font=\normalsize,color={rgb, 255:red, 0; green, 0; blue, 0 }  ,opacity=1 ]  {$1-x_{2} '$};
\draw (61.94,123.74) node [anchor=north west][inner sep=0.75pt]  [font=\normalsize,color={rgb, 255:red, 0; green, 0; blue, 0 }  ,opacity=1 ]  {$\frac{1}{2}$};
\draw (59.2,61.32) node [anchor=north west][inner sep=0.75pt]  [font=\normalsize,color={rgb, 255:red, 0; green, 0; blue, 0 }  ,opacity=1 ]  {$x_{2} '$};
\draw (385.49,86.51) node [anchor=north west][inner sep=0.75pt]  [font=\normalsize,color={rgb, 255:red, 0; green, 0; blue, 0 }  ,opacity=1 ]  {$f:$};
\draw (426.98,191.8) node [anchor=north west][inner sep=0.75pt]    {$I_{0}$};
\draw (474.63,191.8) node [anchor=north west][inner sep=0.75pt]    {$I_{1}$};
\draw (524.23,193.04) node [anchor=north west][inner sep=0.75pt]    {$I_{2}$};
\draw (424.22,239.98) node [anchor=north west][inner sep=0.75pt]    {$I_{0} '$};
\draw (471.87,239.98) node [anchor=north west][inner sep=0.75pt]    {$I_{1} '$};
\draw (521.46,241.22) node [anchor=north west][inner sep=0.75pt]    {$I_{2} '$};
\draw (573.47,214.22) node [anchor=north west][inner sep=0.75pt]    {$J_{3}$};
\draw (384.96,200.35) node [anchor=north west][inner sep=0.75pt]  [font=\normalsize,color={rgb, 255:red, 65; green, 117; blue, 5 }  ,opacity=1 ]  {$f:$};

\end{tikzpicture}